\documentclass{article}

\usepackage{microtype}
\usepackage{graphicx}
\usepackage{subfigure}
\usepackage{booktabs} 

\usepackage{arxiv}
\usepackage[numbers,sort&compress]{natbib}

\usepackage{hyperref}

\usepackage{amsmath}
\usepackage{amssymb}
\usepackage{mathtools}
\usepackage{amsthm}

\usepackage[capitalize,noabbrev]{cleveref}

\theoremstyle{plain}
\newtheorem{theorem}{Theorem}[section]
\newtheorem{proposition}[theorem]{Proposition}
\newtheorem{lemma}[theorem]{Lemma}
\newtheorem{corollary}[theorem]{Corollary}
\theoremstyle{definition}

\newtheorem{assumption}[theorem]{Assumption}
\theoremstyle{remark}
\newtheorem{remark}[theorem]{Remark}

\usepackage[disable,textsize=tiny]{todonotes}
\usepackage{dsfont}
\usepackage{booktabs}
\usepackage{multirow}
\usepackage{tabularx}
\usepackage{wrapfig}
\usepackage[dvipsnames]{xcolor}
\usepackage{colortbl}
\usepackage{array}
\usepackage{pifont}
\usepackage{xspace}
\usepackage{makecell}
\usepackage{enumitem}
\usepackage[most]{tcolorbox}
\usepackage{float}
\usepackage{wrapfig}
\usepackage{caption}
\usepackage{placeins}
\usepackage{tikz}

\newcommand{\indep}{\mathrel{\perp\mspace{-10mu}\perp}} 
\newcommand{\lessconst}{\mathrel{\lesssim}}
\newcommand{\E}{\mathbb{E}}                
\newcommand{\var}{\mathrm{Var}}            
\newcommand{\1}{\mathds{1}}                
\newcommand{\diff}{\mathrm{d}}

\DeclareMathOperator*{\argmin}{arg\,min}

\definecolor{shallowyellow}{RGB}{255, 255, 210}  
\definecolor{shallowblue}{RGB}{210, 235, 255}    
\definecolor{darkgreen}{RGB}{0,100,0}
\definecolor{lightred}{RGB}{180,40,40}
\definecolor{lightgreen}{RGB}{0,120,0}
\definecolor{icmlgray}{RGB}{245,245,245}
\newtcolorbox{contributionbox}{
  colback=icmlgray,
  colframe=black,
  boxrule=0pt,
  arc=2pt,
  left=6pt,
  right=6pt,
  top=6pt,
  bottom=6pt,
  boxsep=0pt,
  breakable
}

\newcommand{\xmark}{\textcolor{lightred}{\ding{55}}}
\newcommand{\cmark}{\textcolor{ForestGreen}{\ding{51}}}

\newcommand{\methods}{LT-O-learners\xspace}
\newcommand{\LTT}{LT-T-learner\xspace}
\newcommand{\LTRA}{LT-RA-learner\xspace}
\newcommand{\LTIPW}{LT-IPW-learner\xspace}
\newcommand{\DR}{LT-O-DR-learner\xspace}
\newcommand{\TO}{LT-O-TO-learner\xspace}
\newcommand{\LO}{LT-O-LO-learner\xspace}
\newcommand{\DO}{LT-O-DO-learner\xspace}

\title{Orthogonal Learner for Estimating Heterogeneous Long-Term Treatment Effects}

\author{%
  Haorui Ma\thanks{Equal contribution. Correspondence to: Haorui Ma \texttt{<H.Ma@lmu.de>}, Stefan Feuerriegel \texttt{<feuerriegel@lmu.de>}.} \\
  AI in Management, LMU Munich\\
  Munich Center for Machine Learning\\
  Munich, Germany\\
  \And
  Dennis Frauen\footnotemark[1] \\
  AI in Management, LMU Munich\\
  Munich Center for Machine Learning\\
  Munich, Germany\\
  \And
  Valentyn Melnychuk\footnotemark[1] \\
  AI in Management, LMU Munich\\
  Munich Center for Machine Learning\\
  Munich, Germany\\
  \And
  Stefan Feuerriegel\footnotemark[1] \\
  AI in Management, LMU Munich\\
  Munich Center for Machine Learning\\
  Munich, Germany\\
}

\begin{document}

\maketitle

\begin{abstract}
Estimation of heterogeneous \textit{long-term} treatment effects (HLTEs) is relevant for personalized decision-making in marketing, economics, and medicine, where short-term observational datasets are often combined with long-term observational datasets. However, HLTE estimation is challenging due to limited overlap in treatment assignments or in long-term outcomes for certain subpopulations, which can lead to unstable HLTE estimates with large finite-sample variance. To address this challenge, we introduce the \textbf{LT-O-learners} (Long-Term Orthogonal Learners), a set of novel orthogonal learners for HLTE estimation in the canonical HLTE setting with surrogacy. The key idea of our \methods is to \textit{retarget} the loss via custom overlap weights that downweight low-overlap samples. We show that the retargeted loss recovers the true HLTE pointwise and satisfies Neyman-orthogonality. We further prove two key theoretical results: (i)~The nuisance error enters the error bound only through higher-order terms, which means our learners are robust to nuisance estimation error. (ii)~Under a linear function class, the retargeting effectively controls the asymptotic variance of the HLTE estimator via the overlap weights in low-overlap regimes. We conduct experiments on synthetic and real-world datasets to confirm the theoretical properties of our \methods, particularly robustness in low-overlap regimes. To our knowledge, ours are the first orthogonal learners for HLTE estimation robust to low overlap in long-term settings.
\end{abstract}

\section{Introduction}


Estimating heterogeneous \textit{long-term} treatment effects (HLTEs) from short-term studies is a common problem in marketing, economics, and medicine~\citep{Feuerriegel.2024, Yang.2024-LearningTheOptimalPolicy, Athey.2025}. In many real-world applications, the primary outcomes of interest materialize only after substantial delays or incur high measurement costs. As a result, short-term studies typically record only surrogate outcomes, and must be combined with auxiliary data sources that capture long-term outcomes in order to estimate HLTEs~\citep{Cai.2024-latent, Cai.2025-Individual, Chen.2025}.

\textbf{Example.} In job training, policy-makers seek long-term effects on employment (e.g., earnings in 3 years) but only observe employment status or earnings within a short time period~\citep{Hotz.2006}. Similarly, recommending a new therapy targets long-term benefits (e.g., years of life gained), while only short-term surrogates such as tumor shrinkage are available~\citep{Belin.2020}. Timely decisions thus require estimating heterogeneous long-term effects from short-term outcomes, aided by historical data that contain long-term outcomes but may lack treatment information.




Formally, the task of estimating long-term treatment effects is typically formalized as follows~\citep{Chen.2023, Yang.2024-Targeting, Athey.2025}: a short-term dataset $\mathcal{D}_1$ (either experimental or observational), in which treatment is observed alongside surrogate outcomes, and a long-term observational dataset $\mathcal{D}_2$, in which long-term outcomes are recorded but treatment assignment may be missing or unobserved (e.g., because it is a new medical treatment or marketing intervention that was not yet applied before). The two datasets have shared covariates, which then enable identification and estimation of long-term treatment effects under suitable assumptions. In this paper, we focus on the canonical setting with surrogacy and comparability assumptions~\citep{Athey.2025, Chen.2023, Yang.2024-Targeting, Kallus.2025}. Under identification, classic confounding adjustment strategies can be adapted to derive meta-learners\footnote{Meta-learners are model-agnostic algorithms that decompose heterogeneous treatment effect estimation into a collection of sub-regression problems \citep{Kunzel.2019, Curth.2021}.} for HLTE estimation, such as the \emph{long-term} T-learner, RA-learner, IPW-learner, and DR-learner (See Section~\ref{sec:standard_learners}, Appendix~\ref{app:analogy}). 



However, the meta-learners mentioned above fail to address challenges due to \textit{low overlap}, which can lead to unstable estimates with large finite-sample variance~\citep{Jesson.2020, Fisher.2023}. Broadly, overlap refers to the availability of comparable samples across relevant regimes for individuals with similar covariates. In the HLTE setting, overlap arises in two different forms. (i)~\textit{\textbf{Limited treatment overlap:}} in the short-term dataset $\mathcal{D}_1$, some subpopulations may have a low probability of receiving one of the treatment arms, leading to limited treatment overlap in finite samples. (ii)~\textit{\textbf{Limited long-term outcome overlap:}} unique to the long-term setting is that some subpopulations may have a low probability of observing long-term outcomes in the historical dataset $\mathcal{D}_2$. As a result, when overlap is weak in either dimension, HLTE meta-learners are prone to high finite-sample variance and unstable subgroup-level estimates (see Section~\ref{sec:standard_learners},~\ref{subsec:instantiation} and Appendix~\ref{app:variance}). 



\begin{wrapfigure}{r}{0.5\textwidth}
    \centering
    \vspace{-0.4cm}
    \captionsetup{width=\linewidth}
    \includegraphics[width=\linewidth]{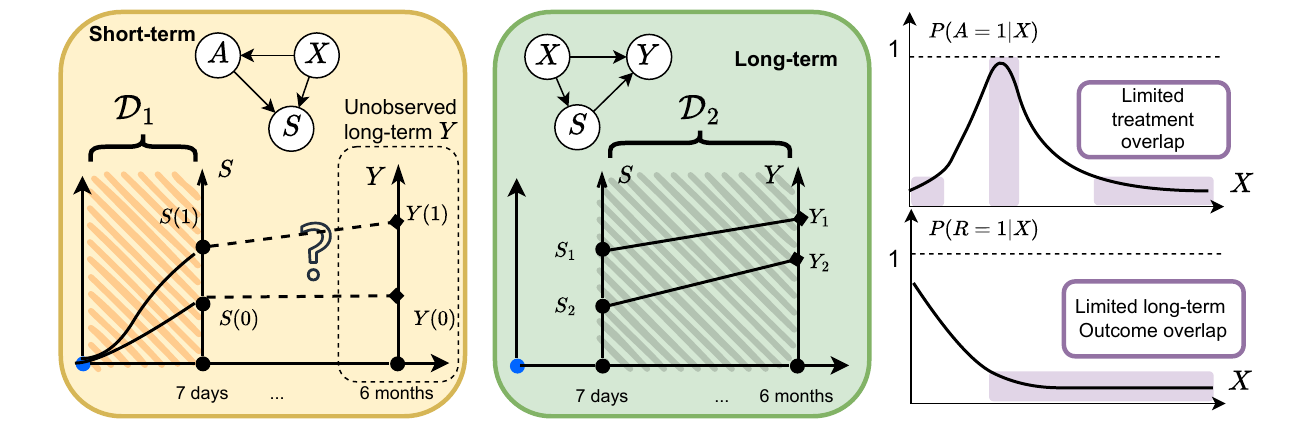}
    \caption{\textbf{Our HLTE setting:} The yellow and green box illustrate our two-sample setting. $\mathcal{D}_1$ and $\mathcal{D}_2$ share covariates $X$ and surrogates $S$. The surrogate index (i.e. $\E_{\mathcal{D}_2}[Y|S,X]$), learned in the long-term dataset, is used as a proxy for long-term outcomes in the short-term dataset. The left plot shows the two overlap issues.}
    \label{fig:concept}
    \vspace{-0.5cm}
\end{wrapfigure}

In this paper, we address these challenges by proposing the \textbf{\methods} (Long-Term Orthogonal Learners), a novel class of orthogonal meta-learners for estimating HLTEs from a short-term dataset $\mathcal{D}_1$ and a long-term observational dataset $\mathcal{D}_2$. The key idea is to \textit{retarget} the learning objective via overlap weights that downweight samples with weak treatment or long-term outcome overlap, yielding a loss that still recovers the true HLTE pointwise and satisfies Neyman-orthogonality. Because the efficiency-optimal weight is intractable---it depends on all nuisances and breaks orthogonality---we instead build a class of orthogonal learners with customizable weights, instantiated by four canonical members (\DR, \TO, \LO, \DO) that target different overlap regimes. We further prove the benefits of Neyman-orthogonality and overlap weighting with two respective theorems to highlight the theoretical strength. Importantly, our \methods are fully model-agnostic and can be instantiated with arbitrary machine learning models, making them broadly applicable in high-dimensional HLTE settings.



Our \textbf{contributions} are:\footnote{Code is available at \href{https://anonymous.4open.science/r/Orthogonal-learners-for-Long-term-effects-64DD}{https://anonymous.4open.science/r/Orthogonal-learners-for-Long-term-effects-64DD}. } (1)~We propose the \textit{\methods}, a novel class of orthogonal meta-learners for HLTE estimation that are robust to different forms of low overlap. (2)~We prove two theoretical results: (i)~the nuisance error enters the error bound only through higher-order terms; and (ii)~under a linear function class, the weighting scheme controls the asymptotic variance of the HLTE estimator in low-overlap regimes. (3)~We conduct experiments on synthetic and real-world datasets to confirm the theoretical properties of our \methods, particularly robustness in low-overlap regimes.


\vspace{-0.2cm}

\section{Related work}
\label{sec:related_work}

\textbf{Orthogonal learners for heterogeneous treatment effects (HTEs):} Several machine learning methods have been developed for the standard HTE setting with a single dataset \citep[e.g.,][]{Shalit.2017, Curth.2021}, where, unlike our HLTE setting, the outcome of interest is directly available in the randomized (or observational) data. State-of-the-art methods are model-agnostic two-stage estimators, in which nuisance functions are estimated in the first stage and where the target HTE is learned by minimizing an orthogonal loss function \citep{Foster.2023} in the second stage to predict the HTEs. Orthogonality ensures robustness w.r.t. nuisance estimation errors \citep{Chernozhukov.2018} and thus prevents plug-in bias~\citep{Kennedy.2022-Sempiparametric}.
Examples include the DR-learner \citep{Kennedy.2023}, and R-learner \citep{Nie.2021}. However, these learners are \textit{not} directly transferable from the HTE to the HLTE setting. 



\textbf{Data combination for causal inference:} A large body of work studies treatment effect estimation by combining multiple datasets, such as combinations of observational datasets~\citep[e.g.,][]{Yang.2020, Yuxin.2025}, combinations of experimental datasets~\citep[e.g.,][]{Debray.2015, Tan.2022, Brantner.2024-comparison}, or mixed combinations of both observational \textit{and} experimental datasets \citep[e.g.,][]{Hatt.2022, Wu.2022-integrative, Brantner.2023, Yao.2024combining, Yuxin.2025, parikh.2025double, Chen.2025, Cai.2025-Individual, Chen.2025-sequential}
, including for both average and heterogeneous treatment effects. For a detailed review, we refer to \citet{Colnet.2024}.
In these settings, the motivation for data combination is typically to leverage complementary information across datasets while assuming that treatment \textit{and} outcomes are observed in \textit{all} datasets. Hence, the works from this stream are different from our paper and \textit{not} applicable. One paper aims to relax the setting by assuming a combination where long-term outcomes are observed in only one dataset while treatments must still be observed in both \citep{Cai.2025-Individual}; however, this is still \textit{different} from our setting, where treatments are observed in only one dataset and long-term outcomes in the other dataset.

\textbf{Average long-term treatment effects (ALTEs):} A related but different setting aims at estimating \textit{average} long-term treatment effects (ALTEs). In this setting, ALTEs are typically identified by mediating the treatment effects via surrogate variables under suitable assumptions \citep[e.g.,][]{Prentice.1989, Frangakis.2002, Lauritzen.2004}. One approach is to build a ``surrogate index'' (i.e., a group of short-term endpoints) as the proxy for long-term outcomes under unconfoundedness \citep{Athey.2025, Chen.2023}. Subsequent works estimate ALTEs with modified settings and assumptions~\citep[e.g.,][]{Athey.2020a, Ghassami.2022, Goffrier.2023, Obradovic.2024,Cai.2024-latent, Imbens.2025, Huang.2026}. Other works also study ALTE estimation in the context of policy optimization tasks \citep[e.g.,][]{McDonald.2023, Park.2024bracketing, Yang.2024-Targeting, Yang.2024-LearningTheOptimalPolicy, Wang.2024, Saito.2024LOPE}. However, these works all target the ALTE, which only captures \textit{population-level} effects and which thus \textit{ignores} how long-term treatment effects vary across individuals or subpopulations. We summarize the different settings for long-term causal effects in Table~\ref{tab:lte-settings}.

\textbf{Heterogeneous long-term treatment effects (HLTEs):} There are some works studying HLTEs under alternative settings, such as \textit{dose–response} estimation~\citep{Singh.2022, Yang.2025-dose-response} (See Table~\ref{tab:lte-settings} for an overview of the long-term causal effects literature). These approaches focus on \textit{continuous} treatments and rely on different estimation strategies because of which they are \textit{not} applicable to the canonical HLTE setting with binary treatments considered in this paper. \citet{Huang.2024-Doingmore} constructs a noise-reduced surrogate index via separate imputation; however, this requires additional long-term variables other than $Y$ and is thus \emph{not} applicable to our setting. \citet{Chen.2025} proposes an EIF-based orthogonal HLTE estimator in the setting of Conditional Additive Equi-Confounding Bias~\citep{Ghassami.2022}. However, it requires observed treatment in both datasets and does not adjust for confounding in the canonical setting with the surrogacy assumption. Thus, it is \emph{not} applicable. Finally, these learners are \emph{not} designed to be robust to (i)~low treatment overlap and (ii)~low long-term outcome overlap in finite samples. 
 
 

 \textbf{Overlap-weighted estimators for ATEs/HTEs:} A separate line of work uses retargeted overlap weighting to stabilize ATE/HTE estimation in low-overlap regimes. \citet{Crump.2009} retarget the average treatment effect (ATE) to a weighted ATE on the well-supported subpopulation, thereby reducing finite-sample variance at the price of changing the estimand. For heterogeneous effects, R-learner~\citep{Nie.2021} implicitly weights its loss by the treatment propensity $\pi(1-\pi)$~\citep{Fisher.2023}, and \citet{Morzywolek.2023} generalize this to a general class of treatment-propensity-weighted orthogonal learners; \citet{Hess.2026} further extend treatment-overlap weighting to the longitudinal HTE estimation. However, these works are \emph{not} tailored for the HLTE setting, and \emph{do not} address the limited overlap in observing long-term outcomes.


\textbf{Research gap:} Existing methods for HLTE estimation fail to address the challenge due to (i)~limited treatment overlap and (ii)~limited long-term outcome issue in finite samples (see Table~\ref{tab:hlte-learners}). To the best of our knowledge, we are the first to provide a general class of orthogonal learners that is robust in low-overlap regimes. 

\vspace{-0.2cm}

\section{Problem formulation} 
\label{sec:formulation}

\textbf{Setup:} We follow the canonical long-term outcome setting \citep{Athey.2025, Chen.2023}, but adapt it from ALTEs to HLTEs. We thus consider two datasets: a short-term dataset $\mathcal{D}_1$ and a long-term historical observational dataset $\mathcal{D}_2$, with $N_1$ and $N_2$ units respectively.  $\mathcal{D}_1$ and $\mathcal{D}_2$ differ in terms of which variables are observed (see Figure~\ref{fig:concept}). 

Let $A \in \{0,1\}$ be the treatment variable, $X\in \mathcal{X}\subset\mathbb{R}^{d_x}$ be pre-treatment covariates, $S \in \mathcal{S}\subset \mathbb{R}^{d_s}$ be the short-term surrogate variable, $Y\in\mathbb{R}$ be the long-term outcome, and $R \in \{0,1\}$ be the dataset indicator, denoting whether the observation belongs to $\mathcal{D}_1$ ($R=0$) or $\mathcal{D}_2$ ($R=1$).

The covariates $X$ and short-term surrogates $S$ are observed in both $\mathcal{D}_1$ and $\mathcal{D}_2$. The long-term outcome $Y$ is observed only in $\mathcal{D}_2$. The treatment $A$ is only observed in the short-term dataset $\mathcal{D}_1$ but \textit{not} in $\mathcal{D}_2$ (e.g., because this is a new treatment that was not applied previously). Formally, we can combine $\mathcal{D}_1$ and $\mathcal{D}_2$ to build a unified dataset with $N=N_1+N_2$ units:
$
    \mathcal{D}=\mathcal{D}_1 \cup \mathcal{D}_2=\{(X_i, (1 - R_i)A_i, S_i, R_i, R_iY_i)\}_{i=1}^N$.

\textbf{Nuisance functions:} Our method relies on a set of nuisance functions: 

\vspace{-0.5cm}
\footnotesize
\begin{align}
    &\pi_s(s,x)=\Pr(A=1\mid S = s,X = x,R = 0), \;\;\; \pi(x)=\Pr(A=1\mid X = x,R = 0),\\
    &\rho_s(s,x)=\Pr(R=1\mid S = s,X = x),  \;\;\;   \rho(x)=\Pr(R=1\mid X = x),\\
    &h(s,x)=\mathbb E[Y\mid S = s,X = x,R = 1],  \;\;\; \mu(a,x)=\mathbb E \big[h(S,X)\mid A = a,X = x,R = 0\big] .
\end{align}
\vspace{-0.5cm}
\normalsize

We let $\eta$ denote the complete set of nuisance functions, i.e.,
\begin{equation}
\label{eq:def-eta}
    \eta = (\pi(x), \pi_s(s,x), \rho(x),\rho_s(s,x),h(s,x), \mu(a,x)).
\end{equation}
\textbf{Causal estimand:} We build upon the potential outcome framework~\citep{Rubin.1972} and aim to estimate the HLTE in the short-term dataset, given by
\begin{align}
\label{tau0-def}
    \tau^0(x)=\E\left[Y(1)-Y(0) \mid X=x,R=0\right].
\end{align}
This estimand captures how long-term treatment effects vary at the individual or subgroup level as a function of covariates $X$, thus allowing for fine-grained heterogeneity in long-term responses to inform personalized decision-making. In contrast, population-level estimands such as the ALTE as in \citet{Athey.2025} target average effects across the entire population and do not characterize which individuals or subpopulations benefit most from the treatment.


\textbf{Identifiability: } We use the following standard assumptions from~\citet{Athey.2025} to ensure the identification. 
\begin{assumption}[Consistency]
\label{assumption-consistency} \emph{For all $a \in \{0,1\}$, we have $S(a)=S, \,Y(a)=Y$ whenever $A=a$.}
\end{assumption}
\begin{assumption}[Positivity] 
\label{assumption-positivity} \textit{The treatment overlap is nondegenerate: $0< P(A=1 \mid X=x, R=0)<1, \; \forall x\in \mathcal{X}$.}
\end{assumption}

\begin{assumption}[$\mathcal{D}_1$ unconfoundedness]
\label{assumption-unconf} \emph{For all $a \in \{0,1\}$, we have} $A \indep S(a), Y(a) \mid X, R=0$.
\end{assumption}

\begin{assumption}[Surrogacy~\citep{Prentice.1989}]
\label{assumption-surrogacy} \emph{The long-term effect is fully mediated by the surrogate variable in $\mathcal{D}_1$:} $A \indep Y \mid S, X,R=0$ and $0 < P(R=1 \mid S=s,X=x) < 1$.
\end{assumption}
\begin{assumption}[Long-term outcome comparability]
\label{assumption-compara} \emph{The distribution of $Y$ is invariant to whether it belongs to $\mathcal{D}_1$ or $\mathcal{D}_2$:} $R \indep Y \mid S, X$.
\end{assumption}
Assumptions~\ref{assumption-consistency}, \ref{assumption-positivity}, \ref{assumption-unconf} are standard in the observational causal inference literature~\citep{Rubin.1974, Imbens.1994, Kallus.2025, Athey.2025}. Assumption~\ref{assumption-surrogacy} is widely adopted in prior literature on long-term effects~\citep{Athey.2025, Chen.2023, Yang.2024-Targeting, Kallus.2025}. Assumption~\ref{assumption-compara} requires that the conditional distribution of $Y$ given $(S,X)$ is transportable between $\mathcal{D}_1$ and $\mathcal{D}_2$, which is a mild assumption that is common in most data combination tasks~\citep{Athey.2025, Chen.2023, Yang.2024-Targeting}. When all the above assumptions are satisfied, we call the problem the \emph{standard surrogacy model}~\citep{Chen.2023}. Under the standard surrogacy model, $h(S,X)=\mathbb E[Y\mid S, X, R = 1]$ is an unbiased proxy of $\mathbb{E}[Y(1) \mid S, X, R=0]$ and $\mathbb{E}[Y(0) \mid S,X,R=0]$ on $\mathcal{D}_1$\footnote{\citet{Athey.2025} refers to $h$ as the surrogate index.}. We can then identify the HLTEs via (see Appendix~\ref{app:identification} for a formal proof):
\begin{align}
    & \tau^0(x) =\mathbb{E}[Y(1)-Y(0) \mid X=x,R=0] \label{eq:identification-tau0} \\
              =&\mathbb{E}\bigl[h(S,X)\mid A=1,X=x,R=0\bigr] - \mathbb{E}\bigl[h(S,X)\mid A=0,X=x,R=0\bigr] = \mu(1,x)-\mu(0,x).
    \notag
\end{align}
\textbf{Note on relaxing assumptions:} Surrogacy is a strong assumption, which requires $S$ to fully mediate the effect of $A$ on $Y$, and $A \indep Y \mid S, X, R=0$ cannot be checked from $\mathcal{D}_1$ alone~\citep{Prentice.1989, Freedman.1992}; \citet{Athey.2025} argue that high-dimensional, multi-modal surrogates -- e.g., comprehensive job skills or multi-biomarker oncology panels -- make surrogacy substantively more plausible. We adopt both assumptions without relaxation, but make their impact transparent: in Appendix~\ref{app:bias-surrogacy-violation} we derive a pointwise bias decomposition for the T-learner (a direct instantiation of the identification) under joint violation (in analogy with the ALTE analysis of \citet{Athey.2025}), and Corollary~\ref{cor:sensitivity} turns this into a closed-form sensitivity bound computable from $\mathcal{D}_1$ data alone given a user-supplied bound on the residual direct effect.




\section{T-, RA-, and IPW-learners for HLTE estimation}
\label{sec:standard_learners}


In this section, we briefly introduce how the T-, RA-, and IPW-learners for the standard HTE estimation setting\footnote{This refers to the setting where there is only a single dataset that contains fully observed $(X,A,Y)$ with no unobserved confounding~\citep{Kunzel.2019}.} can be adapted to the HLTE setting (see Appendix~\ref{app:analogy} for details). Under the surrogacy model, $h$ is an unbiased proxy of the unobserved $Y$ in $\mathcal{D}_1$. As a result, we can simply replace $Y$ by $h(S,X)$ in the standard HTE estimators to construct the analogous HLTE learners. While new in our setting, these learners serve as natural baselines and provide useful intuition for the challenges of HLTE estimation, and thus motivate our \methods. We summarize these learners in Table~\ref{tab:hlte-learners}.

\begin{wraptable}{r}{0.5\textwidth}
\centering
\vspace{-0.3cm}
\caption{\textbf{Overview of meta-learners for HLTE estimation.} Our \methods are the only approaches that are both Neyman-orthogonal and robust to different types of low-overlap scenarios.}
\label{tab:hlte-learners}
\resizebox{0.5\textwidth}{!}{%
\begin{tabular}{l l
    >{\columncolor{shallowyellow}}c
    @{\hspace{8pt}}
    >{\columncolor{shallowblue}}c
    >{\columncolor{shallowblue}}c}
\toprule
\textbf{Meta-learner}
& \makecell[l]{\textbf{Analogy in}\\\textbf{standard HTE}}
& \makecell{\textbf{Neyman}\\\textbf{orthogonal}}
& \makecell{\textbf{Robust to low}\\\textbf{treatment overlap}}
& \makecell{\textbf{Robust to low}\\\textbf{outcome overlap}} \\
\midrule
\LTT
& T-learner~\citep{Kunzel.2019}
& \xmark & \xmark & \xmark \\
\LTRA
& RA-learner~\citep{Curth.2021}
& \xmark & \xmark & \xmark \\
\LTIPW
& IPW-learner~\citep{Curth.2021}
& \xmark & \xmark & \xmark \\
\midrule
\textbf{\DR} (\emph{ours}, $\omega\equiv 1$)
& DR-learner~\citep{Kennedy.2023}
& \cmark & \xmark & \xmark \\
\textbf{\TO} (\emph{ours}, $\omega=\pi^2(1-\pi)^2$)
& R-learner~\citep{Nie.2021}
& \cmark & \cmark & \xmark \\
\textbf{\LO} (\emph{ours}, $\omega=\rho$)
& --
& \cmark & \xmark & \cmark \\
\textbf{\DO} (\emph{ours}, $\omega=\pi^2(1-\pi)^2\rho$)
& --
& \cmark & \cmark & \cmark \\
\bottomrule
\end{tabular}
}
\vspace{-0.3cm}
\end{wraptable}

$\bullet$\,\textbf{Long-term T-learner (\LTT):} The simplest approach is to estimate the conditional response functions $\mu_a(x)=\E\left[h(S,x)|A=a, X=x,R=0\right], \;a\in\{0,1\}$. Then, we obtain the \emph{plug-in} HLTE estimator: $\hat\tau^0_{T}(x)=\hat\mu_1(x) - \hat\mu_0(x)$. We name it the long-term T-learner following the standard HTE setting~\citep{Kunzel.2019}. 
%
%
%

$\bullet$\,\textbf{Long-term RA-learner (\LTRA):} The RA-learner follows a two-stage pseudo-outcome regression. It fits estimates $(\hat{h}, \hat\mu_1, \hat\mu_0)$ in the first stage to construct the regression-adjusted pseudo-outcome$\mathcal{T}_\mathrm{RA}= A(\hat{h}(S,X)-\hat\mu_0(X)) + (1-A)(\hat\mu_1(X)-\hat{h}(S,X))$. In the second stage, it regresses the regression-adjusted pseudo-outcome on $\mathcal{D}_1$:
\begin{align}
\label{ra-loss}
    \mathcal{L}_{\mathrm{RA}}=\E\left[\1(R=0)(\mathcal{T}_\mathrm{RA} - g(X))^2\right].
\end{align}


$\bullet$ \,\textbf{Long-term IPW-learner (\LTIPW):} The IPW-learner uses inverse-propensity-based pseudo-outcome:
$\mathcal{T}_\mathrm{IPW} = \left(\frac{A}{\hat\pi(X)}-\frac{1-A}{1-\hat\pi(X)}\right) \hat{h}(S,X)$, and the second stage loss is defined as:
\begin{align}
\label{ipw-loss}
    \mathcal{L}_{\mathrm{IPW}}=\E\left[\1(R=0)(\mathcal{T}_\mathrm{IPW} - g(X))^2\right]
\end{align}

\textbf{Why do we need orthogonal learners:}
The \LTT is subject to plug-in bias~\citep{Kennedy.2022-Sempiparametric}. While the \LTRA and \LTIPW mitigate plug-in bias with a two-stage procedure, they are still sensitive to the nuisance estimation errors $\hat\eta-\eta$.  To overcome these limitations, current state-of-the-art two-stage methods for treatment effect estimation use the so-called \emph{orthogonal learners}~\citep{Chernozhukov.2018, Foster.2023}. A learner with loss function $\mathcal{L}(g, \eta)$ is defined as orthogonal if the Gateaux derivative of the loss w.r.t. the nuisance parameters vanishes at the truth:
\begin{align}
    D_\eta D_g \mathcal{L}(g, \eta)[\hat{g}-g, \hat\eta - \eta] = 0. 
\end{align}
Intuitively, orthogonality ensures that the gradient of the loss function is insensitive to perturbations in the nuisance parameters~\citep{Chernozhukov.2018}. This allows faster convergence of the target estimator even with a slow convergence of the nuisances. We prove this property for our \methods in Theorem~\ref{thm:error-bound}.

\section{Our \methods}
\label{sec:method}

\textbf{Main idea of our \methods:} Existing studies have shown that orthogonal learners without sample reweighting are usually hurt by low overlaps under finite samples~\citep{Fisher.2023, Frauen.2025-model, Frauen.2025-survival} (see Appendix~\ref{app:variance} for the variance analysis of the unweighted \DR). Our goal is to design learners that are both orthogonal and robust against low-overlap regimes. To do this, the key idea is to downweight the samples with (i)~low treatment or (ii)~low long-term outcome overlap. 
Conceptually, this corresponds to optimizing a weighted oracle risk of the form
\begin{equation*}
\label{eq:oracle-omega-weighted-loss}
\mathcal{L}_\omega^{\text{oracle}}(g)=\mathbb{E}\bigl[\omega(\pi(X), \rho(X))\cdot(\tau^0(X)-g(X))^2 \mid R = 0\bigr],
\end{equation*}
where $\omega(\pi(X), \rho(X))>0$ represents a \textit{weighting function} that \textit{retargets} the objective to downweight the samples with low overlap while preserving the target HLTE $\tau^0(x)$. A simple example of such weighting function is $\omega(X) = \pi(x)(1-\pi(x)) \rho(x)$, where $\pi(x)(1-\pi(x))$ downweights samples with low treatment overlap and $\rho(x)$ downweights samples with low long-term outcome overlap.

However, the loss $\mathcal{L}_\omega^{\text{oracle}}$ is an oracle objective because it depends on the true HLTE $\tau^0(x)$, which is our target estimand and thus unobserved in practice. Hence, our goal is to construct a loss function $\mathcal{L}_\omega(g;\eta)$ that is equivalent to this oracle objective and can be optimized using observed data. This requires a non-trivial derivation to ensure that the resulting loss (i) shares the same minimizer as the oracle loss and (ii)~is Neyman-orthogonal w.r.t. the nuisance functions.

\textbf{Overview:} In Section~\ref{subsec:ortho-weight-loss}, we orthogonalize the oracle weighted loss and describe the training algorithm. In Section~\ref{subsec:instantiation}, we motivate the weighted instantiations through optimal efficiency bound. Finally, in Section~\ref{subsec:theoretical-benefits}, we establish the theoretical benefits of orthogonality and weighting.

\subsection{Orthogonalizing the weighted oracle loss}
\label{subsec:ortho-weight-loss}

In this section, we derive the orthogonal loss that has the same minimizer as the weighted oracle loss, under true nuisance $\eta$ and model class $\mathcal{G}$. Following a similar approach as in~\citep{Morzywolek.2023}, we first derive the EIF for the \emph{$\omega$-weighted} average long-term treatment effect (WALTE)\footnote{See Appendix~\ref{app:theory} for a theoretical background.}, which is defined as:
\begin{equation}
\label{eq:walte-def}
    \tau^0_\omega =\frac{\E\left[\omega(X)\tau^0(X)\mid R=0\right]}{\E\left[\omega(X)\mid R=0\right]}:= \frac{N_\omega}{D_\omega}.
\end{equation}
Here, for simplicity, we denote $\omega(x)=\omega(\pi(x),\rho(x))$. See Appendix~\ref{app:EIF-weighted-ATE} for the EIF.


Since the Gateaux derivative of the EIF w.r.t. the nuisances vanishes at the truth~\citep{Chernozhukov.2018}, an EIF-based pseudo-outcome naturally produces a Neyman-orthogonal loss, which we show below.


\begin{contributionbox}
\textbf{LT-O-learners:} Let $Z$ be $(X,A,S)$ when $R=0$ and $(X, S, Y)$ when $R=1$. Based on the derived EIF, we propose our orthogonal loss as follows:
\begin{equation}
\label{eq:weighted-ortho-loss-tau0}
\mathcal{L}_\omega(g, \eta) = \mathbb{E} \left[ \omega^*(Z;\eta) g(X)^2 -2\mathcal{T}_\mathrm{LT}(Z;\eta) \cdot g(X)\right],
\end{equation}
where $\omega^*$ is the \textit{LT weighting function}, defined as:
\footnotesize
\begin{align}
    \label{omega-star-def}
    \omega^*(Z;\eta) = \1(R=0) \omega(X) + \Omega(Z;\eta), \text{ where }
    \Omega(Z;\eta) =\1(R=0)\frac{\partial \omega}{\partial \pi} (A - \pi)
    + (1-\rho) \frac{\partial \omega}{\partial \rho} (R - \rho), 
\end{align}
\normalsize
and the \textit{LT pseudo-outcome} $\mathcal{T}_\mathrm{LT}(Z;\eta)$ is
\footnotesize
\begin{align}
\label{Tau-LT-def}
    \mathcal{T}_\mathrm{LT}(Z;\eta) &=  \1(R=0) \omega(X) \hat{\tau}_\mathrm{AIPW}(Z;\eta)
     +  \1(R=1) \omega(X) \psi_\mathrm{obs}(Z;\eta) + (\mu(1,X)-\mu(0,X)) \Omega(Z;\eta),\\
    \text{with } & \hat{\tau}_\mathrm{AIPW}(Z;\eta) = \mu(1,X) - \mu(0,X) + \frac{A-\pi(X)}{\pi(X)(1-\pi(X))}(h(S,X) - \mu(A,X)) , \label{tau-AIPW-def} \\
    & \psi_\mathrm{obs} = \frac{1-\rho_s(S,X)}{\rho_s(S,X)} \left( \frac{\pi_s(S,X) - \pi(X)}{\pi(X)(1-\pi(X))} \right)
    \cdot\big( Y - h(S,X) \big) \label{psi-obs-def}
\end{align}
Here, $\hat{\tau}_\mathrm{AIPW}(Z;\eta)$ is the augmented IPW (AIPW) score, and $\psi_\mathrm{obs}(Z;\eta)$ is the bias-correction term. 
\end{contributionbox}
\normalsize
Intuitively, $\omega^*$ consists of a main part $\omega(x)$ and a residual-based $\Omega(Z;\eta)$\footnote{$\Omega(Z;\eta)$ can be viewed as a mean $0$ term that corrects the first-order bias of the weight estimator $\omega{(\hat\pi(X),\hat\rho(X))}$.} and thus $\omega^*\approx \omega$ on average when $R=0$. Take $\omega=\pi^2(1-\pi)^2\rho$ as an example: When either (a) treatment overlap is low (small $\pi(1-\pi)$) or (b) long-term outcome overlap is low (small $\rho$), the sample is downweighted in the pseudo-outcome regression in Eq.~(\ref{eq:weighted-ortho-loss-tau0}) via $\omega^*\approx \omega$. In fact, we have the following (see Appendix~\ref{app:proofs} for detailed proofs):
\begin{align}
    \E\left[\omega^*(Z;\eta)\mid X\right] = (1-\rho(X))\omega(X), \;\;\;
    \E\left[\mathcal{T}_\mathrm{LT}(Z;\eta) \mid X\right] = (1-\rho(X))\omega(X) \tau^0(X).  \notag
\end{align}
Hence, when we optimize $\mathcal{L}_\omega(g, \eta)$, we are always solving for $g^*(X)=\E\left[\mathcal{T}_{LT} \mid X\right] / \E\left[\omega^*\mid X\right]=\tau^0(X), \; \forall X\in \mathcal{X}$ regardless of the weighting function. 

Below, we show that the proposed loss function $\mathcal{L}_\omega(g, \eta)$ satisfies two crucial properties: (1) Neyman-orthogonality, and (2) equivalence with the oracle loss $\mathcal{L}_\omega^{\text{oracle}}(g)$.

\begin{theorem}[Neyman-orthogonality]
    \label{thm:orthogonality}
    The loss $\mathcal{L}_\omega(g, \eta)$ is orthogonal w.r.t. the nuisance $\eta$.
    \vspace{-0.1cm}
\end{theorem}
Theorem~\ref{thm:orthogonality} establishes that the functional derivative of $\mathcal{L}_\omega(g, \eta)$ is Neyman-orthogonal w.r.t. the nuisance $\eta$, i.e. $D_\eta D_g \mathcal{L}_\omega(g, \eta)[\hat{g}-g, \hat\eta - \eta] = 0.$ This means, nuisance estimation error $\delta_\eta=\hat\eta-\eta$ only has second-order impact on the gradient, which makes the learner robust against first-order nuisance estimation error.

\begin{theorem}[Oracle equivalence]
    \label{thm:equiv-loss}
    The minimizer of the loss $\mathcal{L}_\omega(g, \eta)$ within a functional space $\mathcal{G}$ is also the minimizer of the oracle weighted loss, namely:
    \begin{equation}
        \argmin_{g\in \mathcal{G}} \mathcal{L}_\omega(g, \eta) = \argmin_{g \in \mathcal{G}} \mathcal{L}_\omega^{\mathrm{oracle}}(g).
    \end{equation}
    A direct implication is that, if $\tau^0(\cdot)\in \mathcal{G}$, then minimizing our orthogonal loss $\mathcal{L}_\omega(g, \eta)$ yields the true HLTE estimator.
\end{theorem}
\begin{proof}
    See Appendix~\ref{app:proof-neyman} and \ref{app:proof-equiv} for the proofs of Theorems~\ref{thm:orthogonality} and \ref{thm:equiv-loss}, respectively.
\end{proof}
Theorem~\ref{thm:orthogonality} and Theorem~\ref{thm:equiv-loss} together show that the proposed loss $\mathcal{L}_\omega(g, \eta)$ is indeed the orthogonalized version of the oracle weighted loss $\mathcal{L}_\omega^{\text{oracle}}(g)$, and it preserves the same minimizer when $\mathcal{G}$ is well-specified. In practice, we minimize an empirical surrogate of $\mathcal{L}_\omega(g, \hat\eta)$ via $K$-fold cross-fitting~\citep{Chernozhukov.2018, Kennedy.2023, Morzywolek.2023, Foster.2023}: we partition $\mathcal{D}=\{Z_i\}_{i=1}^N$ into $K=5$ folds; for each fold, we fit the nuisances $\hat\eta=(\hat\pi, \hat\pi_s,\hat\rho,\hat\rho_s, \hat{h}, \hat\mu)$ on the remaining $K-1$ folds to evaluate $\hat\omega^*_i=\omega^*(Z_i;\hat\eta)$ and $\hat{\mathcal{T}}_i=\mathcal{T}_\mathrm{LT}(Z_i;\hat\eta)$, and then solve
\vspace{-0.3cm}
\begin{align}
\label{eq:empirical-loss}
    \hat{g}=\argmin_{g\in\mathcal{G}} \sum_{i=1}^N \hat\omega^*_i g(X_i)^2- 2\hat{\mathcal{T}}_i g(X_i).
\end{align}
\vspace{-0.7cm}
\subsection{Instantiations}
\label{subsec:instantiation}

\textbf{Is there an optimal $\omega$?} We derive the optimal weighting function $\omega_\mathrm{opt}$ that minimizes the efficiency bound of the WALTE in Eq.~(\ref{eq:walte-def}) (see Appendix~\ref{sec:optimality}). However, such $\omega_\mathrm{opt}(x)$ involves computing terms with expectations of the nuisances, which will break the key orthogonality of the loss. To achieve a tractable weighting function, we derive a closed-form upper bound to $\omega_\mathrm{opt}$ (see Appendix~\ref{sec:lower_bound}):
\begin{align}
\label{eq:upper-bound-optimal-weight}
\omega_\mathrm{opt}(x) \;\leq\; \left(\frac{C_t(x)^2}{\pi(x)(1-\pi(x))} + \frac{C(x)^2}{\rho(x)\,\pi(x)^2(1-\pi(x))^2}\right)^{\!-1} \lesssim \rho(x)\,\pi(x)^2(1-\pi(x))^2.
\end{align}
From Eq.~(\ref{eq:upper-bound-optimal-weight}), we can see the variance comes from two sources: (1)~treatment overlap, which contributes a $\pi(x)^2 (1 - \pi(x))^2$ term, and (2)~outcome overlap, which contributes a $\rho(x)$ term. This naturally motivates the choices of $\omega$ summarized below (see Appendix~\ref{app:instantiations} for full derivations.):

\begin{contributionbox}

\begin{itemize}[leftmargin=*, itemsep=2pt]
\item[\textcolor{Orange}{$\blacktriangle$}] \textbf{\DR} ($\omega\equiv 1$): no overlap weighting; recovers the EIF of ALTE~\citep{Athey.2025, Chen.2023}.
\item[\textcolor{LimeGreen}{\ding{58}}] \textbf{\LO} ($\omega=\rho$): targets low long-term outcome overlap.
\item[\textcolor{ForestGreen}{$\bigstar$}] \textbf{\TO} ($\omega=\pi^2(1-\pi)^2$): targets low treatment overlap.
\item[\textcolor{PineGreen}{$\blacksquare$}] \textbf{\DO} ($\omega=\pi^2(1-\pi)^2\rho$): targets joint treatment and outcome overlap.
\end{itemize}
\end{contributionbox}
\emph{Problems of the \DR:} Although orthogonal, the \DR does not apply any sample-weighting strategy, which leads to large finite-sample variance~\citep{Fisher.2023}. We analyze the variance of the DR pseudo-outcome (defined in Appendix~\ref{app:DR-LT-learner}) and show why it inflates under low overlap in Appendix~\ref{app:variance}. As a result, the \DR is not robust to low overlap, motivating the weighted instantiations.

\textbf{Note:} We point out that $\omega_\mathrm{opt}$ is only optimal in the sense of minimizing the efficiency bound, \emph{not} the HLTE estimation error. In fact, finding the optimal weight for general heterogeneous treatment effects remains an open question. This motivates a class of learners with customizable weights. In practice, we recommend choosing a suitable weight based on the estimated overlap scores $\hat\pi$ and $\hat\rho$.


\subsection{Theoretical benefits}
\label{subsec:theoretical-benefits}

In this section, we show the benefits of orthogonality and weighting with two theorems:
\begin{theorem}[Error bound\footnote{The proof adapts the technique of \citet{Morzywolek.2023} to the long-term surrogate setting.}]
\label{thm:error-bound}
    Let $\hat{g}$ be the minimizer of the empirical loss $\mathcal{L}_{\omega,n}(g, \hat\eta)$ and $g_0$ be the minimizer of $\mathcal{L}_\omega(g, \eta)$ under the true nuisance. Under the assumptions in Appendix~\ref{app:error-bound-assumptions},
$    \|\hat{g}-g_0\|^2 \lesssim R_g + \|\hat\mu-\mu\|_4^2(\|\hat\rho-\rho\|_4^2 + \|\hat\pi-\pi\|_4^2)
    + \|\hat{h}-h\|_4^2( \|\hat\pi_s-\pi_s\|_4^2 + \|\hat\rho_s-\rho_s\|_4^2) + \|\hat\pi-\pi\|_4^4 + \|\hat\rho-\rho\|_4^4$,
    where $R_g=\mathcal{L}_\omega(\hat{g}, \hat{\eta}) - \mathcal{L}_\omega(g_0, \hat{\eta})$ is the quasi-oracle rate.
\end{theorem}
Furthermore, under a linear class $\mathcal{G}=\{g(x)=\theta^\top x:\theta\in\mathbb{R}^d\}$ with $\tau^0\in\mathcal{G}$, $\hat\theta$ admits a closed-form solution (Appendix~\ref{app:linear}). To our knowledge, no existing overlap-weighted orthogonal learner provides an asymptotic variance analysis under a parametric function class; the following result fills this gap:
\begin{theorem}[Variance stabilization]
\label{thm:linear-asymvar-main}
Under standard Z-estimator regularity conditions and Assumption~\ref{ass:low-overlap}: \emph{(a)}~$\sqrt n(\hat\theta-\theta_0)\xrightarrow{d}\mathcal{N}\bigl(0,\, \mathcal{V}(\omega)\bigr)$, where $\mathcal{V}(\omega)$ is the asymptotic covariance matrix of $\hat\theta$ (see Appendix~\ref{app:linear}); \emph{(b)}~for any sequence of DGPs with overlap score $0<\epsilon\to0$, $v^\top \mathcal{V}(\mathbf 1)\,v\to\infty$ for every direction $v\neq 0$, while $v^\top \mathcal{V}\bigl(\pi(1-\pi)\rho\bigr)\,v$ stays bounded.
\end{theorem}
\begin{proof}
    See Appendix~\ref{app:proof-error-bound} and \ref{app:linear-asymvar} for proofs of Theorems~\ref{thm:error-bound} and \ref{thm:linear-asymvar-main}, respectively.
\end{proof}

Theorem~\ref{thm:error-bound} shows that, due to Neyman-orthogonality, nuisance error enters only through higher-order interaction terms, so the HLTE estimator preserves its convergence rate even when nuisances are fit by flexible machine learning at slow nonparametric rates~\citep{Chernozhukov.2018, Foster.2023}. Theorem~\ref{thm:linear-asymvar-main} then shows that, for linear target models, the weight $\pi(1-\pi)\rho$ suffices to absorb the inverse-overlap blow-up of unweighted \DR's variance; stability of the \DO with a stronger weight $\pi^2(1-\pi)^2\rho$ follows naturally.

\vspace{-0.2cm}

\section{Experiments}
\label{sec:experiments}

In this section, we conduct experiments to empirically validate our theoretical results and to demonstrate the benefits of our proposed \methods in low-overlap settings. The experiments serve two main purposes: (1)~\textit{to show that the limited overlap is a primary source of estimation instability under finite samples in standard learners}; and (2)~\textit{to demonstrate that our proposed learner addresses this challenge through Neyman-orthogonality and our overlap weighting}.


\textbf{Baselines:} We evaluate the four instantiations of our \methods (see Table~\ref{tab:hlte-learners}). 
There is no direct benchmark specifically designed for our HLTE setting. Therefore, as baselines, we derive and adapt standard meta-learners, namely, a \LTT, \LTRA, and \LTIPW to our setting (see Section~\ref{sec:standard_learners}; details in Appendix~\ref{app:baselines}). The \DR is the unweighted special case of our \methods; we report it alongside the weighted variants but do not consider it part of our core contribution.
In our experiments, all meta-learners are instantiated using the same neural network architecture and identical hyperparameters (see Appendix~\ref{app:implementation}), thus ensuring that performance differences reflect the effectiveness of the respective adjustment strategies rather than model architecture. We measure the \textit{precision in estimation of heterogeneous effects (PEHE)}. Each experiment is repeated 5 times with different random seeds.


\begin{table*}
    \vspace{-0.5cm}
    \centering
    \begin{minipage}[t]{0.48\textwidth}
        \centering
        \vspace{0pt}
        \caption{\textbf{Synthetic experiments.} Mean $\pm$ std. dev. of the PEHE [$\downarrow$] across 5 runs. Improv. denotes the percentage improvement of the best LT-O-learner over the best baseline.}
        \label{tab:synthetic-main}
        \resizebox{\textwidth}{!}{%
            \begin{tabular}{@{}l l l l l@{}}
                \toprule \midrule
                \textbf{Overlap type} & $\mathbf{\mathcal{D}_{*}}$ & $\mathbf{\mathcal{D}_{t}}$ & $\mathbf{\mathcal{D}_{o}}$ & $\mathbf{\mathcal{D}_{t+o}}$ \\ \midrule
                \LTT           & $0.22 \pm 0.02$ & $0.43 \pm 0.10$ & $0.20 \pm 0.02$ & $0.26 \pm 0.03$ \\
                \LTRA          & $0.21 \pm 0.02$ & $0.56 \pm 0.07$ & $0.18 \pm 0.03$ & $0.67 \pm 0.08$ \\
                \LTIPW         & $0.23 \pm 0.03$ & $0.70 \pm 0.23$ & $0.19 \pm 0.04$ & $1.27 \pm 0.78$ \\
                \midrule
                \DR (\textit{ours})   & $0.03 \pm 0.01$ & $0.60 \pm 0.61$ & $0.23 \pm 0.17$ & $1.08 \pm 0.55$ \\
                \TO (\textit{ours})   & $\mathbf{0.03 \pm 0.01}$ & $0.09 \pm 0.01$ & $0.29 \pm 0.29$ & $0.23 \pm 0.18$ \\
                \LO (\textit{ours})   & $0.05 \pm 0.03$ & $0.45 \pm 0.16$ & $0.08 \pm 0.02$ & $0.39 \pm 0.19$ \\
                \DO (\textit{ours}) & $0.03 \pm 0.01$ & $\mathbf{0.07 \pm 0.02}$ & $\mathbf{0.07 \pm 0.01}$ & $\mathbf{0.10 \pm 0.02}$ \\ \midrule
                \textbf{Improv.}    & \textcolor{lightgreen}{84.9\%} & \textcolor{lightgreen}{84.1\%} & \textcolor{lightgreen}{62.2\%} & \textcolor{lightgreen}{61.2\%} \\ \bottomrule
                \end{tabular}
        }
    \end{minipage}%
    \hfill 
    \begin{minipage}[t]{0.48\textwidth} 
    \vspace{0pt}
        \centering
        \caption{\textbf{Experiments with real-world data (GAIN):} Mean $\pm$ std pseudo-PEHE [$\downarrow$] across 5 runs.  Improv. denotes the percentage improvement of the best LT-O-learner over the best baseline.}
        \label{tab:real-world-main}
        \resizebox{\textwidth}{!}{%
        \begin{tabular}{@{}l cc cc@{}}
            \toprule \midrule
            & \multicolumn{2}{c}{\textbf{Earnings (\$)}} & \multicolumn{2}{c}{\textbf{Employment rate}} \\
            \cmidrule(lr){2-3} \cmidrule(lr){4-5}
            \textbf{Learner} & $Q_{13:36}$ & $Q_{30:36}$ & $Q_{13:36}$ & $Q_{30:36}$ \\ \midrule
            \LTT   & $13.98 \pm 0.05$ & $15.71 \pm 0.06$ & $0.170 \pm 0.001$ & $0.205 \pm 0.001$ \\
            \LTRA  & $13.54 \pm 0.01$ & $15.50 \pm 0.09$ & $0.169 \pm 0.001$ & $0.205 \pm 0.001$ \\
            \LTIPW & $23.63 \pm 8.59$ & $22.89 \pm 5.74$ & $1.244 \pm 1.743$ & $1.111 \pm 1.318$ \\
            \midrule
            \DR (\textit{ours}) & $18.88 \pm 7.08$ & $19.44 \pm 5.98$ & $0.208 \pm 0.055$ & $0.921 \pm 1.144$ \\
            \TO (\textit{ours}) & $\mathbf{10.86 \pm 0.13}$ & $12.55 \pm 0.19$ & $0.120 \pm 0.002$ & $0.153 \pm 0.001$ \\
            \LO (\textit{ours}) & $13.26 \pm 2.73$ & $13.47 \pm 4.76$ & $0.130 \pm 0.037$ & $0.172 \pm 0.111$ \\
            \DO (\textit{ours}) & $10.93 \pm 0.21$ & $\mathbf{12.16 \pm 0.11}$ & $\mathbf{0.119 \pm 0.003}$ & $\mathbf{0.150 \pm 0.003}$ \\ \midrule
            \textbf{Improv.} & \textcolor{lightgreen}{19.8\%} & \textcolor{lightgreen}{21.5\%} & \textcolor{lightgreen}{29.4\%} & \textcolor{lightgreen}{26.5\%} \\ \bottomrule
        \end{tabular}
        }
    \end{minipage}
    \vspace{-0.5cm}
\end{table*}

\underline{\textbf{Experiments with synthetic data:}} We evaluate the performances across multiple synthetic datasets. We simulate 4 scenarios: $\mathcal{D}_*, \mathcal{D}_t, \mathcal{D}_o, \mathcal{D}_{t+o}$, which correspond to no overlap issue, low treatment overlap only, low long-term outcome overlap only, and mixed overlap issue (see Appendix~\ref{app:synthetic} for details on the data-generating process). 
\underline{\emph{Results:}} The results are in Table~\ref{tab:synthetic-main}. We find: \textbf{(1)}~As suggested by our theory, our \methods outperform all baselines on $\mathcal{D}_t, \mathcal{D}_o, \mathcal{D}_{t+o}$, with $84.1\%$, $62.2\%$, and $61.2\%$ improvements over the best baseline. Notably, the \DR performs well on $\mathcal{D}_*$ but worse than baselines under overlap regimes. \textbf{(2)}~The \DO dominates every low overlap regime, validating the joint weight $\omega=\pi^2(1-\pi)^2\rho$ derived from the variance lower bound (Appendix~\ref{sec:optimality}); the \TO and \LO remain competitive within their target regimes. \textbf{(3)}~Overall, our \methods, based on overlap weighting, consistently outperform baselines, highlighting the benefits of Neyman-orthogonality and sample weighting in low-overlap settings. In Appendix~\ref{app:pehe-stratified}, we further report PEHE stratified by overlap level to confirm that our \methods do not trade accuracy on low-overlap samples for overall stability.

\begin{wrapfigure}{r}{0.3\textwidth} 
  \centering
  \vspace{-0.2cm}
  \captionsetup{width=\linewidth} 
  \includegraphics[width=\linewidth]{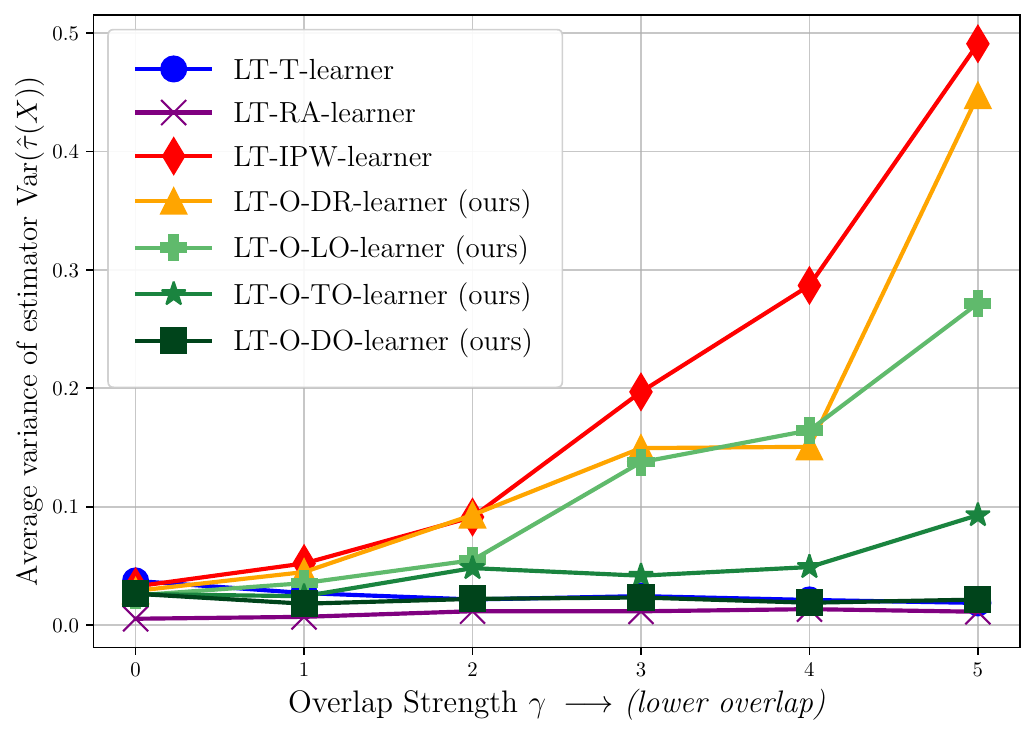}
  \caption{\textbf{Variance of estimator across different overlap scenarios ($\gamma$):}
  Low overlap leads to a high finite-sample variance of the DR-learner, while our \methods maintain a low variance.}
  \label{fig:estimator-variance-overlap}
  \vspace{-0.4cm}
\end{wrapfigure}

\textbf{$\bullet$\,Overlap-induced high variance:} We vary the parameter $(\gamma_\pi, \gamma_\rho) = (\gamma, 0.5\gamma)$ to gradually worsen the overlap in the data-generating process of $\mathcal{D}_{t+o}(\gamma)$. For each $\gamma$, we repeat the experiment using Monte Carlo simulation for $5$ runs and then calculate the variance of the resulting estimators, i.e., $V(\gamma)=\E_X\left[\var_{\mathcal{D}}(\hat\tau(X)) \mid \text{overlap}=\gamma\right]$.
\underline{\textit{Results: }} Figure~\ref{fig:estimator-variance-overlap} shows that the variance of the \DO remains low even under very-low-overlap regions, indicating it is more stable compared to the \DR. This result confirms that our novel overlap weighting strategy is effective in reducing the variance of the pseudo-outcomes.

\begin{wrapfigure}{r}{0.6\textwidth}
    \centering
    \vspace{-0.3cm}
    \captionsetup{width=\linewidth}
    \includegraphics[width=\linewidth]{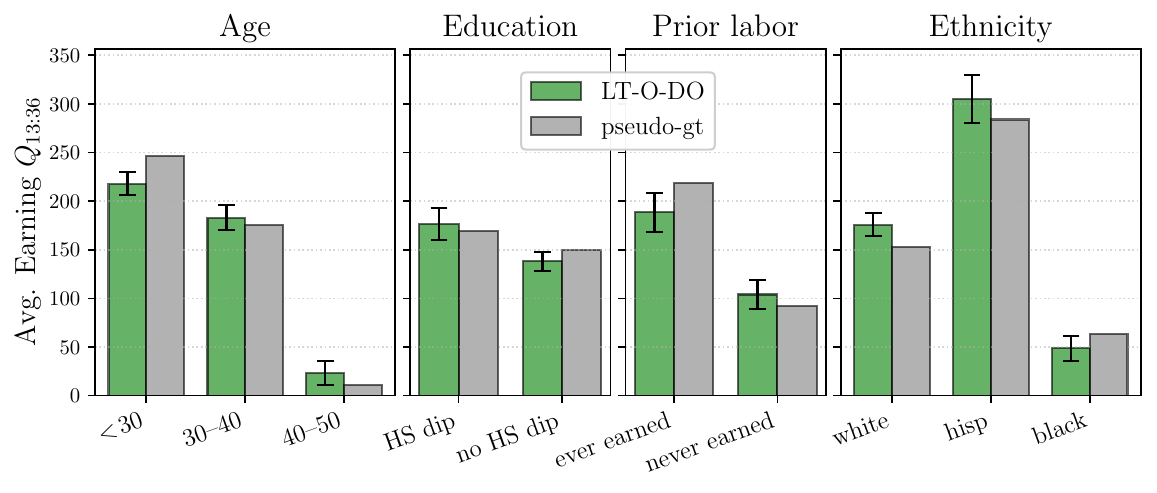}
    \caption{\textbf{Heterogeneity on GAIN.} Subgroup HLTEs on $Q_{13:36}$ earnings of the \DO against the pseudo ground-truth. Our method captures the known heterogeneity across age, education, prior labour status, and ethnicity.}
    \label{fig:gain-subgroup}
    \vspace{-0.3cm}
\end{wrapfigure}

\underline{\textbf{Experiments with real-world data:}} We use our \methods to evaluate the long-term effects of job training on labour market outcome with GAIN~\citep{Hotz.2006}, following~\citep{Athey.2025, Kallus.2025}. The dataset records quarterly earnings for 9 years after randomized training assignment on four locations: Riverside, San Diego, Los Angeles, and Alameda. As in \citet{Athey.2025}, we drop $Y$ from the Riverside subset and treat it $\mathcal{D}_1$ and applied rejection sampling to simulate low treatment overlap. We consider two long-term outcomes: average quarterly \textit{earnings} (\$) and \textit{employment rate}, each averaged over quarters $13$--$36$ (full long-run horizon) and $30$--$36$ (late horizon, where surrogacy is most challenging). Since the ground-truth HLTE is unavailable, we use the dropped $Y$ from $\mathcal{D}_1$ to compute the pseudo ground-truth (See Appendix~\ref{app:gain} for details).
\underline{\textit{Results:}} Table~\ref{tab:real-world-main} reports the mean $\pm$ std of pseudo-PEHE across runs. Across all four outcome--horizon combinations, our weighted \methods consistently outperform the baselines: the \DO attains the lowest PEHE in three columns and ties with the \TO on the remaining one, with relative $19.8\%$--$29.4\%$ improvements. The \DO and \TO also exhibit lower std than the un-weighted \DR and \LTIPW (e.g., $\pm0.119$ vs.\ $\pm5.987$ on $Q_{30:36}$ earnings), which confirms the robustness of our overlap weighting using real-world data. In Appendix~\ref{app:bias-surrogacy-violation-exp}, we additionally analyze sensitivity to surrogacy violation by varying the surrogate dimension $T$; the \DO remains the most accurate across all $T$.

\textbf{$\bullet$\, Subgroup-level heterogeneity:} We compare subgroup-level HLTE of the \DO against the pseudo ground-truth on $Q_{13:36}$ earnings (Figure~\ref{fig:gain-subgroup}). Across all four partitions---age, education, prior labour attachment, and ethnicity---the \DO recovers known heterogeneity patterns of the GAIN program~\citep{Hotz.2006}: training effects decline sharply with age, are slightly larger for high-school graduates and individuals with prior earnings, and vary substantially across ethnic subgroups. This qualitative agreement indicates that the predictions of our \methods are substantively sensible on real labour-market data, beyond aggregate PEHE.

\textbf{Conclusion:} We develop a novel class of orthogonal meta-learners for estimating HLTEs. Through theorectical analysis and empirical experiments, we demonstrate that our method achieves superior robustness in low-overlap regimes.

\bibliographystyle{plainnat}
\bibliography{references}

\newpage
\appendix

\section{Extended related work}
\label{app:extended-related-work}

\subsection{Comparison of settings for long-term causal effect estimation}
\label{app:lte-settings-comparison}

\providecommand{\cmarkk}{\ding{51}}
\providecommand{\xmarkk}{\ding{55}}
\providecommand{\cmarkgrey}{\textcolor{gray}{\ding{51}}}

The literature on long-term causal effect estimation differs along three axes: (i)~which variables are observed in the short-term ($\mathcal{D}_1$) and long-term ($\mathcal{D}_2$) datasets, (ii)~the identification mechanism used to bridge $S$ and $Y$, and (iii)~the target estimand. \Cref{fig:canonical-surrogacy} depicts the canonical surrogacy setting, and \Cref{tab:lte-settings} summarizes representative methods grouped by their high-level identification setting: latent unconfoundedness, equi-confounding, proximal data fusion, and surrogacy. Within the surrogacy setting in which $A$ is unobserved in $\mathcal{D}_2$, ours is the only method targeting the heterogeneous long-term treatment effect while simultaneously achieving Neyman orthogonality and robustness to low overlap.

\begin{figure}[h]
\centering
\begin{tikzpicture}[
  observed/.style={circle, draw, minimum size=7mm, inner sep=1pt, font=\small},
  latent/.style={circle, draw, dashed, fill=gray!10, minimum size=7mm, inner sep=1pt, font=\small},
  arr/.style={->, >=latex, thick},
  surr/.style={->, >=latex, thick, dotted}
]
\node[observed] (X1) at (1.5, 1.6) {$X$};
\node[observed] (A1) at (0, 0)     {$A$};
\node[observed] (S1) at (1.5, 0)   {$S$};
\node[latent]   (Y1) at (3, 0)     {$Y$};
\draw[arr] (X1) -- (A1);
\draw[arr] (X1) -- (S1);
\draw[arr] (X1) -- (Y1);
\draw[arr] (A1) -- (S1);
\draw[arr] (S1) -- (Y1);
\draw[surr] (A1) to[bend right=35] node[midway, draw=none, fill=white, inner sep=2pt] {\textcolor{red}{\large $\boldsymbol{\times}$}} (Y1);
\node at (1.5, -1.6) {$\mathcal{D}_1$: short-term experimental};

\node[observed] (X2) at (7, 1.6) {$X$};
\node[observed] (S2) at (6.2, 0) {$S$};
\node[observed] (Y2) at (7.8, 0) {$Y$};
\draw[arr] (X2) -- (S2);
\draw[arr] (X2) -- (Y2);
\draw[arr] (S2) -- (Y2);
\node at (7, -1.6) {$\mathcal{D}_2$: long-term observational};
\end{tikzpicture}
\caption{\textbf{Causal diagram for the canonical surrogacy setting.} The short-term experimental dataset $\mathcal{D}_1$ observes $(X, A, S)$ while $Y$ is unobserved (dashed node); the long-term observational dataset $\mathcal{D}_2$ observes $(X, S, Y)$ with treatment $A$ absent. The dotted arrow $A \to Y$ marked with $\times$ indicates that the surrogacy assumption rules out a direct effect of $A$ on $Y$ given $(S, X)$.}
\label{fig:canonical-surrogacy}
\end{figure}

\begin{table}[h]
\centering
\caption{\textbf{Comparison of settings for long-term causal effect estimation.} Methods are grouped by their high-level identification setting. \cmarkk{}/\cmarkgrey{} indicates that a variable is observed; for the treatment columns, \cmarkk{} denotes binary/categorical treatment and \cmarkgrey{} denotes continuous treatment. \xmarkk{} indicates that the variable is not observed. \cmark{}/\xmark{} in the Benefits columns indicate whether each method satisfies the corresponding property.}
\label{tab:lte-settings}
\resizebox{\textwidth}{!}{%
\begin{tabular}{l l c c c c c c c l l c c}
\toprule
\multirow{2}{*}{\textbf{Setting}}
& \multirow{2}{*}{\textbf{Paper}}
& \multicolumn{3}{c}{\textbf{Short-term} $\mathcal{D}_1$}
& \multicolumn{4}{c}{\textbf{Long-term} $\mathcal{D}_2$}
& \multirow{2}{*}{\makecell[l]{\textbf{Identification}\\\textbf{through}}}
& \multirow{2}{*}{\textbf{Target}}
& \multicolumn{2}{c}{\textbf{Benefits}} \\
\cmidrule(lr){3-5}\cmidrule(lr){6-9}\cmidrule(lr){12-13}
& & $X$ & $A$ & $S$ & $X$ & $A$ & $S$ & $Y$ & & & \makecell{Neyman\\orthogonality} & \makecell{Robustness to\\low overlap} \\
\midrule
\multirow{2}{*}{\makecell[l]{Latent\\unconfoundedness}}
 & \citet{Chen.2023} (latent unconf.) & \cmarkk & \cmarkk    & \cmarkk & \cmarkk & \cmarkk    & \cmarkk & \cmarkk & Latent unconfoundedness & ALTE                  & \cmark & \xmark \\
 & \citet{Yang.2025-dose-response}    & \cmarkk & \cmarkgrey & \cmarkk & \cmarkk & \cmarkgrey & \cmarkk & \cmarkk & Latent unconfoundedness & HLTE (dose-response)  & \xmark & \xmark \\
\midrule
\multirow{2}{*}{\makecell[l]{Equi-\\confounding}}
 & \citet{Ghassami.2022} & \cmarkk & \cmarkk & \cmarkk & \cmarkk & \cmarkk & \cmarkk & \cmarkk & Additive equi-conf.$^{\dagger}$        & ALTE, ATT & \cmark & \xmark \\
 & \citet{Chen.2025}     & \cmarkk & \cmarkk & \cmarkk & \cmarkk & \cmarkk & \cmarkk & \cmarkk & CAECB                                  & HLTE      & \cmark & \xmark \\
\midrule
Proximal
 & \citet{Imbens.2025}   & \cmarkk & \cmarkk & \cmarkk & \cmarkk & \cmarkk & \cmarkk & \cmarkk & Outcome + selection bridge functions   & ALTE      & \cmark & \xmark \\
\midrule
\rowcolor{shallowyellow}
 & \citet{Athey.2025}                  & \cmarkk          & \cmarkk          & \cmarkk          & \cmarkk          & \xmarkk          & \cmarkk          & \cmarkk          & Surrogacy + comparability          & ALTE                & \cmark          & \xmark          \\
\rowcolor{shallowyellow}
 & \citet{Chen.2023} (stat.~surrogacy) & \cmarkk          & \cmarkk          & \cmarkk          & \cmarkk          & \xmarkk          & \cmarkk          & \cmarkk          & Surrogacy + comparability          & ALTE                & \cmark          & \xmark          \\
\rowcolor{shallowyellow}
 & \citet{Singh.2022}                  & \cmarkk          & \cmarkgrey       & \cmarkk          & \cmarkk          & \xmarkk          & \cmarkk          & \cmarkk          & Surrogacy + comparability (RKHS)   & Dose-response curve & \xmark          & \xmark          \\
\rowcolor{shallowyellow}
 & \citet{Yang.2024-Targeting}         & \cmarkk          & \cmarkk          & \cmarkk          & \cmarkk          & \xmarkk          & \cmarkk          & \cmarkk          & Surrogacy + comparability          & Policy learning     & \cmark          & \xmark          \\
\rowcolor{shallowyellow}
\multirow{-5}{*}{Surrogacy (canonical setting)}
 & \textbf{Ours}                       & \textbf{\cmarkk} & \textbf{\cmarkk} & \textbf{\cmarkk} & \textbf{\cmarkk} & \textbf{\xmarkk} & \textbf{\cmarkk} & \textbf{\cmarkk} & \textbf{Surrogacy + comparability} & \textbf{HLTE}       & \textbf{\cmark} & \textbf{\cmark} \\
\bottomrule
\end{tabular}
}
\par\vspace{2pt}
\footnotesize
$^{\dagger}$\citet{Ghassami.2022} additionally propose bespoke-IV and single-proxy proximal variants under the same data setting.
\end{table}

\clearpage
\newpage
\section{Additional experiments}
\label{app:more-experiments}

\subsection{Ablations w.r.t. Neyman orthogonality}
\label{app:ablation-orthogonality}
In this experiment, we analyze why Neyman-orthogonality is crucial for robustness. We use the synthetic $\mathcal{D}_{t+o}$ dataset (mixed overlap) and vary the sample size, where a lower sample size implies that the nuisance function estimation is more difficult. In addition to the existing baselines, we add two ablations: \textbf{weighted but non-orthogonal} versions of the LT-RA-learner and LT-DR-learner (see Appendix~\ref{app:baselines} for details). These ablations also employ the same overlap weighting $\omega = \pi^2(1-\pi)^2\rho$ as the \DO, allowing us thus to isolate the effect of Neyman-orthogonality. 

\underline{\textit{Results:}} The left panel of Figure~\ref{fig:neyman} shows that the orthogonal \DO consistently outperforms non-orthogonal ablations in finite-sample regimes. The right panel further illustrates how nuisance function estimation performance deteriorates as the sample size decreases. Taken together, these results show that our orthogonal \methods remain robust, thus effectively mitigating the bias induced by nuisance estimation errors.

\begin{figure}[h!]
    \centering
    \includegraphics[width=1.0\linewidth]{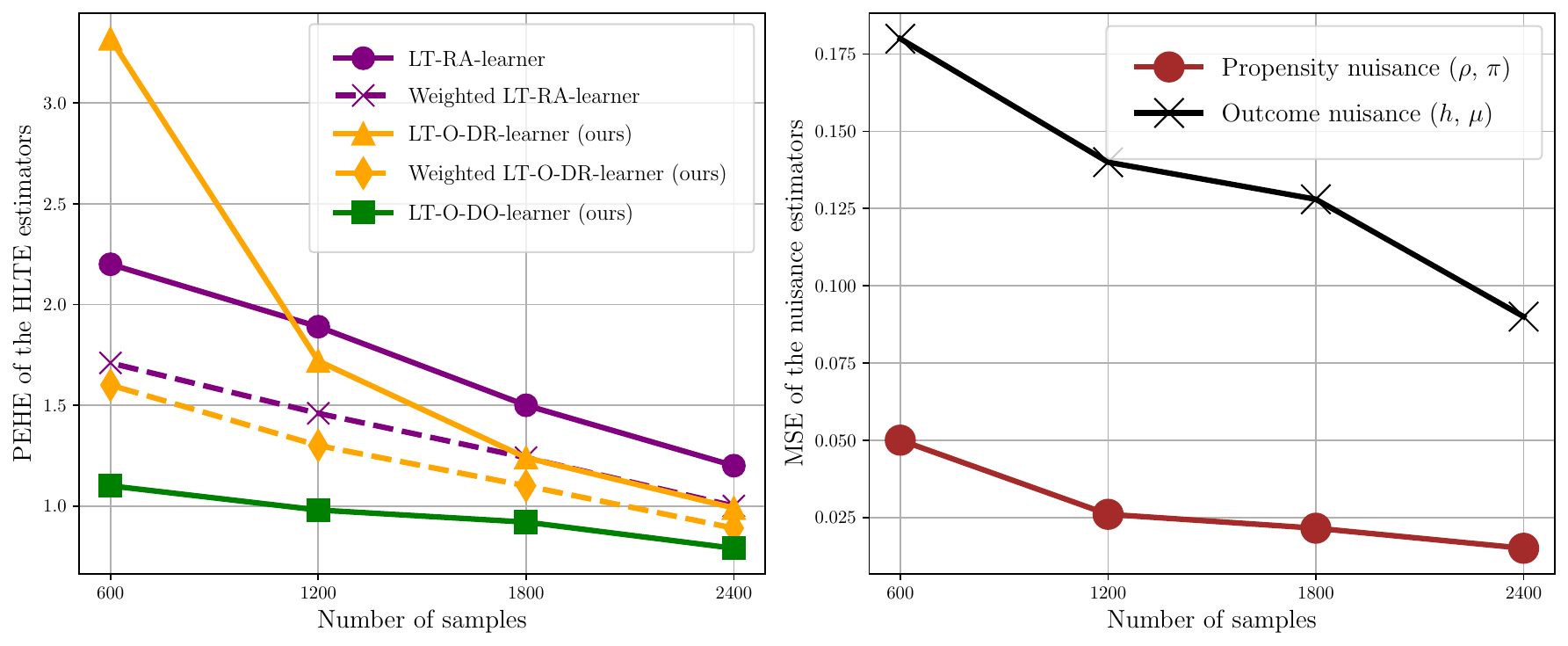}
    \caption{\textbf{Benefit of Neyman-orthogonality:} (\textit{Left})~PEHE vs. sample size. The orthogonal LT-O-DO-learner outperforms the weighted (but non-orthogonal) DR/RA-learner. (\textit{Right})~Mean squared error (MSE) of the nuisance estimation. Taken together, high errors in the nuisance functions at small sample sizes drive the instability of non-orthogonal baselines.}
    \label{fig:neyman}
    \vspace{-0.5cm}
\end{figure}

\subsection{PEHE stratified by overlap}
\label{app:pehe-stratified}
In this experiment, we take a granular examination of the performance of our \methods in the synthetic experiment. We stratify the test samples into five groups based on the overlap score $\text{overlap}(x) = \pi(x) (1 - \pi(x)) \rho(x)$. We then report the quantiles of the per-sample $\|\hat \tau(x) - \tau^*(x)\|$ within each group on the test set. We conduct this analysis for all four overlap scenarios, i.e., no overlap regmine ($\mathcal{D}_*$), low treatment overlap ($\mathcal{D}_t$), low outcome overlap ($\mathcal{D}_o$), and low treatment and outcome overlap ($\mathcal{D}_{t+o}$). 

\underline{\textit{Results:}} We plot the quantiles of the per-sample $\|\hat \tau(x) - \tau^*(x)\|$ within each overlap stratum in Figures~\ref{fig:syn-overlap-strat-dual}--\ref{fig:syn-overlap-strat-healthy}. Importantly, our \methods remain accurate even in low-overlap strata, despite the sample retargeting. For instance, in the low treatment overlap scenario $\mathcal{D}_t$ and joint low overlap scenario $\mathcal{D}_{t+o}$, our \DO and \TO achieve the lowest PEHE across all strata. In contrast, non-orthogonal baselines and the \DR have significantly higher PEHE. This shows that our \methods do \emph{not} trade accuracy in the low-overlap samples for the overall population stability. On the contrary, it stabilizes the estimation across all the samples, including those with low overlap.

\begin{figure}[h!]
    \centering
    \includegraphics[width=1.0\linewidth]{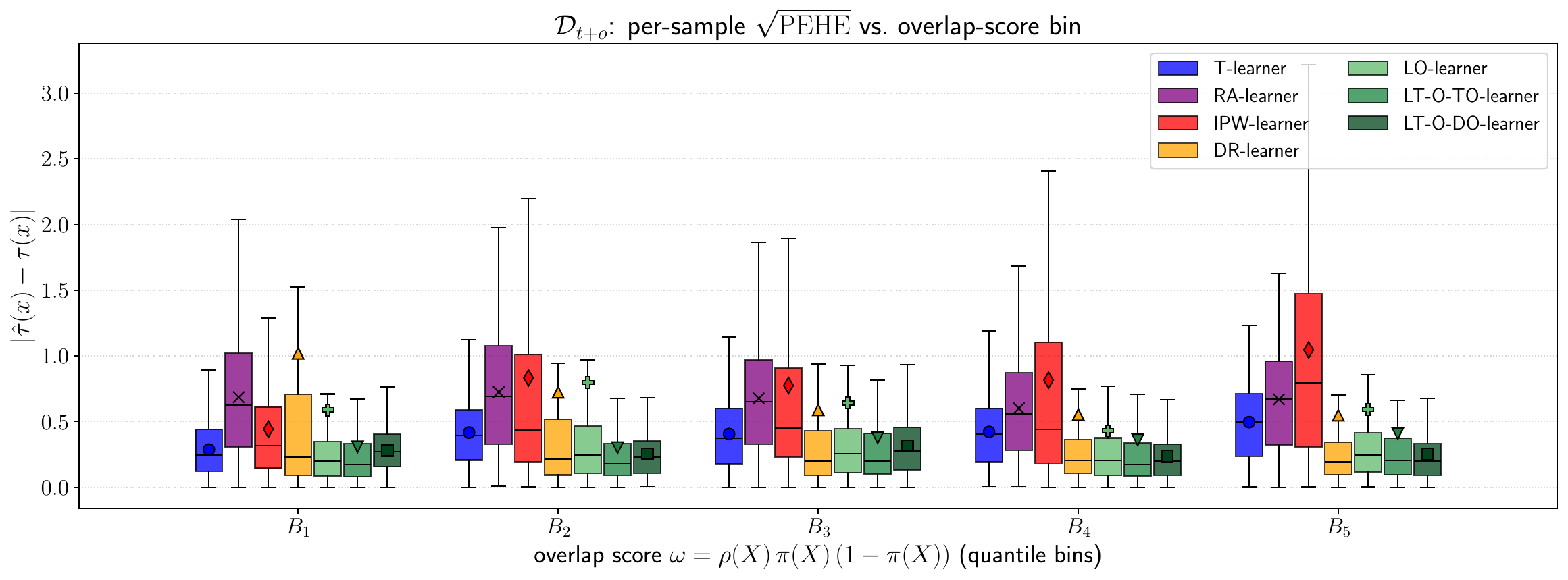}
    \caption{\textbf{PEHE stratified by overlap on $\mathcal{D}_{t+o}$:}  The boxplot shows the 0\%, 25\%, 50\%, 75\%, and 100\% quantiles, as well as the mean (marker) of the square root pseudo-PEHE within each overlap stratum. The plots show that our \methods (\DO, \TO, \LO) remain accurate even in low-overlap strata, despite the sample retargeting.}
    \label{fig:syn-overlap-strat-dual}
    \vspace{-0.2cm}
\end{figure}

\begin{figure}[h!]
    \centering
    \includegraphics[width=1.0\linewidth]{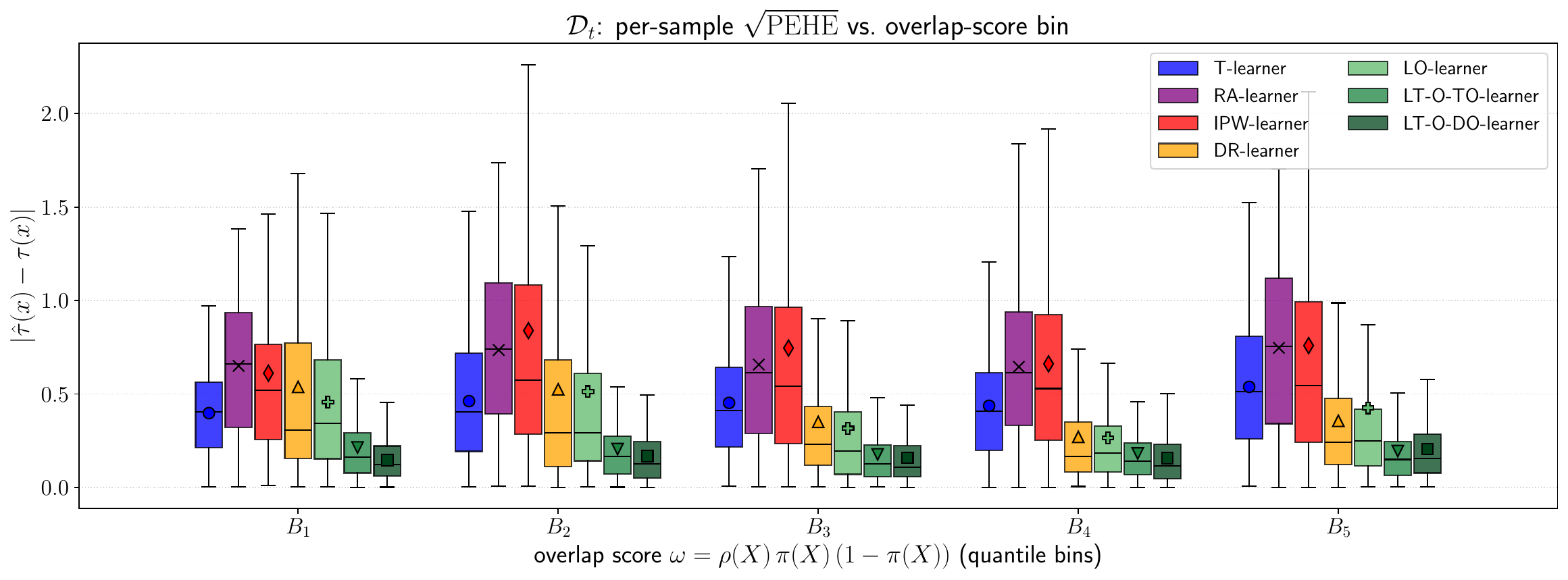}
    \caption{\textbf{PEHE stratified by overlap on $\mathcal{D}_{t}$:}  The boxplot shows the 0\%, 25\%, 50\%, 75\%, and 100\% quantiles, as well as the mean (marker) of the square root pseudo-PEHE within each overlap stratum. The plots show that our \methods (\DO, \TO) remain accurate even in low-overlap strata, despite the sample retargeting.}
    \label{fig:syn-overlap-strat-low-treatment}
    \vspace{-0.5cm}
\end{figure}

\begin{figure}[h!]
    \centering
    \includegraphics[width=1.0\linewidth]{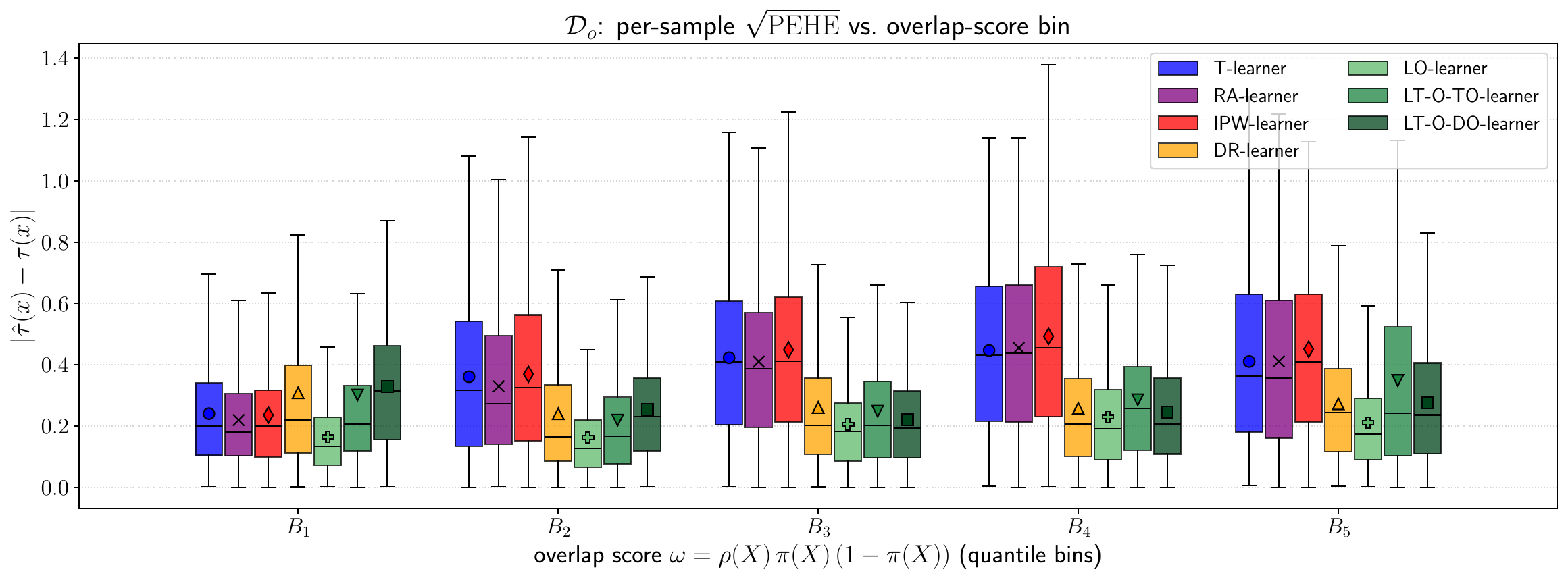}
    \caption{\textbf{PEHE stratified by overlap on $\mathcal{D}_{o}$:}  The boxplot shows the 0\%, 25\%, 50\%, 75\%, and 100\% quantiles, as well as the mean (marker) of the square root pseudo-PEHE within each overlap stratum. The plots show that our \LO remains accurate even in low-overlap strata, despite the sample retargeting.}
    \label{fig:syn-overlap-strat-low-outcome}
    \vspace{-0.5cm}
\end{figure}

\begin{figure}[h!]
    \centering
    \includegraphics[width=1.0\linewidth]{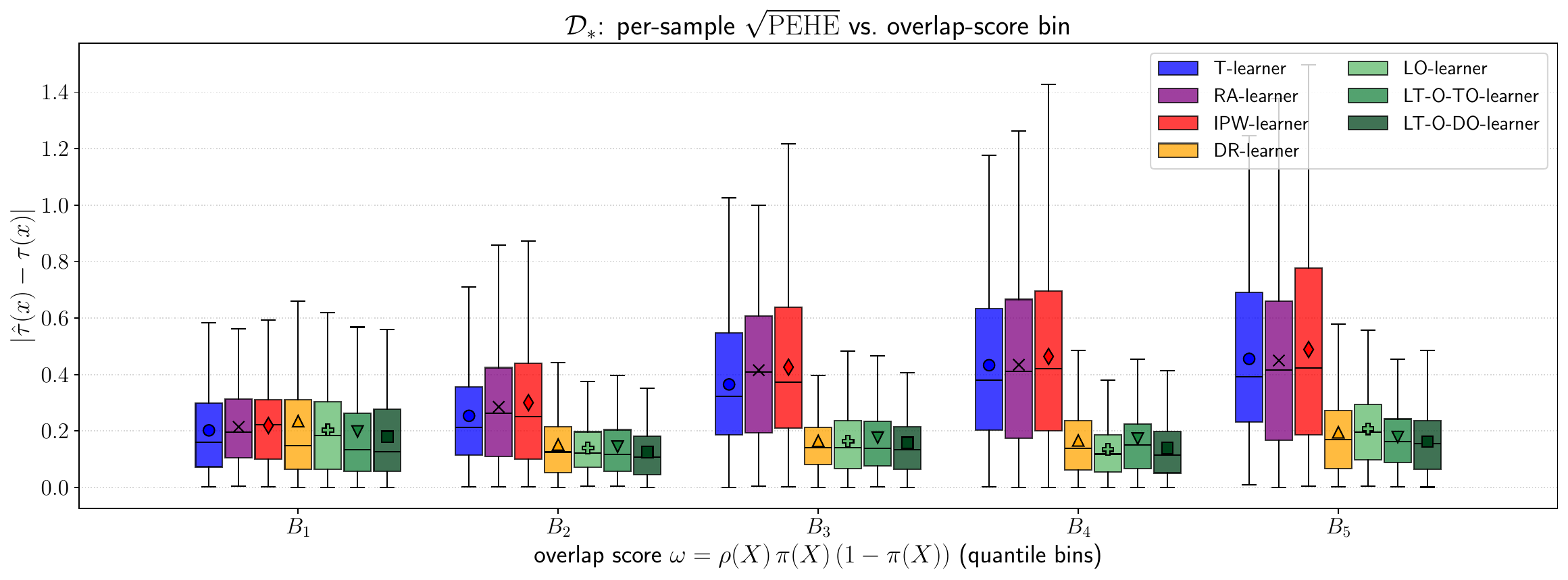}
    \caption{\textbf{PEHE stratified by overlap on $\mathcal{D}_{*}$:}  The boxplot shows the 0\%, 25\%, 50\%, 75\%, and 100\% quantiles, as well as the mean (marker) of the square root pseudo-PEHE within each overlap stratum. The plots show that our \methods (\DO, \TO, \LO) remain accurate even in low-overlap strata, despite the sample retargeting.}
    \label{fig:syn-overlap-strat-healthy}
    \vspace{-0.5cm}
\end{figure}

\FloatBarrier

\subsection{Surrogacy violation sensitivity analysis}
\label{app:bias-surrogacy-violation-exp}
In this experiment, we analyze the bias of our \methods under surrogacy violation on the real-world GAIN dataset. The surrogate variable $S$ in GAIN is a multi-dimensional variable $(S_1, \ldots, S_T)$, where $S_t$ is the employment status at quarter $t$ after the job training program. In principle, the higher the $T$, the more likely that the surrogacy assumption is satisfied, as more post-treatment information is captured by the surrogate. We thus vary $T$ from 2 to 12 and report the pseudo-PEHE of \LTT, \DR, and \DO. 

\begin{figure}[h!]
    \centering
    \includegraphics[width=0.8\linewidth]{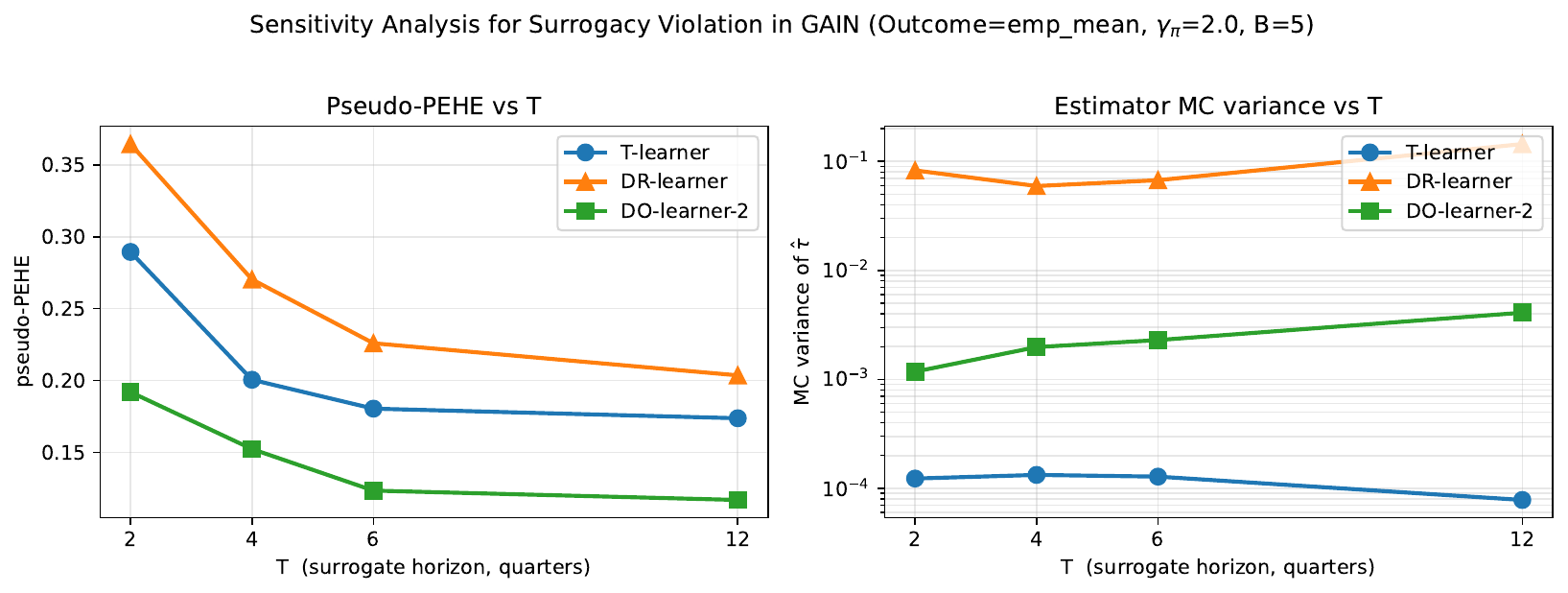}
    \caption{\textbf{Surrogacy violation sensitivity analysis on GAIN:} The left (right) figure shows the mean (variance) of pseudo-PEHE of the \LTT, \DR, and \DO across different surrogate dimensions $T$. The surrogacy violation is more severe at lower $T$, which leads to higher pseudo-PEHE for all methods. However, our \DO remains more robust than the baselines across all $T$.}
    \label{fig:gain-surrogacy-sensitivity}
\end{figure}

\underline{\textit{Results:}} As shown in Figure~\ref{fig:gain-surrogacy-sensitivity}, the pseudo-PEHE of all methods decreases as $T$ increases, which is consistent with the intuition that surrogacy violation is more severe at lower $T$. However, our \DO remains more robust than the baselines across all $T$. In GAIN, the surrogacy test of \citet{Athey.2025} suggests that surrogacy is not strongly rejected for $T \geq 6$, while $T=2$ violates it. This is consistent with our finding that the pseudo-PEHE is high for $T\leq 4$ and stabilizes for $T \geq 6$.

\newpage
\section{Theory}
\label{app:theory}

\subsection{Identification}
\label{app:identification}
In this section, we briefly demonstrate the identifiability of the heterogeneous long-term treatment effects under the surrogacy model. The identification of average long-term treatment effects is first shown by~\citet{Athey.2025}. We make a trivial extension to the heterogeneous setting. 
\begin{lemma}[Identification of HLTE]
    Under Assumptions~\ref{assumption-consistency} (Consistency), \ref{assumption-positivity} (Positivity), \ref{assumption-unconf} (Experimental Unconfoundedness), \ref{assumption-surrogacy} (Surrogacy), and \ref{assumption-compara} (Comparability), the heterogeneous long-term treatment effect $\tau^0(x)$ is identified as:
    \begin{align}
        \tau^0(x) = \mu(1,x) - \mu(0,x),
    \end{align}
    where $\mu(a,x) = \mathbb{E}[h(S,X) \mid A=a, X=x, R=0]$ and $h(s,x) = \mathbb{E}[Y \mid S=s, X=x, R=1]$.
\end{lemma}
\begin{proof}
    We begin with the definition of the causal estimand of interest:
    \begin{align}
        \tau^0(x) = \mathbb{E}[Y(1) - Y(0) \mid X=x, R=0].
    \end{align}
    To estimate $\tau^0$, it suffices to identify the conditional expectation of the potential outcome $\mathbb{E}[Y(a) \mid X=x, R=0]$ for any $a \in \{0,1\}$.
    
    First, using Assumption~\ref{assumption-consistency} (Consistency) and Assumption~\ref{assumption-unconf} (Unconfoundedness), we can relate the potential outcome to the conditional expectation given the treatment assignment $A=a$:
    \begin{align}
        \mathbb{E}[Y(a) \mid X=x, R=0] &= \mathbb{E}[Y(a) \mid A=a, X=x, R=0] \nonumber \\
        &= \mathbb{E}[Y \mid A=a, X=x, R=0].
    \end{align}
    Then, we apply the tower law for expectation by conditioning on the surrogate variable $S$:
    \begin{align}
    \label{eq:proof-ident}
        \mathbb{E}[Y \mid A=a, X=x, R=0] &= \mathbb{E}_{S} \Big[ \mathbb{E}[Y \mid S, A=a, X=x, R=0] \Big| A=a, X=x, R=0 \Big].
    \end{align}
    We invoke Assumption~\ref{assumption-surrogacy} (Surrogacy), which states that $A \indep Y \mid S, X, R=0$. This allows us to remove the conditioning on $A$ in the inner expectation:
    \begin{align}
        \mathbb{E}[Y \mid S, A=a, X=x, R=0] = \mathbb{E}[Y \mid S, X, R=0].
    \end{align}
    Since the long-term outcome $Y$ is not observed in the experimental dataset $\mathcal{D}_1$ ($R=0$), we require Assumption~\ref{assumption-compara} (Long-term outcome comparability), specifically $Y \indep R \mid S, X$, to transport the outcome model from the observational dataset $\mathcal{D}_2$:
    \begin{align}
        \mathbb{E}[Y \mid S, X, R=0] &= \mathbb{E}[Y \mid S, X, R=1] \nonumber \\
        &= h(S,X).
    \end{align}
    Substituting this back into Eq.~(\ref{eq:proof-ident}), we get:
    \begin{align}
        \mathbb{E}[Y(a) \mid X=x, R=0] &= \mathbb{E}_{S} \Big[ h(S,X) \Big| A=a, X=x, R=0 \Big] \nonumber \\
        &= \mu(a,x).
    \end{align}
   Hence, we yield
    \begin{align}
        \tau^0(x) = \mu(1,x) - \mu(0,x).
    \end{align}
\end{proof}

\subsection{Efficient influence function}
\label{app:eif}

\paragraph{What is the efficient influence function (EIF)? } The EIF is a fundamental concept in semiparametric statistics that describes the sensitivity of an estimator to the underlying data distribution. Formally, consider a functional $\psi(\mathbb{P})$ of the data distribution $\mathbb{P}$. In our case,
\begin{align}
    \psi(\mathbb{P})=\mathbb{E}_{S}[\mathbb{E}[Y \mid S, X, R=1]| A=1, X=x, R=0]- \mathbb{E}_{S}[\mathbb{E}[Y \mid S, X, R=1]| A=0, X=x, R=0].
\end{align}
The influence function $\phi(Z)$ of an estimator $\hat{\psi}$  is defined as the Gateaux derivative of the functional at $\mathbb{P}$ in the direction of a point mass $\delta_Z$~\citep{Hampel.1974}:
\begin{align}
    \lim_{\epsilon \to 0} \frac{\psi((1-\epsilon)\mathbb{P} + \epsilon \delta_Z) - \psi(\mathbb{P})}{\epsilon} = \phi(Z).
\end{align}

To illustrate this in our context, consider the average long-term treatment effect (ALTE), $\tau = \mathbb{E}[\tau^0(X)]$. The EIF for the ALTE in the standard surrogacy model is given by the centered version of the doubly-robust pseudo-outcome derived in Eq.~(\ref{DR-pseudo-outcome}):
\begin{align}
    \phi_{ALTE}(Z) = \mathcal{T}_\mathrm{DR}(Z;\eta) - \tau.
\end{align}
First derived in \citet{Chen.2023, Athey.2025}, the corresponding EIF defines the \textbf{efficiency bound} of the ALTE estimation. An estimator is asymptotically efficient (i.e., it achieves the lowest possible asymptotic variance, known as the Cram\'er-Rao lower bound) if and only if it is asymptotically linear with influence function equal to the EIF~\citep{Bickel.1993, Tsiatis.2006}.

\paragraph{Properties of EIF:}
The EIF has two properties that are critical for constructing robust two-stage learners:
\begin{enumerate}
    \item \textbf{Mean zero:} The expectation of the influence function is zero, i.e., $\mathbb{E}[\phi(Z)] = 0$. This implies that the uncentered EIF (or pseudo-outcome), i.e., $\mathcal{T}_\mathrm{DR}(Z)$ is an unbiased estimator of the target parameter: $\mathbb{E}[\mathcal{T}_\mathrm{DR}(Z)] = \tau$.
    
    \item \textbf{Neyman-orthogonality:} The most important property is that the EIF is insensitive to perturbations in the nuisance parameters $\eta$. Mathematically, the G\^ateaux derivative of the expected EIF with respect to the nuisance functions is zero:
    \begin{align}
        \frac{\partial}{\partial r} \mathbb{E}\left[ \phi(Z; \eta + r(\hat{\eta} - \eta)) \right] \bigg|_{r=0} = 0.
    \end{align}
    This \emph{orthogonality} condition implies that the estimation error of the nuisance functions $\hat{\eta} - \eta$ affects the final target estimate $\tau$ only at a second-order rate (i.e., $\|\hat{\eta} - \eta\|^2$). This allows the final estimator to be still $\sqrt{n}$-consistent even if the nuisance components are estimated at slower rates (e.g., $n^{-1/4}$), which is common when using machine learning methods~\citep{Chernozhukov.2018}.
\end{enumerate}

\paragraph{Connection between EIF and DR-pseudo-outcome:}
In our setting, we estimate a \textbf{conditional parameter} $\tau^0(x)$ instead of a scalar parameter $\tau^0$. The connection between the EIF and the \emph{EIF-based pseudo-outcome} (or, equivalently, DR-pseudo-outcome) used in DR-learner~\citep{Kennedy.2023} is that the EIF-based pseudo-outcome is constructed to be the \textit{uncentered} EIF of the corresponding average effect, conditioned on covariates $X$.

Specifically, let $\mathcal{T}_\mathrm{DR}(Z)$ be the uncentered EIF for the average long-term effect. We use it to approximate the unobserved individual treatment effect $Y(1)-Y(0)$ because
\begin{align}
    &\mathbb{E}[\mathcal{T}_\mathrm{DR}(Z) \mid X=x, R=0] = \tau^0(x) \\
    &\mathbb{E}[\mathcal{T}_\mathrm{DR}(Z) \mid X=x, R=1] = 0.    
\end{align}
By defining a loss function $\mathcal{L}(g) = \mathbb{E}[\1(R=0)g(X)^2 - 2\mathcal{T}_\mathrm{DR}g(X)]$, the minimizer is the conditional expectation $\mathbb{E}[\mathcal{T}_\mathrm{DR}|X, R=0]$, which recovers $\tau^0(x)$. Crucially, because $\mathcal{T}_\mathrm{DR}(Z)$ inherits the orthogonality property of the EIF, the second stage regression retains robustness against nuisance estimation errors. 

\paragraph{Semantics of Neyman-orthogonality: functional vs. loss.} In this paper, we frequently use the concept of Neyman-orthogonality. We point out that actually orthogonality may refer to different target objects with different semantic implications:

\begin{itemize}
    \item \textbf{Functional orthogonality:} 
    Primarily used for scalar parameters $\theta$ (e.g., ATE) defined by a moment condition $\mathbb{E}[\psi(Z; \theta, \eta)] = 0$. Orthogonality requires $\partial_\eta \mathbb{E}[\psi(Z; \theta_0, \eta_0)] = 0$~\citep{Chernozhukov.2018}. It ensures that nuisance estimation error $\hat{\eta} - \eta$ only has second-order impact on the target parameter $\hat\theta$.

    \item \textbf{Loss orthogonality:} 
    Used for learning target functions $g(\cdot)$ (e.g., CATE/HLTE) via risk minimization $\min_g \mathcal{L}(g, \eta)$. Orthogonality requires the \textbf{cross-derivative} to vanish: $D_\eta D_g \mathcal{L}(g_0, \eta_0) = 0$~\citep{Foster.2023}. It thus ensures that distortions in $\hat{\eta}$ may change the shape of the loss landscape but do not shift the \emph{location} of the minimizer $g^*$ to the first order.

\end{itemize}

\subsection{EIF of the \texorpdfstring{$\omega$}{omega}-weighted average long-term effects}
\label{app:EIF-weighted-ATE}
The derivation generalizes the EIF for the unweighted ALTE established by \citet{Chen.2023} and \citet{Athey.2025} to the $\omega$-weighted case. The efficient influence function for the parameter $\tau^0_\omega$ in Eq.~(\ref{eq:walte-def}) is given by:
\begin{align}
\label{eq:eif-explicit}
\phi(Z; \eta) = \frac{1}{p_0 D_\omega} \Bigg[ &\1(R=0)\, \omega(X) \bigl( \hat{\tau}_{\mathrm{AIPW}}(Z;\eta) - \tau^0_\omega \bigr) \notag\\
&+ \1(R=1)\, \omega(X)\, \psi_{\mathrm{obs}}(Z;\eta) + \bigl( \tau^0(X) - \tau^0_\omega \bigr) \Omega(Z;\eta) \Bigg],
\end{align}
where $p_0 = P(R=0)$ and $D_\omega = \mathbb{E}[\omega(X)|R=0]$. For simplicity, we use the shorthand $\omega(X) \equiv \omega(\pi(X), \rho(X))$. The components are defined as:
\begin{align}
    \hat{\tau}_{\mathrm{AIPW}}(Z;\eta) &= \mu(1,X) - \mu(0,X) \notag\\
    &\quad + \frac{A}{\pi(X)}\bigl(h(S,X) - \mu(1,X)\bigr) - \frac{1-A}{1-\pi(X)}\bigl(h(S,X) - \mu(0,X)\bigr) \\
    \psi_{\mathrm{obs}}(Z;\eta) &= \frac{1-\rho_s(S,X)}{\rho_s(S,X)} \left( \frac{\pi_s(S,X) - \pi(X)}{\pi(X)(1-\pi(X))} \right) \big( Y - h(S,X) \big) \\
    \Omega(Z;\eta) &= \1(R=0) \frac{\partial \omega}{\partial \pi} (A - \pi(X)) + (1-\rho(X)) \frac{\partial \omega}{\partial \rho} (R - \rho(X)).
\end{align}

\begin{lemma}[Neyman-orthogonality]
\label{lemma:orthogonality}
    $\phi(Z; \eta)$ is the efficient influence function, which satisfies the Neyman-orthogonality condition with respect to the complete set of nuisance functions $\eta = (\pi(x), \pi_s(S,X), \rho(x), \rho_s(S,X), h(s,x), \mu(a,x))$. That is, the G\^ateaux derivative of the expected EIF vanishes at the true nuisance parameters:
    \begin{align}
        \partial_r \mathbb{E}\left[ \phi(Z; \eta + r\delta) \right] \big|_{r=0} = 0,
    \end{align}
    for any valid direction $\delta = \hat{\eta} - \eta$.
\end{lemma}
The proof is essentially the same as in Appendix~\ref{app:proof-neyman}. Hence, we omit the proof.

\newpage

\section{Bias analysis under assumption violation}
\label{app:bias-surrogacy-violation}
In this section, we rigorously analyze the pointwise bias of the HLTE identification formulation when the surrogacy assumption (Assumption~\ref{assumption-surrogacy}) is violated. Note that \citet{Athey.2025} analyzed the bias of the ALTE under surrogacy violations, while we adapt their method to extend to the pointwise HLTE setting.

\begin{proposition}[Pointwise bias of the T-learner under surrogacy and comparability violations]
\label{prop:hlte-bias}
Let Assumptions~\ref{assumption-consistency} (Consistency),~\ref{assumption-positivity} (Positivity), and~\ref{assumption-unconf} ($\mathcal{D}_1$ Unconfoundedness) hold, but allow Assumption~\ref{assumption-surrogacy} (Surrogacy) and Assumption~\ref{assumption-compara} (Comparability) to be violated. Recall
\begin{align}
\pi(s,x) \;:=\; P(A=1\mid S=s, X=x, R=0), \qquad
\pi(x) \;:=\; P(A=1\mid X=x, R=0),
\end{align}
and the conditional outcome regressions, assumed to be well-defined on the relevant conditional supports, i.e.,
\begin{align*}
\mu(s,a,x,R=0) &:= E[Y \mid S=s, A=a, X=x, R=0],\\
\mu(s,x,R=0)   &:= E[Y \mid S=s, X=x, R=0],\\
\mu(s,x,R=1)   &:= E[Y \mid S=s, X=x, R=1].
\end{align*}
Then, we have that
\begin{align}
\mu(s,x,R=0) \;=\; \pi(s,x)\,\mu(s,1,x,R=0) + (1-\pi(s,x))\,\mu(s,0,x,R=0).
\end{align}
Let
\begin{align}
\delta(s,x) \;:=\; \mu(s,1,x,R=0) - \mu(s,0,x,R=0), \qquad
\eta(s,x)   \;:=\; \mu(s,x,R=1) - \mu(s,x,R=0)
\end{align}
denote the \emph{direct effect} of $A$ on $Y$ given $(S,X)$ in the experimental sample and the \emph{comparability gap} between samples respectively. Assumption~\ref{assumption-surrogacy} and Assumption~\ref{assumption-compara} imply the corresponding conditional mean restrictions
\begin{align}
\delta(s,x) \equiv 0, \qquad
\eta(s,x) \equiv 0.
\end{align}
Define the T-learner population estimand
\begin{align}
\bar{\tau}(x) \;:=\; \mu^{\mathrm{TL}}(1,x) - \mu^{\mathrm{TL}}(0,x), \qquad
\mu^{\mathrm{TL}}(a,x) \;:=\; E\bigl[\mu(S,X,R=1) \,\big|\, A=a, X=x, R=0\bigr].
\end{align}
Then, the bias of $\bar{\tau}$ with respect to $\tau^0(x) = E[Y(1) - Y(0) \mid X=x, R=0]$ is
\begin{equation}
\label{eq:bias-decomp}
\tau^0(x) - \bar{\tau}(x) \;=\; \frac{1}{\pi(x)\,(1-\pi(x))}\,\Bigl\{\mathcal{B}_{\mathrm{surr}}(x) \;-\; \mathcal{B}_{\mathrm{comp}}(x)\Bigr\},
\end{equation}
where
\begin{align*}
\mathcal{B}_{\mathrm{surr}}(x) &:= E\bigl[\,\pi(S,X)\,(1-\pi(S,X))\,\delta(S,X)\,\big|\, X=x, R=0\bigr],\\
\mathcal{B}_{\mathrm{comp}}(x) &:= E\bigl[\,(\pi(S,X) - \pi(x))\,\eta(S,X)\,\big|\, X=x, R=0\bigr].
\end{align*}
Marginalizing Eq.~(\ref{eq:bias-decomp}) over $X\mid R=0$ gives the corresponding ALTE bias decomposition, in analogy with Theorem~4(iv) of \citet{Athey.2025}.
\end{proposition}

\begin{proof}
We establish Eq.~(\ref{eq:bias-decomp}) in five steps.

\paragraph{Identification of $\tau^0(x)$ under unconfoundedness.}
Assumption~\ref{assumption-unconf} gives $A \perp\!\!\!\perp (Y(0),Y(1)) \mid X, R=0$. For $a\in\{0,1\}$,
\begin{align*}
E[Y(a) \mid X=x, R=0]
&= E[Y(a) \mid A=a, X=x, R=0] && \text{(Assumption~\ref{assumption-unconf})}\\
&= E[Y \mid A=a, X=x, R=0] && \text{(Assumption~\ref{assumption-consistency})}\\
&= E\bigl[\mu(S,a,X,R=0) \,\big|\, A=a, X=x, R=0\bigr] && \text{(tower law).}
\end{align*}
Subtracting the $a=0$ case from the $a=1$ case,
\begin{equation}
\label{eq:tau-true}
\tau^0(x) \;=\; E[\mu(S,1,X,R=0)\mid A=1, X=x, R=0]
        \;-\; E[\mu(S,0,X,R=0)\mid A=0, X=x, R=0].
\end{equation}

\paragraph{Pointwise decomposition.}
By the definition of $\pi(s,x)$,
\begin{align}
\mu(s,x,R=0) \;=\; \pi(s,x)\,\mu(s,1,x,R=0) + (1-\pi(s,x))\,\mu(s,0,x,R=0).
\end{align}
Inserting $\mu(s,x,R=0)$ as a pivot, for any $s,x$:
\begin{align}
\mu(s,x,R=1) - \mu(s,1,x,R=0)
&= \bigl[\mu(s,x,R=1) - \mu(s,x,R=0)\bigr] + \bigl[\mu(s,x,R=0) - \mu(s,1,x,R=0)\bigr] \nonumber\\
&= \eta(s,x) - (1-\pi(s,x))\,\delta(s,x), \label{eq:diff-1}\\[2pt]
\mu(s,x,R=1) - \mu(s,0,x,R=0)
&= \eta(s,x) + \pi(s,x)\,\delta(s,x). \label{eq:diff-0}
\end{align}

\paragraph{Computing $\bar{\tau}(x)-\tau^0(x)$.}
Combining the definition of $\bar{\tau}$ with Eq.~(\ref{eq:tau-true}), Eq.~(\ref{eq:diff-1}), and Eq.~(\ref{eq:diff-0}),
\begin{align}
\bar{\tau}(x) - \tau^0(x)
&= E[\mu(S,X,R=1) - \mu(S,1,X,R=0) \mid A=1, X=x, R=0] \nonumber\\
&\quad - E[\mu(S,X,R=1) - \mu(S,0,X,R=0) \mid A=0, X=x, R=0] \nonumber\\
&= \underbrace{\bigl\{E[\eta(S,X)\mid A=1,x,R=0] - E[\eta(S,X)\mid A=0,x,R=0]\bigr\}}_{=:\,(\mathrm{I})} \nonumber\\
&\quad - \underbrace{\bigl\{E[(1-\pi(S,X))\delta(S,X)\mid A=1,x,R=0] + E[\pi(S,X)\delta(S,X)\mid A=0,x,R=0]\bigr\}}_{=:\,(\mathrm{II})}. \label{eq:two-terms}
\end{align}

\paragraph{Bayes' rule.}
For any integrable $g$, since $f(s\mid X=x, A=a, R=0) = f(s\mid X=x, R=0)\,P(A=a\mid s,x,R=0)/P(A=a\mid x,R=0)$,
\begin{align}
\label{eq:bayes}
E[g(S,X)\mid A=1, X=x, R=0] &= \frac{E[\pi(S,X)\,g(S,X)\mid X=x, R=0]}{\pi(x)}, \\
E[g(S,X)\mid A=0, X=x, R=0] &= \frac{E[(1-\pi(S,X))\,g(S,X)\mid X=x, R=0]}{1-\pi(x)}.
\end{align}

\paragraph{Deriving $(\mathrm{I})$ and $(\mathrm{II})$.}
Apply Eq.~(\ref{eq:bayes}) to the surrogacy term:
\begin{align*}
(\mathrm{II})
&= \frac{E[\pi(S,X)(1-\pi(S,X))\,\delta(S,X)\mid x,R=0]}{\pi(x)}
 + \frac{E[\pi(S,X)(1-\pi(S,X))\,\delta(S,X)\mid x,R=0]}{1-\pi(x)}\\
&= E[\pi(S,X)(1-\pi(S,X))\,\delta(S,X)\mid X=x, R=0]\cdot\frac{(1-\pi(x))+\pi(x)}{\pi(x)(1-\pi(x))}
\;=\; \frac{\mathcal{B}_{\mathrm{surr}}(x)}{\pi(x)(1-\pi(x))}.
\end{align*}
For the comparability term, again by Eq.~(\ref{eq:bayes}) and the identity
$(1-\pi(x))\,\pi(s,x) - \pi(x)\,(1-\pi(s,x)) = \pi(s,x) - \pi(x)$,
\begin{align*}
(\mathrm{I})
&= \frac{E[\pi(S,X)\,\eta(S,X)\mid x,R=0]}{\pi(x)} - \frac{E[(1-\pi(S,X))\,\eta(S,X)\mid x,R=0]}{1-\pi(x)}\\
&= \frac{E\bigl[\,\{(1-\pi(x))\,\pi(S,X) - \pi(x)\,(1-\pi(S,X))\}\,\eta(S,X)\,\big|\,x,R=0\bigr]}{\pi(x)(1-\pi(x))}
\;=\; \frac{\mathcal{B}_{\mathrm{comp}}(x)}{\pi(x)(1-\pi(x))}.
\end{align*}
Substituting both into Eq.~(\ref{eq:two-terms}) gives
$\bar{\tau}(x) - \tau^0(x) = \bigl(\mathcal{B}_{\mathrm{comp}}(x) - \mathcal{B}_{\mathrm{surr}}(x)\bigr)/\bigl(\pi(x)(1-\pi(x))\bigr)$,
which is Eq.~(\ref{eq:bias-decomp}).
\end{proof}

\begin{corollary}[Pointwise sensitivity bound under bounded direct effect]
\label{cor:sensitivity}
Suppose Assumption~\ref{assumption-compara} (Comparability) holds, and there exists $c:\mathcal{X}\to\mathbb{R}_{+}$ with $|\delta(s,x)|\le c(x)$ for almost every $s$ in the support of $S\mid X=x, R=0$. Then
\begin{align}
\bigl|\,\tau^0(x) - \bar{\tau}(x)\,\bigr| \;\le\; c(x)\cdot \frac{E[\pi(S,X)(1-\pi(S,X))\mid X=x, R=0]}{\pi(x)(1-\pi(x))} \;\le\; c(x).
\end{align}
The second inequality follows from concavity of $u\mapsto u(1-u)$ and Jensen's inequality applied to $\pi(x) = E[\pi(S,X)\mid X=x, R=0]$. In particular, the right-hand side is computable from $R=0$ data alone (it does not require $Y$), making it directly usable as a sensitivity-analysis input.
\end{corollary}

\newpage

\section{Instantiations}
\label{app:instantiations}

\subsection{Derivation of the \DR (special case \texorpdfstring{$\omega\equiv 1$}{omega = 1})}
\label{app:DR-LT-learner}
When we trivially set $\omega\equiv 1$, i.e.\ no overlap weighting in the loss, the LT pseudo-outcome reduces to the EIF of the ALTE from~\citet{Chen.2023, Athey.2025}, and we recover the \DR under the standard surrogacy model. Substituting $\omega\equiv 1$ into Eq.~(\ref{eq:weighted-ortho-loss-tau0}) and Eq.~(\ref{Tau-LT-def}) yields
\begin{align}
    \mathcal{L}_1(g, \eta)
    &= \mathbb{E} \left[ \1(R=0)g(X)^2
    -2\mathcal{T}_\mathrm{DR}\cdot g(X)\right], \\
    \mathcal{T}_\mathrm{DR}(Z;\eta)
    &= \1(R=0) \hat{\tau}_\mathrm{AIPW}(Z;\eta)
    + \1(R=1)\psi_\mathrm{obs}(Z;\eta). \label{DR-pseudo-outcome}
\end{align}
Here, $\hat{\tau}_\mathrm{AIPW}(Z;\eta)$ is the AIPW pseudo-outcome with $Y$ replaced by the surrogate index $h(S,X)$, and $\psi_\mathrm{obs}$ corrects the first-order error of $h$ using the observed long-term outcome $Y$ in $\mathcal{D}_2$.

\subsection{Derivation of the \TO}
\label{app:TO-LT-learner}
To derive the \TO, we set the weighting function to $\omega(X) = \pi(X)^2(1-\pi(X))^2$ to specifically target treatment overlap. The squared form is motivated by the variance lower bound in Appendix~\ref{sec:optimality}, whose binding term scales with $1/[\rho\,\pi^2(1-\pi)^2]$. We substitute $\omega$ into Eq.~(\ref{omega-star-def}) and Eq.~(\ref{Tau-LT-def}).

First, the residual term $\Omega(Z;\eta)$ becomes:
\begin{align}
    \Omega(Z;\eta)
    &= \1(R=0)\frac{\partial \omega}{\partial \pi}(A-\pi) \notag \\
    &= \1(R=0)\,2\pi(1-\pi)(1-2\pi)(A-\pi).
\end{align}
Substituting back into $\omega^*$:
\begin{align}
\omega^*(Z;\eta)
&= \1(R=0)\bigl[\pi^2(1-\pi)^2
  + 2\pi(1-\pi)(1-2\pi)(A-\pi)\bigr].
\end{align}
Although $\omega^*$ might be negative, its conditional expectation $\E[\omega^*\mid X]=(1-\rho)\,\pi^2(1-\pi)^2$ is positive.
The LT-pseudo-outcome $\mathcal{T}_\mathrm{LT}$ becomes:
\begin{align}
\mathcal{T}_\mathrm{LT}(Z;\eta)
&= \1(R=0)\Big[
    \pi^2(1-\pi)^2\,\hat{\tau}_\mathrm{AIPW}
    + (\mu_1-\mu_0)\,2\pi(1-\pi)(1-2\pi)(A-\pi)
    \Big] \notag \\
&\quad + \1(R=1)\,\pi^2(1-\pi)^2\,\psi_\mathrm{obs} \notag \\
&= (\mu_1-\mu_0)\,\omega^*(Z;\eta) \notag \\
&\quad + \1(R=0)\,\pi(1-\pi)(A-\pi)
    \bigl(h(S,X)-\mu(A,X)\bigr) \notag \\
&\quad + \1(R=1)\,\pi^2(1-\pi)^2\,\psi_\mathrm{obs}.
\end{align}
Plugging $\omega^*$ and $\mathcal{T}_\mathrm{LT}$ into $\mathcal{L}_\omega = \mathbb{E}_n[\omega^* g^2 - 2\mathcal{T}_\mathrm{LT} g]$ yields the \TO loss
\begin{align}
\mathcal{L}_{\pi^2(1-\pi)^2}(g;\eta)
&= \mathbb{E}\Bigl[
    \omega^*(Z;\eta)\,g(X)^2
    - 2\,\mathcal{T}_\mathrm{LT}(Z;\eta)\,g(X)
    \Bigr].
\end{align}

\textbf{Remark:} The simpler choice $\omega = \pi(1-\pi)$ yields the closed form $\omega^*=\1(R=0)(A-\pi)^2$ and a loss whose $R=0$ component completes the square to $\bigl((h(S,X)-m(X))-(A-\pi)g(X)\bigr)^2$, closely resembling the standard R-learner~\citep{Nie.2021}. The squared variant $\omega=\pi^2(1-\pi)^2$ adopted here departs from this clean form, but it more closely tracks the binding term of the variance lower bound in long-term settings (Appendix~\ref{sec:optimality}) and yields better empirical performance under low overlap (cf.\ Section~\ref{sec:experiments}). The simpler $\omega=\pi(1-\pi)$ variant is provided as an alternative instantiation in Appendix~\ref{app:other-instantiations}.

\subsection{Derivation of the \LO}
\label{app:LO-LT-learner}
The \LO is intended for settings where only long-term outcome overlap (LO) is low. We choose $\omega(x)=\rho(x)$ to downweight samples with a low probability of observing long-term outcomes. The corresponding LT weighting function and LT-pseudo-outcome are
\begin{align*}
    \Omega(Z;\eta)
    &= (1-\rho(X))(R-\rho(X)),\\
    \omega^*(Z;\eta)
    &= (R-\rho(X))^2,\\
    \mathcal{T}_\mathrm{LT}(Z;\eta)
    &= (R-\rho(X))^2(\mu(1,X)-\mu(0,X)) \\
    &\quad + \1(R=0)\,\rho(X)
      \frac{A-\pi(X)}{\pi(X)(1-\pi(X))}
      \bigl(h(S,X)-\mu(A,X)\bigr) \\
    &\quad + \1(R=1)\,\rho(X)\psi_\mathrm{obs}(Z;\eta).
\end{align*}

\subsection{Derivation of the \DO}
\label{app:DO-LT-learner}
The \DO is designed for settings in which both treatment overlap and long-term outcome overlap are limited. We choose $\omega(x)=\pi(x)^2(1-\pi(x))^2\rho(x)$, which is the inverse of the binding term in the variance lower bound of Appendix~\ref{sec:optimality} and jointly downweights samples with weak treatment overlap and low probability of observing long-term outcomes.

Substituting into Eq.~(\ref{omega-star-def}) with $\partial\omega/\partial\pi=2\pi(1-\pi)(1-2\pi)\rho$ and $\partial\omega/\partial\rho=\pi^2(1-\pi)^2$, the residual term is
\begin{align*}
\Omega(Z;\eta)
&= \1(R=0)\,2\pi(1-\pi)(1-2\pi)\rho\,(A-\pi) \\
&\quad + (1-\rho)\,\pi^2(1-\pi)^2\,(R-\rho).
\end{align*}
The resulting LT-weighting function and LT-pseudo-outcome are
\begin{align*}
    \omega^*(Z;\eta) &=
      \1(R=0)\,\rho\bigl[
        \rho\,\pi^2(1-\pi)^2
        + 2\pi(1-\pi)(1-2\pi)(A-\pi)
      \bigr] \\
    &\quad + \1(R=1)\,\pi^2(1-\pi)^2(1-\rho)^2, \\
    \mathcal{T}_\mathrm{LT}(Z;\eta)
    &= (\mu(1,X)-\mu(0,X))\,\omega^*(Z;\eta) \\
    &\quad + \1(R=0)\,\pi(1-\pi)\,\rho\,(A-\pi)
      \bigl(h(S,X)-\mu(A,X)\bigr) \\
    &\quad + \1(R=1)\,\pi^2(1-\pi)^2\,\rho\,\psi_\mathrm{obs}(Z;\eta).
\end{align*}
As with the \TO, $\omega^*$ is sample-wise sign-indefinite but has non-negative conditional expectation $\E[\omega^*\mid X]=(1-\rho)\,\pi^2(1-\pi)^2\rho \ge 0$.

\subsection{Other possible instantiations}
\label{app:other-instantiations}
The four instantiations above (\DR, \TO, \LO, \DO) cover the most common overlap regimes. We also list a few additional variants that may be useful in practice.

\textbf{LT-O-TO-learner (alternative, $\omega=\pi(1-\pi)$):}
The simpler choice $\omega(X)=\pi(X)(1-\pi(X))$ --- the canonical \citet{Crump.2009} treatment-overlap weight --- yields a particularly clean derivation. With $\partial\omega/\partial\pi=1-2\pi$, the residual term is $\Omega(Z;\eta)=\1(R=0)(1-2\pi)(A-\pi)$, and
\begin{align*}
\omega^*(Z;\eta)
&= \1(R=0)\bigl[\pi(1-\pi)+(1-2\pi)(A-\pi)\bigr] \\
&= \1(R=0)(A-\pi)^2,\\
\mathcal{T}_\mathrm{LT}(Z;\eta)
&= \1(R=0)(A-\pi)\bigl(h(S,X)-m(X)\bigr)
  + \1(R=1)\,\psi_t(Z;\eta),
\end{align*}
where $m(X):=\pi\mu_1+(1-\pi)\mu_0=\E[h(S,X)\mid X,R=0]$ and $\psi_t:=\frac{1-\rho_s}{\rho_s}(\pi_s-\pi)(Y-h)$. Completing the square for the $R=0$ component yields, up to an additive constant independent of $g$,
\begin{align*}
\mathcal{L}_{\pi(1-\pi)}(g;\eta)
&\equiv \mathbb{E}\Bigl[
    \1(R=0)\bigl((h(S,X)-m(X))-(A-\pi(X))g(X)\bigr)^2 \\
&\quad - 2\,\1(R=1)\,\psi_t(Z;\eta)\,g(X)
    \Bigr],
\end{align*}
which closely resembles the standard R-learner~\citep{Nie.2021}: the $\1(R=0)$ component merely replaces $Y$ with the surrogate index $h$, and the $\1(R=1)$ correction restores Neyman-orthogonality against the first-order error in $\hat h$. We adopt the squared variant $\omega=\pi^2(1-\pi)^2$ as the main \TO since it more closely tracks the binding term of the variance lower bound and yields better empirical performance under low overlap (cf.\ Section~\ref{sec:experiments}).

\textbf{LT-O-DO-learner (alternative, $\omega=\pi(1-\pi)\rho$):}
Pairing the canonical treatment-overlap weight with $\rho$ also targets joint overlap but with a single $\pi(1-\pi)$ factor rather than the squared one. The resulting LT-weighting function and LT-pseudo-outcome are
\begin{align*}
\omega^*(Z;\eta)
&= \1(R=0)\,\rho\bigl[
    \rho\,\pi(1-\pi)+(1-2\pi)(A-\pi)
    \bigr] \\
&\quad + \1(R=1)\,\pi(1-\pi)(1-\rho)^2,\\
\mathcal{T}_\mathrm{LT}(Z;\eta)
&= (\mu(1,X)-\mu(0,X))\,\omega^*(Z;\eta) \\
&\quad + \1(R=0)\,\rho(A-\pi)
    \bigl(h(S,X)-\mu(A,X)\bigr) \\
&\quad + \1(R=1)\,\rho\,\pi(1-\pi)\,\psi_\mathrm{obs}(Z;\eta).
\end{align*}

\textbf{LT-O-LO$\frac{1}{2}$-learner:} We consider an alternative weighting function $\omega(x)=\sqrt{\rho(x)}$ to provide a smaller penalty for low long-term outcome overlap. Here $\partial\omega/\partial\rho=1/(2\sqrt{\rho})$, so the corresponding LT-weighting function and LT-pseudo-outcome are given by:
\begin{align*}
    \omega^*(Z;\eta) &=
    \begin{cases}
       \frac{\sqrt{\rho(X)}(1+\rho(X))}{2}, & \text{if } R=0, \\
       \frac{(1-\rho(X))^2}{2\sqrt{\rho(X)}}, & \text{if } R=1.
    \end{cases} \\
    \mathcal{T}_\mathrm{LT}(Z;\eta)
    &= (\mu(1,X)-\mu(0,X))\,\omega^*(Z;\eta) \\
    &\quad + \1(R=0)\sqrt{\rho(X)}
      \frac{A-\pi(X)}{\pi(X)(1-\pi(X))}
      \bigl(h(S,X)-\mu(A,X)\bigr) \\
    &\quad + \1(R=1)\sqrt{\rho(X)}\,\psi_\mathrm{obs}(Z;\eta).
\end{align*}
Equivalently, the normalized pseudo-outcome
$\mathcal{T}_\mathrm{LT}(Z;\eta)/\omega^*(Z;\eta)$ is
\begin{align*}
\mu(1,X)-\mu(0,X)
+ \begin{cases}
\frac{2}{1+\rho(X)}
\frac{A-\pi(X)}{\pi(X)(1-\pi(X))}
\bigl(h(S,X)-\mu(A,X)\bigr), & \text{if } R=0, \\
\frac{2\rho(X)}{(1-\rho(X))^2}\psi_\mathrm{obs}(Z;\eta),
& \text{if } R=1.
\end{cases}
\end{align*}

\newpage

\section{Proofs}
\label{app:proofs}
In the following section, we make use of the following properties:

\textbf{(1) $\E\left[\Omega(Z;\eta) \mid X\right]=0$}: 
\begin{align}
\label{eq:Omega-mean-zero-property}
    \E\left[\Omega(Z;\eta) \mid X\right] &= \E\left[\1(R=0) \frac{\partial \omega}{\partial \pi} (A - \pi(X)) + (1-\rho(X)) \frac{\partial \omega}{\partial \rho} (R - \rho(X)) \mid X\right] \notag \\
    &=(1 - \rho(X))\E\left[\frac{\partial \omega}{\partial \pi} (A - \pi(X))\mid X, R = 0\right] \notag \\
    &= 0. \qquad \qquad \qquad \qquad  \qquad \qquad(\text{By definition, } \pi(X)=\E\left[A\mid X,R=0\right])
\end{align}

\textbf{(2) $\E\left[\omega^*(Z;\eta) \mid X\right]=(1-\rho(X))\omega(X)$}: 

Since $\E\left[\Omega(Z;\eta) \mid X\right]=0$ and
\begin{align}
    \omega^*(Z;\eta) &= \1(R=0) \omega(X) + \Omega(Z;\eta),
\end{align}
we have 
\begin{align}
    \E\left[\omega^*(Z;\eta) \mid X\right]= \E\left[\1(R=0) \omega(X)\right] = (1-\rho(X))\omega(X)
\end{align}

\textbf{(3)  $\E\left[\1(R=1)\psi_{\mathrm{obs}}(Z;\eta) \mid X\right] = 0$}:
\begin{align}
\label{eq:psi-mean-zero-property}
    \E\left[\1(R=1)\psi_{\mathrm{obs}}(Z;\eta) \mid X\right] &= \rho(X) \E\left[\frac{1-\rho_s(S,X)}{\rho_s(S,X)} \left( \frac{\pi_s(S,X) - \pi(X)}{\pi(X)(1-\pi(X))} \right) \big( Y - h(S,X) \mid X, R=1\right] \notag \\
    = \rho(X) \E_S&\left[ \frac{1-\rho_s(S,X)}{\rho_s(S,X)} \left( \frac{\pi_s(S,X) - \pi(X)}{\pi(X)(1-\pi(X))} \right)\E_Y\left[ \big( Y - h(S,X) \mid S, X, R=1\right] \mid X, R=1\right] \\
    &=0.
\end{align}

\textbf{(4) $\E\left[\hat{\tau}_{\mathrm{AIPW}}(Z;\eta)\mid X, R=0\right]=\tau^0(X)$}:
\begin{align}
&\;\;\;\;\;\E\!\left[\hat{\tau}_{\mathrm{AIPW}}(Z;\eta)\mid X, R=0\right]\\
&=\E\!\left[
\mu(1,X)-\mu(0,X)
+\frac{A}{\pi(X)}\big(h(S,X)-\mu(1,X)\big)
-\frac{1-A}{1-\pi(X)}\big(h(S,X)-\mu(0,X)\big)
\ \Bigm|\ X,R=0\right] \notag\\
&=\mu(1,X)-\mu(0,X)
+\E\!\left[
\E\!\left[\frac{A}{\pi(X)}\big(h(S,X)-\mu(1,X)\big)\Bigm|X,A,R=0\right]
\Bigm|X,R=0\right] \notag\\
&\quad
-\E\!\left[
\E\!\left[\frac{1-A}{1-\pi(X)}\big(h(S,X)-\mu(0,X)\big)\Bigm|X,A,R=0\right]
\Bigm|X,R=0\right]. \notag
\end{align}
For the first augmentation term, using that $\pi(X)$ and $\mu(1,X)$ are $X$-measurable and that $A\in\{0,1\}$,
\begin{align}
\E\!\left[\frac{A}{\pi(X)}\big(h(S,X)-\mu(1,X)\big)\Bigm|X,A,R=0\right]
&=\frac{A}{\pi(X)}\Big(\E[h(S,X)\mid X,A,R=0]-\mu(1,X)\Big).\label{eq:aipw-term1-cond}
\end{align}
When $A=1$, by definition of $\mu(1,X)$,
\begin{align}
\E[h(S,X)\mid X,A=1,R=0]=\mu(1,X),
\end{align}
so the bracket in Eq.~(\ref{eq:aipw-term1-cond}) is zero; when $A=0$,  $A/\pi(X)=0$. Hence,
\begin{align}
\E\!\left[\frac{A}{\pi(X)}\big(h(S,X)-\mu(1,X)\big)\Bigm|X,R=0\right]=0.
\label{eq:aipw-term1-zero}
\end{align}
And with the same derivation, we have $\E\!\left[\frac{1-A}{1-\pi(X)}\big(h(S,X)-\mu(0,X)\big)\Bigm|X,A,R=0\right]=0$.
Therefore, 
\begin{align}
\label{eq:tau-AIPW-mean-tau}
\E\!\left[\hat{\tau}_{\mathrm{AIPW}}(Z;\eta)\mid X, R=0\right]
=\mu(1,X)-\mu(0,X)=\tau^0(X).
\end{align}

\textbf{(5) $\E\left[\mathcal{T}_\mathrm{LT}(Z;\eta) \mid X\right] = (1-\rho(X))\omega(X) \tau^0(X)$:}

By definition, 
\begin{align}
    \mathcal{T}_\mathrm{LT}(Z;\eta) =  \1(R=0) \omega(X) \hat{\tau}_\mathrm{AIPW}(Z;\eta)
     +  \1(R=1) \cdot\omega(X) \psi_\mathrm{obs}(Z;\eta) + (\mu(1,X)-\mu(0,X)) \Omega(Z;\eta),\notag
\end{align}
Using properties (1), (3) and (4) from above, we have immediately that $\E\left[\mathcal{T}_\mathrm{LT}(Z;\eta) \mid X\right] = (1-\rho(X))\omega(X) \tau^0(X).$

\subsection{Variance analysis}
\label{app:variance}
Here, we formally state the variance of the long-term DR-pseudo-outcome:

\begin{lemma}[Variance of DR-pseudo-outcome]
\label{lemma-variance}
The variance of DR-pseudo-outcome is

\begin{align}
    \var(\mathcal{T}_\mathrm{DR} \mid X) \geq \underbrace{(1-\rho(X)) \cdot \Sigma_{t}(X)}_{\text{treatment overlap}} + \underbrace{\rho(X) \cdot \Sigma_{\text{o}}(X)}_{\text{long-term outcome overlap}}
\end{align}

\vspace{-0.6cm}
\begin{align}
    \Sigma_{t}(X) &= \mathbb{E}\Bigl[ \left(\frac{1}{\pi(X)(1-\pi(X))}\right)^2 \sigma_h(A,X) \,\big|\, X \Bigr]  \\
    \Sigma_{\text{o}}(X) &= \mathbb{E}_{S|X} \left[ \left(\frac{1-\rho_s(S,X)}{\rho_s(S,X)}\left(\frac{\pi_s-\pi}{\pi(1-\pi)}\right)\right)^2 \sigma_y(S,X) \right],
\end{align}
where $\sigma_h(X,A)=\var(h(S,X)|X,A,R=0)$, $\sigma_y(S,X)=\var(Y|S,X,R=1)$.
\normalsize
\end{lemma}

This is similar to the variance analysis from \citet{Fisher.2023}. When the treatment overlap $\pi(x)(1-\pi(x)$ is close to zero, both $\Sigma_t(x)$ and $\Sigma_o(x)$ become very large. When there is limited long-term outcome overlap, namely $\rho(x)$ close to zero (and thus for $\rho_s(S,X)$ as well),  we have that $\rho(x)\Sigma_o(x)$ scales roughly with $\E[\frac{1}{\rho_s(S,X)}|X=x]$. Therefore, $\mathcal{T}_\mathrm{DR}$ has a high variance when at least one of the overlaps is low. As a result, the \DR is \emph{not} robust against low overlap.

\textbf{Proof.}
By the Law of Total Variance, conditioning on the dataset indicator $R$, we have:
\begin{align}
    \var(\mathcal{T}_{\mathrm{DR}} \mid X) &= \mathbb{E}_R[\var(\mathcal{T}_{\mathrm{DR}} \mid X, R) \mid X] + \var_R(\mathbb{E}[\mathcal{T}_{\mathrm{DR}} \mid X, R] \mid X) \\
    &\geq \mathbb{E}_R[\var(\mathcal{T}_{\mathrm{DR}} \mid X, R) \mid X]
\end{align}
where the inequality holds because variance is non-negative. Expanding the expectation over $R \in \{0, 1\}$ with $\rho(X) = \mathbb{P}(R=1|X)$, we obtain:
\begin{equation}
    \var(\mathcal{T}_{\mathrm{DR}} \mid X) \geq (1-\rho(X)) \underbrace{\var(\mathcal{T}_{\mathrm{DR}} \mid X, R=0)}_{\text{(A)}} + \rho(X) \underbrace{\var(\mathcal{T}_{\mathrm{DR}} \mid X, R=1)}_{\text{(B)}} \label{eq:total_var}
\end{equation}

\textbf{Part (A): Treatment overlap instability ($R=0$).}
In the experimental dataset ($R=0$), $\mathcal{T}_{\mathrm{DR}}$ reduces to the AIPW estimator:
\begin{align}
    \hat\tau_{\mathrm{AIPW}}=\frac{A-\pi(X)}{\pi(X)(1-\pi(X))}(h(S,X) - \mu(A,X)) + \mu(1,X)-\mu(0,X):=Z_0 + \mu(1,X)-\mu(0,X).
\end{align}
Hence, we have
\begin{align}
    \var(\mathcal{T}_{\mathrm{DR}} \mid X, R=0)) &= \var(\hat\tau_{\mathrm{AIPW}} \mid X, R=0)) \notag \\
    &= \var\left(\frac{A-\pi(X)}{\pi(X)(1-\pi(X))}(h(S,X) - \mu(A,X))\mid X, R=0\right) \notag \\
    &= \mathbb{E}_{A}[\text{Var}(Z_0 \mid A, X)] + \text{Var}_{A}(\underbrace{\mathbb{E}[Z_0 \mid A, X]}_{=0})
\end{align}
The inner expectation is zero because $\mu(A,X) = \mathbb{E}[h(S,X)|A,X]$. Thus, we are left with the expected variance:
\begin{equation}
    \Sigma_t(X) := \var(\mathcal{T}_{\mathrm{DR}} \mid X, R=0))=\var(Z_0 \mid X) = \mathbb{E}_{A|X} \left[ \left( \frac{A-\pi(X)}{\pi(X)(1-\pi(X))} \right)^2 \var(h(S,X) \mid A, X) \right]
\end{equation}

This term scales with $\frac{1}{\pi(X)^2(1-\pi(X))^2}$, causing high variance under limited treatment overlap.

\textbf{Part (B): Long-term outcome overlap instability ($R=1$).}
In the observational dataset ($R=1$), $\mathcal{T}_{\mathrm{DR}}$ uses the observation proxy $\psi_{\mathrm{obs}}$. Let:
$$
Z_1 = \frac{1-\rho_s(S,X)}{\rho_s(S,X)} W_{\pi}(S,X) (Y - h(S,X)),
$$
where $W_{\pi}(S,X)=\left( \frac{\pi_s(S,X) - \pi(X)}{\pi(X)(1-\pi(X))} \right)$.
Conditioning on the surrogate $S$, we get:
\begin{align}
    \var(Z_1 \mid X, R=1) = \mathbb{E}_{S}[\var(Z_1 \mid S, X, R=1)] + \var_{S}(\underbrace{\mathbb{E}[Z_1 \mid S, X, R=1]}_{=0}).
\end{align}
The inner expectation is zero because $h(S,X) = \mathbb{E}[Y|S,X, R=1]$. Thus:
\begin{align}
    \Sigma_o(X) := \var(Z_1 \mid X) = \mathbb{E}_{S|X} \left[ \left( \frac{1-\rho_s(S,X)}{\rho_s(S,X)} \right)^2 \left(\frac{\pi_s(S,X) - \pi(X)}{\pi(X)(1-\pi(X))} \right)^2 \var(Y \mid S, X) \right]
\end{align}
This term scales with $\frac{1}{\rho_s(S,X)^2}$, causing high variance under limited long-term outcome overlap.

Combining (A) and (B) together yields the lower bound:
\begin{align}
    \var(\mathcal{T}_{\mathrm{DR}} \mid X) \geq (1-\rho(X)) \Sigma_{t}(X) + \rho(X) \Sigma_{o}(X)
\end{align}

\subsection{Neyman-orthogonality}
\label{app:proof-neyman}
Let $\partial_g \mathcal{L}_\omega(g, \eta) = S(g, \eta) \cdot \dot{g}(X) $ be the derivative of the loss. By definition, the loss in Eq.~(\ref{eq:weighted-ortho-loss-tau0}) is orthogonal if we can show Neyman-orthogonality of the following term:
\begin{align}
    \partial_r \mathbb{E}[S(g, \eta + r\delta)] \big|_{r=0} = 0, \quad \forall \delta=\hat\eta-\eta.
\end{align}
The reason is that $\dot{g}(X)$ is irrelevant w.r.t. the nuisance $\eta$. The term can be written as:
\begin{align}
    \E[S(g,\eta)] = \mathbb{E}\Big[ \1(R=0)\omega(X)(g(X) - \hat{\tau}_{\mathrm{AIPW}}) - \1(R=1)\omega(X)\psi_{\mathrm{obs}} - \big( (\mu(1,X)-\mu(0,X)) - g(X) \big) \Omega(Z;\eta)\Big].
\end{align}
where
\begin{align}
    \hat{\tau}_{\mathrm{AIPW}}(Z;\eta) &= \mu(1,X) - \mu(0,X) \notag  + \frac{A}{\pi(X)}(h(S,X) - \mu(1,X)) - \frac{1-A}{1-\pi(X)}(h(S,X) - \mu(0,X)) \\
    \psi_{\mathrm{obs}}(Z;\eta) &= \frac{1-\rho_s(S,X)}{\rho_s(S,X)} \left( \frac{\pi_s(S,X) - \pi(X)}{\pi(X)(1-\pi(X))} \right) \big( Y - h(S,X) \big) \\
    \Omega(Z;\eta) &= \1(R=0) \frac{\partial \omega}{\partial \pi} (A - \pi(X)) + (1-\rho(X)) \frac{\partial \omega}{\partial \rho} (R - \rho(X)).
\end{align}

\textbf{1. Orthogonality w.r.t. outcome nuisances $\mu(a,x)$:}

The nuisance $\mu(a,x)$ appears in $\hat{\tau}_{\mathrm{AIPW}}$ and $(\mu(1,X)-\mu(0,X)) \Omega(Z;\eta)$. Let $\mu_r = \mu + r\delta_\mu$.
\begin{align}
    \partial_r \mathbb{E}[S]\Big|_{r=0} &= -\mathbb{E}\left[ \1(R=0)\omega(X) \partial_r \hat{\tau}_{\mathrm{AIPW}} + (\delta_\mu(1) - \delta_\mu(0))\Omega(Z;\eta)\right] \notag \\
    &= -\mathbb{E}\left[ \1(R=0)\omega(X) \left( \delta_\mu(1) - \delta_\mu(0) - \frac{A}{\pi}(\delta_\mu(1)) + \frac{1-A}{1-\pi}(\delta_\mu(0)) \right) \right] \notag \\
    &= -\mathbb{E}_X\left[ \omega(X) (1-\rho(X))\E\left[ \delta_\mu(1) - \delta_\mu(0) - \frac{A}{\pi}(\delta_\mu(1)) + \frac{1-A}{1-\pi}(\delta_\mu(0)) \mid X, R=0 \right] \right]
\end{align}
Here, we use the property that $\rho(x)=\E[R\mid X]$ (by definition), and $\E\left[\Omega \mid X\right]=0$ (see Eq.~(\ref{eq:Omega-mean-zero-property}). Conditioning on $X$ and $R=0$, and using $\mathbb{E}[A|X, R=0]=\pi(X)$:
\begin{align}
    \mathbb{E}[ \delta_\mu(1) - \delta_\mu(0) - \frac{A}{\pi}(\delta_\mu(1)) + \frac{1-A}{1-\pi}(\delta_\mu(0)) \mid X, R=0 ] = \delta_\mu(1)(1 - \frac{\pi}{\pi}) - \delta_\mu(0)(1 - \frac{1-\pi}{1-\pi}) = 0.
\end{align}
Therefore, we have that $\partial_r \mathbb{E}[S]\Big|_{r=0} = 0$.

\textbf{2. Orthogonality w.r.t. surrogate outcome $h(s,x)$:}

The nuisance $h$ appears in both $\hat{\tau}_{\mathrm{AIPW}}$ and $\psi_{\mathrm{obs}}$.
For the AIPW term, the derivative contribution is:
\begin{align}
    D_1 &= -\mathbb{E}\left[ \1(R=0)\omega(X) \left( \frac{A}{\pi}\delta_h - \frac{1-A}{1-\pi}\delta_h \right) \right] \notag \\
    &= -\E_X\left[(1-\rho(X))\mathbb{E}_S\left[ \omega(X) \frac{A - \pi}{\pi(1-\pi)} \delta_h \mid X, R=0\right]\right] \notag \\
    &= -\mathbb{E}_X\left[ (1-\rho(X))\E_{S}\left[\omega(X) \frac{\pi_s(S,X) - \pi}{\pi(1-\pi)} \delta_h \mid X, R=0\right]\right],
\end{align}
For the $\psi_{\mathrm{obs}}$ term,  we make use of the density ratio property that $\frac{\diff P(S \mid R=1, X)}{\diff P(S \mid R=0, X)} = \frac{\rho_s(S,X)}{1-\rho_s(S,X)} \frac{P(R=0|X)}{P(R=1|X)}$ to transform the measure from the observational to the experimental sample:
\begin{align}
    D_2 &= -\mathbb{E}\left[ \1(R=1)\omega(X) \frac{1-\rho_s(S,X)}{\rho_s(S,X)} \frac{\pi_s(S,X)-\pi(X)}{\pi(1-\pi)} (-\delta_h) \right] \notag \\
    &= \mathbb{E}_{X}\left[\E_S\left[ \rho(X) \frac{1-\rho_s(S,X)}{\rho_s(S,X)} \omega(X) \frac{\pi_s(S,X) - \pi}{\pi(1-\pi)} \delta_h \mid X, R=1 \right]\right] \notag \\
    &=\mathbb{E}_{X}\left[\E_S\left[(1-\rho(X))\omega(X) \frac{\pi_s(S,X) - \pi}{\pi(1-\pi)} \delta_h \mid X, R=0\right]\right].
\end{align}
Thus $D_1 + D_2 = 0$.

\textbf{3. Orthogonality w.r.t. $\pi(x)$:}

Let $\pi_r(X)=\pi(X)+r\,\delta_\pi(X)$, then we take the derivative w.r.t. $r$:
\begin{align}
\partial_r \E[S(g,\eta+r\delta_\pi)]\Big|_{r=0}
&=
\E\Big[
\1(R=0)\,\partial_r\omega_r(X)\Big|_{r=0}\,\big(g(X)-\hat\tau_{AIPW,0}\big)
-\1(R=0)\omega(X)\,\partial_r\hat\tau_{AIPW,r}\Big|_{r=0} \Big] \notag \\
&\quad
-\E\Big[\1(R=1)\,\partial_r\omega_r(X)\Big|_{r=0}\,\psi_{obs,0}
+\1(R=1)\omega(X)\,\partial_r\psi_{obs,r}\Big|_{r=0}\Big] \notag \\
&\quad
-\E\Big[\big((\mu(1,X)-\mu(0,X))-g(X)\big)\,\partial_r\Omega_r(Z)\Big|_{r=0}\Big].
\label{eq:pi-derivative-decomp}
\end{align}

First, $\partial_r\omega_r(X)|_{r=0} = \frac{\partial\omega}{\partial\pi}(X)\,\delta_\pi(X)$. Then
\begin{align}
    F_1:=\E\Big[
\1(R=0)\,\partial_r\omega_r(X)\Big|_{r=0}\,\big(g(X)-\hat\tau_{AIPW,0}\big)\Big] = \E\Big[\1(R=0)\frac{\partial\omega}{\partial\pi}(X)\,\delta_\pi(X) (g(X)-\tau^0(X))\Big]
\end{align}

Second, we have
\begin{align}
\label{proof-partial-pi-1}
    &\E\left[\1(R=0)\partial_r\hat\tau_{AIPW,r}\Big|_{r=0} \mid X\right] \\
    &= (1-\rho(X))\E\left[\partial_r\Bigl(\frac{A}{\pi_r(X)}\Bigr)\Big|_{r=0}(h(S,X) - \mu(1,X)) - \partial_r\Bigl(\frac{1-A}{1-\pi_r(X)}\Bigr)\Big|_{r=0}(h(S,X) - \mu(0,X)) \mid X, R=0\right] \notag \\
    &= (1-\rho(X))\E\left[A\frac{-1}{\pi(X)^2}\delta_\pi(X)(h(S,X) - \mu(1,X)) - (1-A)\frac{1}{(1-\pi(X))^2}\delta_\pi(X)(h(S,X) - \mu(0,X)) \mid X, R=0\right] \notag \\
    &= (1-\rho(X))\delta_\pi(X)\Biggl(\frac{-1}{\pi(X)^2} \underbrace{\E\left[A(h(S,X) - \mu(1,X))  \mid X, R=0\right]}_{=0}  \\
    & \qquad \qquad \qquad \qquad -\frac{1}{(1-\pi(X))^2}\underbrace{\E\left[(1-A)(h(S,X) - \mu(0,X)) \mid X, R=0\right]}_{=0}\Biggr) \notag \\
    &=0
\end{align}
The last equation is zero because $\forall a \in \{0,1\}, \quad \E\left[h(S,X) - \mu(a,X) \mid A=a, X,R=0\right]=0$. Hence
\begin{align}
    F_2:= \E\left[\omega(X)\1(R=0)\partial_r\hat\tau_{AIPW,r}\Big|_{r=0} \right] = 0.
\end{align}

From the property in Eq.~(\ref{eq:psi-mean-zero-property}), we have that $\E\left[\1(R=1)\psi_{\mathrm{obs}}\mid X\right] = 0$. Furthermore, with the exact same derivation, we also have $\E\left[\1(R=1)\partial_r\psi_{obs,r} \mid X\right] = 0$ (because the $Y-h(S,X)$ is mean zero). Hence, the second row in Eq.~(\ref{eq:pi-derivative-decomp}) can be simplified to
\begin{align}
\label{proof-paritial-pi-2}
    F_3:=-&\E\Big[\1(R=1)\,\partial_r\omega_r(X)\Big|_{r=0}\,\psi_{obs,0}
+\1(R=1)\omega(X)\,\partial_r\psi_{obs,r}\Big|_{r=0}\Big] \notag \\
=& \E\Big[\partial_r\omega_r(X)\E\left[\1(R=1)\Big|_{r=0}\,\psi_{obs,0}\mid X\right]
+\omega(X)\,\E\left[\1(R=1)\partial_r\psi_{obs,r}\Big|_{r=0}\mid X\right]\Big] = 0.
\end{align}

Finally, we address the terms involving $\partial_r\Omega_r|_{r=0}$.
Since $\omega$ is a function of both $\pi$ and $\rho$, both components of $\Omega$ depend on $\pi_r$:
$\Omega_r=\1(R=0)\frac{\partial\omega(\pi_r,\rho)}{\partial\pi_r}(A-\pi_r(X))+(1-\rho(X))\frac{\partial\omega(\pi_r,\rho)}{\partial\rho}(R-\rho(X))$.
Differentiating with respect to $r$,
\begin{align}
\partial_r\Omega_r\Big|_{r=0}
&=\1(R=0)\left[
\frac{\partial^2\omega}{\partial\pi^2}(X)\,\delta_\pi(X)\,(A-\pi(X))
-\frac{\partial\omega}{\partial\pi}(X)\,\delta_\pi(X)
\right]
+(1-\rho(X))\frac{\partial^2\omega}{\partial\pi\partial\rho}(X)\,\delta_\pi(X)\,(R-\rho(X)).
\label{eq:Omega-pi-derivative}
\end{align}
Taking conditional expectation given $X$, and using $\E[A-\pi(X)\mid X,R=0]=0$ and $\E[R-\rho(X)\mid X]=0$,
\begin{align}
\E\!\left[\partial_r\Omega_r\Big|_{r=0}\ \Bigm|\ X\right]
&=
(1-\rho(X))\left[
\frac{\partial^2\omega}{\partial\pi^2}(X)\,\delta_\pi(X)\,\underbrace{\E[A-\pi\mid X,R=0]}_{=0}
-\frac{\partial\omega}{\partial\pi}(X)\,\delta_\pi(X)
\right] \notag\\
&\quad+(1-\rho(X))\frac{\partial^2\omega}{\partial\pi\partial\rho}(X)\,\delta_\pi(X)\,\underbrace{\E[R-\rho(X)\mid X]}_{=0} \notag\\
&=-(1-\rho(X))\frac{\partial\omega}{\partial\pi}(X)\,\delta_\pi(X).
\label{eq:Omega-pi-mean}
\end{align}
Therefore,
\begin{align}
\label{proof-NO-paritial-Omega}
F_4:=-\E\Big[\big((\mu(1,X)-\mu(0,X))-g(X)\big)\,\partial_r\Omega_r\Big|_{r=0}\Big]
&=
(1-\rho(X))\E\Big[\big(\tau^0(X)-g(X)\big)\frac{\partial\omega}{\partial\pi}(X)\delta_\pi(X)\Big].
\end{align}

Putting everything together, we have
\begin{align}
    \partial_r \E[S(g,\eta+r\delta_\pi)]\Big|_{r=0}=&F_1 + F_2+F_3+F_4 \notag \\
    =& \E\Big[\1(R=0)\frac{\partial\omega}{\partial\pi}(X)\,\delta_\pi(X) (g(X)-\tau^0(X))\Big]  \\
    &\quad + \E\Big[(1-\rho(X))\frac{\partial\omega}{\partial\pi}(X)\,\delta_\pi(X)\big(\tau^0(X)-g(X)\big)\Big] \notag \\
    =&0.
\end{align}

\textbf{4. Orthogonality w.r.t. $\pi_s(s,x), \rho_s(S,X)$:}

Let $\pi_r(S,X)=\pi_s(S,X)+r\,\delta_{\pi_s}(S,X)$. Since $\pi_s(S,X)$ only appears in $\psi_{\mathrm{obs}}$, therefore
\begin{align}
\partial_r\psi_{obs,r}\Big|_{r=0}
=\left(\partial_r (\frac{1-\rho_s(S,X)}{\rho_s(S,X)} \left( \frac{\pi_s(S,X) - \pi(X)}{\pi(X)(1-\pi(X))} \right))\Big|_{r=0}\right)\big(Y-h(S,X)\big).
\end{align}
Conditioning on $(S,X,R=1)$, we have $\E[Y-h(S,X)\mid S,X,R=1]=0$, and therefore by tower law
\begin{align}
\partial_r \E[S(g,\eta+r\delta_{\pi_s})]\Big|_{r=0} = -\E\left[\1(R=1)\omega(X)\,\partial_r\psi_{obs,r}\Big|_{r=0}\right]=0.
\end{align}
The same derivation applies for $\rho_s(S,X)$.

\textbf{5. Orthogonality w.r.t. $\rho(x)$:}

Let $\rho_r(X)=\rho(X)+r\,\delta_\rho(X)$. The nuisance $\rho(X)$ appears in $\omega(X)$ and in $\Omega(Z;\eta)$, but not in $\hat\tau_{AIPW}$ or $\psi_{\mathrm{obs}}$.
Differentiating the gradient yields two types of terms:
\begin{align}
\partial_r \E[S(g,\eta+r\delta_\rho)]\Big|_{r=0}
&=
\E\Big[\partial_r\omega_r(X)\Big|_{r=0}\cdot\big(\1(R=0)(g-\hat\tau_{AIPW})-\1(R=1)\psi_{\mathrm{obs}}\big)\Big] \notag\\
&\quad
-\E\Big[\big((\mu(1)-\mu(0))-g\big)\,\partial_r\Omega_r\Big|_{r=0}\Big].
\label{eq:rho-derivative-decomp}
\end{align}
For the first term, using properties in Eq.~(\ref{eq:tau-AIPW-mean-tau}) and Eq.~(\ref{eq:psi-mean-zero-property}), we have
\begin{align}
&\E\Big[\partial_r\omega_r(X)\Big|_{r=0}\cdot\big(\1(R=0)(g-\hat\tau_{AIPW})-\1(R=1)\psi_{\mathrm{obs}}\big)\Big] \notag\\
&=
\E_X\Bigg[
\frac{\partial\omega}{\partial\rho}(X)\,\delta_\rho(X)\;
\E\Big[\1(R=0)\big(g(X)-\hat\tau_{AIPW}\big)-\1(R=1)\psi_{\mathrm{obs}} \mid X\Big]
\Bigg] \notag\\
&=
\E_X\Bigg[
\frac{\partial\omega}{\partial\rho}(X)\,\delta_\rho(X)\;
\Big(
\Pr(R=0\mid X)\E[g(X)-\hat\tau_{AIPW}\mid X,R=0]
-\Pr(R=1\mid X)\E[\psi_{\mathrm{obs}}\mid X,R=1]
\Big)
\Bigg] \notag\\
&=
\E_X\Bigg[
\frac{\partial\omega}{\partial\rho}(X)\,\delta_\rho(X)\;
(1-\rho(X))\big(g(X)-\tau^0(X)\big)
\Bigg].
\label{eq:rho-first-term}
\end{align}

For the second term, since $\omega$ depends on both $\pi$ and $\rho$, both components of $\Omega$ depend on $\rho_r$:
\begin{align}
\Omega_r(Z)
=\1(R=0)\frac{\partial\omega(\pi,\rho_r)}{\partial\pi}(X)(A-\pi(X))
+(1-\rho_r(X))\frac{\partial\omega(\pi,\rho_r)}{\partial\rho_r}(X)\big(R-\rho_r(X)\big).
\end{align}
Differentiating and taking conditional expectation given $X$: the cross-derivative $\1(R=0)\frac{\partial^2\omega}{\partial\rho\partial\pi}(X)\,\delta_\rho(X)(A-\pi(X))$ vanishes by $\E[A-\pi(X)\mid X,R=0]=0$, and all remaining terms involving $(R-\rho)$ vanish by $\E[R-\rho(X)\mid X]=0$, leaving:
\begin{align}
    \E\left[\partial_r\Omega_r\Big|_{r=0} \mid X\right] = -(1-\rho(X))\frac{\partial\omega}{\partial\rho}(X)\,\delta_\rho(X).
\end{align}
Since $(\mu(1,X)-\mu(0,X))-g(X)$ is $X$-measurable,
\begin{align}
-\E\Big[\big((\mu(1)-\mu(0))-g\big)\,\partial_r\Omega_r\Big|_{r=0}\Big]
&=
-\E_X\Big[
\big((\mu(1,X)-\mu(0,X))-g(X)\big)\,
\E\!\left[\partial_r\Omega_r\Big|_{r=0}\ \Bigm|\ X\right]
\Big] \notag\\
&=
\E_X\Big[
\big((\mu(1,X)-\mu(0,X))-g(X)\big)\,
(1-\rho(X))\frac{\partial\omega}{\partial\rho}(X)\,\delta_\rho(X)
\Big].
\label{eq:rho-second-term}
\end{align}
Adding the terms Eq.~(\ref{eq:rho-first-term}) and Eq.~(\ref{eq:rho-second-term}) together, we have
\begin{align}
\partial_r \E[S(g,\eta+r\delta_\rho)]\Big|_{r=0}
&=
\E_X\Big[
(1-\rho(X))\frac{\partial\omega}{\partial\rho}(X)\,\delta_\rho(X)
\Big\{\big(g(X)-\tau^0(X)\big)+\big(\mu(1,X)-\mu(0,X)-g(X)\big)\Big\}
\Big] \notag\\
&=
\E_X\Big[
(1-\rho(X))\frac{\partial\omega}{\partial\rho}(X)\,\delta_\rho(X)\cdot 0
\Big]=0.
\end{align}

Combining these results, the total derivative $\partial_r \mathcal{E}(\eta + r\delta)|_{r=0} = 0$, proving Neyman-orthogonality.

\subsection{Proof of Theorem~\ref{thm:equiv-loss}}
\label{app:proof-equiv}
The loss $\mathcal{L}_\omega(g;\eta)$ is, up to a constant (w.r.t. $g$), equivalent with:
\begin{equation}
    \mathcal{R}_\omega(g;\eta)=\E\left[\1(R=0) \omega(X) \big(\hat{\tau}_{\mathrm{AIPW}} -g \big)^2- 2\1(R=1) \omega(X) \psi_{\mathrm{obs}}\cdot g + \Omega\big( (\mu(1,X)-\mu(0,X)) - g \big)^2\right].
\end{equation}
Recall that
\begin{equation}
    \psi_{\mathrm{obs}}(Z;\eta) = \frac{1-\rho_s(S,X)}{\rho_s(S,X)} \left( \frac{\pi_s(S,X) - \pi(X)}{\pi(X)(1-\pi(X))} \right) \big( Y - h(S,X) \big).
\end{equation}
Hence, the conditional expectation of $\1(R=1)\psi_{\mathrm{obs}}(Z;\eta)$ given $X$ equals zero:
\begin{align}
    &\E\left[\1(R=1)\psi_{\mathrm{obs}}(Z;\eta) \mid X\right] \\
    &= \E\left[ \1(R=1) \frac{1-\rho_s(S,X)}{\rho_s(S,X)} \left( \frac{\pi_s(S,X) - \pi(X)}{\pi(X)(1-\pi(X))} \right) \big( Y - h(S,X) \big) \Bigg| X \right] \\
    &= \E\left[ \1(R=1) \frac{1-\rho_s(S,X)}{\rho_s(S,X)} \left( \frac{\pi_s(S,X) - \pi(X)}{\pi(X)(1-\pi(X))} \right) \underbrace{\E[Y - h(S,X) \mid S, X, R=1]}_{=0} \Bigg| X \right] \\
    &= 0.
\end{align}
Note that the conditional expectation of $\Omega(Z;\eta)$ given $X$ is also zero:
\begin{align}
    \E\left[\Omega(Z;\eta) \mid X\right] &= \E\left[ \1(R=0) \frac{\partial \omega}{\partial \pi} (A - \pi(X)) + (1-\rho(X)) \frac{\partial \omega}{\partial \rho} (R - \rho(X)) \Bigg| X \right] \\
    &= \frac{\partial \omega}{\partial \pi} \E\left[\1(R=0)(A - \pi(X)) \mid X\right] + (1-\rho(X)) \frac{\partial \omega}{\partial \rho} \underbrace{\E\left[R - \rho(X) \mid X\right]}_{=0} \\
    &= \frac{\partial \omega}{\partial \pi} \Pr(R=0|X) \underbrace{\E\left[A - \pi(X) \mid X, R=0\right]}_{=0} = 0. \label{exp-Omega-zero}
\end{align}
Hence, the risk function simplifies to
\begin{equation}
\label{eq:simplified-risk}
    \mathcal{R}_\omega(g;\eta)=\E\left[\1(R=0) \omega(X) \big(\hat{\tau}_{\mathrm{AIPW}}(Z;\eta) -g(X) \big)^2\right],
\end{equation}
where 
\begin{equation}
    \hat{\tau}_{\mathrm{AIPW}}(Z;\eta) = \mu(1,X) - \mu(0,X) + \frac{A}{\pi(X)}(h(S,X) - \mu(1,X)) - \frac{1-A}{1-\pi(X)}(h(S,X) - \mu(0,X)).
\end{equation}
We further write the Eq.~(\ref{eq:simplified-risk}) as:
\begin{align}
    \mathcal{R}_\omega(g;\eta) &=\E\left[\1(R=0) \omega(X) \big(\hat{\tau}_{\mathrm{AIPW}}(Z;\eta) -g(X) \big)^2\right] \\
    &= \E\left[\1(R=0) \omega(X) \big(\hat{\tau}_{\mathrm{AIPW}}(Z;\eta)- \tau^0(X) + \tau^0(X) - g(X) \big)^2\right] \\
    &= \E\left[\1(R=0) \omega(X) \bigl(\hat{\tau}_{\mathrm{AIPW}}(Z;\eta)- \tau^0(X)\bigr)^2\right] + \E\left[\1(R=0) \omega(X) \bigl(\tau^0(X) - g(X) \bigr)^2\right] \nonumber \\
    &\quad + 2\E\left[\1(R=0) \omega(X) \bigl(\hat{\tau}_{\mathrm{AIPW}}(Z;\eta)- \tau^0(X)\bigr)\bigl(\tau^0(X) - g(X) \bigr)\right] \\
    &= \E\left[\1(R=0) \omega(X) \bigl(\hat{\tau}_{\mathrm{AIPW}}(Z;\eta)- \tau^0(X)\bigr)^2\right] +  \E\left[\1(R=0) \omega(X) \bigl(\tau^0(X) - g(X)\bigr)^2\right] \\
    &= p_0 \E\left[\omega(X)\bigl(\tau^0(X) - g(X)\bigr)^2 \mid R=0\right] + \text{Const}_g \\
    &\propto p_0 \E\left[\omega(X)\bigl(\tau^0(X) - g(X)\bigr)^2 \mid R=0\right] = \mathcal{L}^{\text{oracle}}_\omega(g;\eta).
\end{align}
The reason the cross term in the decomposition is zero is because
\begin{align}
    &\E\left[\1(R=0) \omega(X) \bigl(\hat{\tau}_{\mathrm{AIPW}}(Z;\eta)- \tau^0(X)\bigr)\bigl(\tau^0(X) - g(X) \bigr)\right] \nonumber \\
    &= \E\left[ \E\left[ \1(R=0) \omega(X) \bigl(\hat{\tau}_{\mathrm{AIPW}}(Z;\eta)- \tau^0(X)\bigr)\bigl(\tau^0(X) - g(X) \bigr) \mid X \right] \right] \\
    &= \E\left[ \1(R=0) \omega(X) \bigl(\tau^0(X) - g(X) \bigr) \underbrace{\E\left[ \hat{\tau}_{\mathrm{AIPW}}(Z;\eta)- \tau^0(X) \mid X, R=0 \right]}_{=0} \right] \\
    &= 0.
\end{align}
Therefore we successfully show that $\mathcal{R}_\omega(g;\eta)$ has the same minimizer as the oracle loss $\mathcal{L}^{\text{oracle}}_\omega(g;\eta)$. Since $\mathcal{R}_\omega(g;\eta)$ is also equivalent with $\mathcal{L}_\omega(g;\eta)$, we finish the proof.

\subsection{Proof of Error bounds}
\label{app:proof-error-bound}

\subsubsection{Assumptions}
\label{app:error-bound-assumptions}
Here, we give the assumptions for deriving the error bound.

\textbf{Notation:} $\|f\| := \|f\|_2 = \sqrt{\E[f^2]}$ denotes the $L_2$ norm under the marginal distribution of $f$'s arguments ($X$ for $\hat g, \hat\mu, \hat\pi, \hat\rho$, and $(S,X)$ for $\hat h, \hat\pi_s, \hat\rho_s$), and $\|f\|_4 := (\E[f^4])^{1/4}$ the corresponding $L_4$ norm.

\begin{assumption}[Nondegenerate overlap]
\label{nondegen-overlap}
\emph{There exists $\alpha_\pi>0, \alpha_\rho>0$ such that}
    \begin{align}
        \forall x, \quad\alpha_\pi<\pi(x) <1-\alpha_\pi, \quad \rho(x)>\alpha_\rho.
    \end{align}
\end{assumption}
\begin{assumption}[Boundedness]
\label{boundedness}
\emph{(i) All nuisance estimators are bounded, i.e., $\exists M >0, r>0$, s.t. $\forall \hat\eta \in \mathcal{B}_r(\eta,\mathcal{G})$, $\|\hat\eta\| \leq M$. (ii)~All random variables are bounded. 
(iii)~$\exists \alpha_\omega >0, r>0$, s.t. $\forall \hat\eta \in \mathcal{B}_r(\eta,\mathcal{G})$, we have $\hat\omega(X)>\alpha_\omega$.}
\end{assumption}
\begin{assumption}[Smoothness of $\omega$]
\label{smoothness}
\emph{The weighting function $\omega(x)$ is lower bounded by a positive $\alpha_\omega>0$ and has at least a second-order derivative. Furthermore, the spectral radius of the Hessian matrix below is always bounded:}
\begin{align}
    \text{sup}_{X\in\mathcal{X}}\lambda_{\max}(\E\left[ \nabla^2_\eta \omega^*(Z;\eta) \mid X \right]) < +\infty.
\end{align}   
\end{assumption}
\begin{assumption}
\label{core-assumption}
$\exists\alpha >0, r>0$, \emph{s.t.} $\forall \hat\eta \in \mathcal{B}_r(\eta,\mathcal{G})$, 
\footnotesize
\begin{align}
    \text{inf}_{g \in \mathcal{G}/\{g_0\}} \frac{\E\left[\omega^*(Z;\hat\eta) (g(X) - g_0(X))^2\right]}{\E\left[(g(X) - g_0(X))^2\right]} = \alpha > 0,
\end{align}
\vspace{-0.3cm}
\normalsize
\end{assumption}
Assumptions~\ref{nondegen-overlap}, \ref{boundedness}, \ref{smoothness} are standard in causal inference literature~\citep{Foster.2023, Kennedy.2023}. Assumption~\ref{core-assumption} is easy to satisfy since $\omega^*(Z;\eta) = \1(R=0) \omega(X) + \Omega(Z;\eta)$, with $\omega(X) > \alpha_\omega>0$ and $\E\left[\Omega(Z;\eta) \mid X\right] = 0$.

\subsubsection{Proof}
We adopt a similar proof as in \citet{Morzywolek.2023} to derive the error bound for orthogonal learners. We use (functional) Taylor-expansion w.r.t. $g$ to the population loss $\mathcal{L}_\omega(g, \hat\eta)$ around $g_0$ to get that $\exists \bar{g} \in B(g_0, \mathcal{G})$ s.t.
\begin{align}
\label{taylor-exp}
     \mathcal{L}_\omega(\hat{g}, \hat{\eta}) = \mathcal{L}_\omega(g_0, \hat{\eta}) + D_g\mathcal{L}_\omega(g_0, \hat{\eta})[\hat{g} - g_0] + \frac{1}{2} D_g^2 \mathcal{L}_\omega(\bar{g}, \hat{\eta})[\hat{g} - g_0, \hat{g} - g_0].
\end{align}

We can compute the second-order term:
\begin{align}
    &D_g^2\mathcal{L}_\omega(\bar{g}, \hat\eta)[\hat{g}-g_0, \hat{g}-g_0] \notag \\
    =& \frac{\partial^2}{\partial t_1\partial t_2} \E\left[\hat\omega^* (\widehat{\mathcal{T}} - \bar{g}(X)-t_1(\hat{g}(X)- g_0(X)) - t_2(\hat{g}(X) - g_0(X))\right] \Bigg |_{t_1=0,t_2=0} \notag \\
    =& \E\left[\hat\omega^* (\hat{g}(X) - g_0(X))^2\right]
\end{align}
With the assumption that
\begin{align}
    \text{inf}_{g \in \mathcal{G}/\{g_0\}} \frac{\E\left[\omega^* (g(X) - g_0(X))^2\right]}{\E\left[(g(X) - g_0(X))^2\right]} = \alpha > 0,
\end{align}
We have 
\begin{align}
    D_g^2\mathcal{L}_\omega(\bar{g}, \hat\eta)[\hat{g}-g_0, \hat{g}-g_0] \geq  \alpha \|\hat{g}-g_0\|_2^2.
\end{align}
We substitute the inequality back into Eq.~(\ref{taylor-exp}) and achieve an primary upper bound for $\|\hat{g}-g_0\|_2^2$:
\begin{equation}
\label{initial-bound}
    \frac{\alpha}{2} \|\hat{g}-g_0\|_2^2 \leq \underbrace{\mathcal{L}_\omega(\hat{g}, \hat{\eta}) - \mathcal{L}_\omega(g_0, \hat{\eta})}_{:=R_g \text{ oracle excess risk}} - D_g\mathcal{L}_\omega(g_0, \hat{\eta})[\hat{g} - g_0]
\end{equation}
In the following proof, we show that $D_g\mathcal{L}_\omega(g_0, \hat{\eta})[\hat{g} - g_0]$ can be decomposed into second-order cross multiplication of nuisance errors:
\begin{align}
    D_g\mathcal{L}_\omega(g_0, \hat{\eta})[\hat{g} - g_0]=& \E\left[ \left(\underbrace{\1(R=0) \hat{\omega}(X) \hat{\tau}_{\mathrm{AIPW}} + \1(R=1) \hat{\omega}(X) \hat\psi_{\mathrm{obs}}}_{\text{Term (1)}} + \underbrace{\big( (\hat\mu_1-\hat\mu_0\big) \hat{\Omega} - \hat{\omega^*}g_0}_{\text{Term (2)}}\right)(\hat{g} - g_0)\right]
\end{align}
We first analyze Term (1). The AIPW pseudo-outcome $\hat{\tau}_{\mathrm{AIPW}}$ can be decomposed into:
\begin{align}
    \E\left[\hat{\tau}_{\mathrm{AIPW}} \mid X, R = 0\right] = & \E\left[ \hat{\mu}_1 - \hat{\mu}_0 + \frac{A}{\hat{\pi}}(\hat{h} - \hat{\mu}_1) - \frac{1-A}{1-\hat{\pi}}(\hat{h} - \hat{\mu}_0) \mid X, R=0 \right] \notag \\
    =& \E\Bigl[\underbrace{\hat{\mu}_1(X) - \hat{\mu}_0(X) + \frac{A}{\pi}(h-\hat\mu_1(X)) - \frac{1-A}{1-\pi}(h-\hat\mu_0(X))}_{\tau^0(X)}\mid X, R = 0\Bigr] \\ &+ \E\Bigl[\underbrace{\left(\frac{A}{\pi} -\frac{A}{\hat\pi}\right)(h-\hat{\mu}_1) - \left(\frac{1-A}{1-\pi} -\frac{1-A}{1-\hat\pi}\right)(h-\hat{\mu}_0)}_{\text{Bias}(\mu, \pi)}\mid X, R = 0\Bigr] \\
    &+ \E\left[ \left(\frac{A}{\hat{\pi}} - \frac{1-A}{1-\hat{\pi}}\right) \delta_h(S,X) \mid X, R=0 \right] \notag \\
    =& \tau^0(X) + \text{Bias}(\mu, \pi) + \underbrace{\E\left[ \left( \frac{\pi_s(S,X)}{\hat{\pi}(X)} - \frac{1-\pi_s(S,X)}{1-\hat{\pi}(X)} \right) \delta_h(S,X) \mid X, R=0 \right]}_{\text{First-order bias}}.
\end{align}
The first brace equals $\tau^0(X)$ because the use of the true propensity score $\pi$ eliminates the bias from $\hat\mu$ (double robustness):
\begin{align} 
&\E\left[ \hat{\mu}_1 + \frac{A}{\pi}(h-\hat\mu_1) \mid X, R=0 \right] \\
&= \hat{\mu}_1 + \frac{\E[A|X,R=0]}{\pi}(\E[h|X,A=1,R=0]-\hat\mu_1) \\ 
&= \hat{\mu}_1 + \frac{\pi}{\pi}(\mu_1-\hat\mu_1) = \mu_1. \quad (\text{Same goes for }\mu_0)
\end{align}
The second brace is the product of nuisance error between $\mu_0,\mu_1$ and $\pi$ because
\begin{align} 
&\E\left[ \left(\frac{A}{\pi(X)} -\frac{A}{\hat\pi(X)}\right)(h(S,X)-\hat{\mu}_1) \mid X, R=0 \right] \\
&=\left(\frac{1}{\pi(X)} -\frac{1}{\hat\pi(X)}\right)\E\left[ (h(S,X)-\hat{\mu}_1) \mid X, A=1,R=0 \right] \\
&= \pi(X) \left(\frac{\hat{\pi}(X) - \pi(X)}{\pi(X)\hat{\pi}(X)}\right) (\mu_1(X) - \hat{\mu}_1(X)) = \frac{\hat{\pi}(X) - \pi(X)}{\hat{\pi}(X)} (\mu_1(X) - \hat{\mu}_1(X)). 
\end{align}
Since $\pi(x)$ is bounded from 0 and 1 (overlap assumption), there exists a constant $C_0 > 0$ such that:
\begin{align}
\label{eq:tauDR-bound}
    \E\left[ \1(R=0) \hat{\omega}(X) \hat{\tau}_{\mathrm{AIPW}} \mid X \right] \leq \1(R=0)\cdot\tau^0(X)+ C_0|\hat{\pi}-\pi|\left(|\hat\mu_1-\mu_1|+|\hat\mu_0-\mu_0| \right) \notag \\
    + \E\left[ \1(R=0)\left( \frac{\pi_s(S,X)}{\hat{\pi}(X)} - \frac{1-\pi_s(S,X)}{1-\hat{\pi}(X)} \right) \delta_h(S,X) \mid X\right]
\end{align}

Then, we analyze the correction term $\hat\psi_{\mathrm{obs}}$. We let $\gamma(s,x):=\frac{1- \rho_s(s,x)}{\rho_s(s,x)}$, $C_\pi(S,X):=\frac{\pi_s(S,X) - \pi(X)}{\pi(X)(1-\pi(X))}$ and analyze the conditional expectation of $\hat{\omega}(X) \hat\psi_{\mathrm{obs}}$ given $X$:
\begin{align}
    \E\left[\1(R=1) \hat{\omega}(X) \hat\psi_{\mathrm{obs}} \mid X\right] =& \E\left[ \1(R=1) \hat{\omega}(X) \hat{\gamma}(S,X) \left( \frac{\hat{\pi}_s(S,X) - \hat{\pi}(X)}{\hat{\pi}(X)(1-\hat{\pi}(X))} \right) (Y - \hat{h}) \mid X\right] \notag \\
    =& \E\left[ \1(R=1) \hat{\omega}(X) \hat{\gamma}(S,X) \hat{C}_\pi(S,X) (-\delta_h(S,X)) \mid X\right] \notag \\
    =& -\E\left[ \1(R=0) \frac{\hat{\gamma}(S,X)}{\gamma_0(S,X)} \hat{\omega}(X) \hat{C}_\pi(S,X) \delta_h(S,X) \mid X\right].
\end{align}
Here, we make use of the density ratio property that $\frac{\diff P(S \mid R=1, X)}{\diff P(S \mid R=0, X)} = \frac{1}{\gamma_0(S,X)} \frac{P(R=0|X)}{P(R=1|X)}$ to transform the measure from the observational to the experimental sample.
Combining the two parts above, we have:
\begin{align}
    \text{Term (1)} &= \E\left[\1(R=0) \hat{\omega}(X) \hat{\tau}_{\mathrm{AIPW}} + \1(R=1) \hat{\omega}(X) \hat\psi_{\mathrm{obs}}\right] \notag \\
    =& \E\left[ \1(R=0) \hat{\omega}(X) \left( \tau^0(X) + \text{Bias}(\mu, \pi) + \delta_h(S,X) \underbrace{ \left( \left[ \frac{\pi_s(S,X) - \hat{\pi}(X)}{\hat{\pi}(1-\hat{\pi})} \right] - \frac{\hat{\gamma}}{\gamma_0} \hat{C}_\pi \right) }_{\text{Residual } \mathcal{R}_h} \right) \right] \notag \\
    \leq & \E\left[ \1(R=0) \hat{\omega}(X) \tau^0(X) \right] + \E\left[C_0|\hat{\pi}-\pi|\left(|\hat\mu_1-\mu_1|+|\hat\mu_0-\mu_0| \right)\right] \notag \\
    & + C_h\cdot\E\left[|\hat{h}(S,X) - h(S,X)| (|\hat\pi_s(S, X) - \pi_s(S,X)| + |\hat\gamma(S,X) - \gamma(S,X)|)\right].
    \label{term-1-bound}
\end{align}
In the last step, we define a positive constant $C_h > 0$ s.t.
\begin{align}
    \mathcal{R}_h &= \left( \left[ \frac{\pi_s(S,X) - \hat{\pi}(X)}{\hat{\pi}(1-\hat{\pi})} \right] - \frac{\hat{\gamma}}{\gamma_0} \hat{C}_\pi \right) \\
    & \leq \left( \left[ \frac{\pi_s(S,X) - \hat{\pi}(X)}{\hat{\pi}(X)(1-\hat{\pi}(X))} \right] - \frac{1-\hat{\rho}_s(S,X)}{1 - \rho_s(S,X)}  \frac{\rho_s(S,X)}{\hat\rho_s(S,X)} \frac{\hat{\pi}(S,X) - \hat{\pi}(X)}{\hat{\pi}(X)(1-\hat{\pi}(X))} \right) \notag \\
    & \leq C_h \bigl(|\hat\pi_s(S,X) - \pi_s(S,X)| + |\hat\rho_s(S,X) - \rho_s(S,X)|\bigr) .
\end{align}

Now, we only need to deal with the remaining terms $\big( (\hat\mu_1-\hat\mu_0\big) \hat{\Omega} - \hat{\omega^*}g_0$ plus the term $\E\left[ \1(R=0) \hat{\omega}(X) \tau^0(X) \right]$. We reorganize these terms as:
\begin{align}
\label{eq:remain-decompose}
&\underbrace{\1(R=0)\hat\omega(X)\tau^0(X)}_{\text{Remainder from Term (1)}} + \underbrace{\left(\hat\mu_1(X)-\hat\mu_0(X)\right)\hat\Omega - \hat{\omega^*}g_0}_{\text{Term (2)}} \notag \\
    &=\1(R=0)\hat\omega(X)\tau^0(X) - \hat{\omega^*}\tau^0(X) + \left(\hat\mu_1(X)-\hat\mu_0(X)\right)\hat\Omega +  \left(\hat{\omega^*}\tau^0(X)- \omega^*\tau^0(X)\right)\\
    & \qquad + \left(\omega^*\tau^0(X) - \omega^*g_0(X)\right) + \left(\omega^*g_0(X) - \hat{\omega^*}g_0(X)\right) \notag \\
    &= \underbrace{\left(\hat\mu_1(X)-\hat\mu_0(X) - \tau^0(X)\right)\hat\Omega}_{(\text{I})} + \underbrace{\left(\omega^*\tau^0(X) - \omega^*g_0(X)\right)}_{(\text{II})} + \underbrace{\left(\omega^* - \hat{\omega^*}\right)g_0(X)}_{(\text{III})}
\end{align}

The second equation is because $\hat{\omega^*} = \1(R=0)\hat\omega(X) + \hat\Omega$ per definition of $\omega^*$.
Note that
\begin{align}
     \E\left[\hat\Omega \mid X\right] &= \E\left[\1(R=0) \frac{\partial \omega}{\partial \pi}(\hat\pi(X)) (A - \hat\pi(X)) + (1-\hat\rho(X)) \frac{\partial \omega}{\partial \rho}(\hat\rho(X)) (R - \hat\rho(X)) \mid X\right] \notag \\
     &=\1(R=0) \frac{\partial \omega}{\partial \pi}(\hat\pi(X)) (\pi(X) - \hat\pi(X)) + (1-\hat\rho(X)) \frac{\partial \omega}{\partial\rho} (\rho(X) - \hat\rho(X)) \notag \\
     & \leq C_\Omega \cdot \left(|\pi(X) - \hat\pi(X)| + |\rho(X) - \hat\rho(X)|\right),
\end{align}
where $C_\Omega>0$ is a constant that upper bounds $\text{max}( |\frac{\partial \omega}{\partial \pi}|, |\frac{\partial \omega}{\partial \rho}|)$.
Hence, part $(\text{I})$ is of second order:

\begin{align} 
\label{part-I} 
(\text{I}) & = \left(\hat\mu_1(X)-\hat\mu_0(X) - \tau^0(X)\right)\hat\Omega \notag \\ 
& = \left( (\hat\mu_1(X) - \mu_1(X)) - (\hat\mu_0(X) - \mu_0(X)) \right) \hat{\Omega} \notag \\ 
& \leq \left( |\hat\mu_1(X) - \mu_1(X)| + |\hat\mu_0(X) - \mu_0(X)| \right) 
\cdot C_\Omega \left(|\pi(X) - \hat\pi(X)| + |\rho(X) - \hat\rho(X)|\right). \end{align}
For part (III), note that $ \Omega(Z;\eta)=\1(R=0) \frac{\partial \omega}{\partial \pi} (A - \pi(X)) + (1-\rho(X)) \frac{\partial \omega}{\partial \rho} (R - \rho(X))$ and $\omega^*(Z;\eta)=\1(R=0)\omega(X)+ \Omega(Z;\eta)$, we verify the vanishing gradients w.r.t the nuisances: 

\begin{align} 
&\E\left[D_\pi \omega^*(Z;\eta)[\hat{\pi}-\pi] \mid X \right] \\
&= \E\left[ (\hat{\pi}-\pi) \left(\1(R=0)\frac{\partial \omega}{\partial \pi} + \1(R=0)\frac{\partial^2 \omega}{\partial \pi^2}(A-\pi) + \1(R=0)\frac{\partial \omega}{\partial \pi}(-1) \right) \Bigg| X \right] \notag \\ 
&= (\hat{\pi}-\pi) \E\left[ \1(R=0) \frac{\partial^2 \omega}{\partial \pi^2} (A-\pi) \mid X \right] \notag \\ 
&= (\hat{\pi}-\pi) (1-\rho) \frac{\partial^2 \omega}{\partial \pi^2} \E[A-\pi \mid X, R=0] = 0. 
\label{omega*-derivative-pi}
\end{align}

\begin{align} 
&\E\left[D_\rho\omega^*(Z;\eta)[\hat\rho-\rho]\mid X\right] \notag \\
&= \left(\hat\rho(X) - \rho(X)\right)\E\left[ \1(R=0)\frac{\partial \omega}{\partial \rho} - \frac{\partial \omega}{\partial \rho}(R-\rho) + (1-\rho)(\frac{\partial^2 \omega}{\partial \rho^2}(R-\rho) - \frac{\partial \omega}{\partial \rho}) \mid X \right] \notag \\ 
&= (\hat\rho - \rho)\bigl[(1-\rho)\frac{\partial \omega}{\partial \rho} - 0 - (1-\rho)\frac{\partial \omega}{\partial \rho}\bigr]=0. 
\label{omega*-derivative-rho}
\end{align}
Hence, $\omega^*-\hat\omega$ is of second order of the nuisance errors (via Taylor expansion around $\eta_0$):
\begin{align} 
\E\left[\omega^*-\hat{\omega^*} \mid X\right] &= \E\left[ \omega^*(\eta) - \left( \omega^*(\eta) + D_\eta \omega^*[\hat{\eta}-\eta] + \frac{1}{2}D^2_\eta \omega^*[\hat{\eta}-\eta, \hat{\eta}-\eta] \right) \mid X\right] \notag \\ 
&= -\underbrace{\E\left[ D_\pi \omega^*[\delta_\pi] + D_\rho \omega^*[\delta_\rho] \mid X \right]}_{=0 \text{ from Eq.~(\ref{omega*-derivative-pi}), (\ref{omega*-derivative-rho})}} - \frac{1}{2}\E\left[ D^2_\eta \omega^*[\hat{\eta}-\eta, \hat{\eta}-\eta] \mid X \right] + o(|\hat{\eta}-\eta|^2) \notag \\
&=\frac{1}{2}\left | (\hat{\eta}-\eta)^\top \E\left[ \nabla^2 \omega^*(Z;\eta) \mid X \right] (\hat{\eta}-\eta) \right| + o(|\hat{\eta}-\eta|^2) \\
&\lesssim \text{sup}_{X}|C_{\omega^*}(X) \left(|\hat{\pi}-\pi|^2 + |\hat{\rho}-\rho|^2\right),
\end{align}
Where $C_{\omega^*}(X)$ is the spectral radius of the conditional expected Hessian matrix $\E\left[ \nabla^2 \omega^*(Z;\eta) \mid X \right]$.
Therefore, we can bound part (III) by:
\begin{equation}
\label{part-III}
    \E\left[(\omega^*-\hat{\omega^*})g_0(X) \mid X\right] \lesssim G_0\cdot\text{sup}_{X}|C_{\omega^*}(X) \left(\|\hat{\pi}-\pi\|^2 + \|\hat{\rho}-\rho\|^2\right).
\end{equation}

For part (II), we directly compute:
\begin{align}
    \label{part-II}
    &\E\left[\omega^*(\tau^0(X)-g_0(X))(\hat{g}(X)-g_0(X)) \right] \notag \\
    =& \E\left[\omega(X)(\tau^0(X)-g_0(X))(\hat{g}(X)-g_0(X)) \mid R = 0\right] \notag \\
    =& -D_g\mathcal{L}^{\text{oracle}}_\omega(g_0,\eta)[\hat{g} - g_0] =0\\
\end{align}
The second step is because $\omega^*(Z;\eta)=\1(R=0)\omega(X)+\Omega(Z;\eta)$ and $\E\left[\Omega(Z;\eta) \mid X\right] = 0$ (see Eq.~(\ref{exp-Omega-zero})). With true nuisance $\eta$, $g_0$ is also the minimizer of $\mathcal{L}^{\text{oracle}}$. Hence, $g_0$ is the orthogonal projection of $\tau^0$ onto $\mathcal{G}$. By the projection theorem, the residual $\tau^0-g_0$ is orthogonal to any function in the space, including $\hat{g}-g_0$. Therefore, the last expression is zero.

Hence, for $D_g\mathcal{L}_\omega(g_0, \hat{\eta})[\hat{g}-g_0]$ we have:
\begin{align}
    & D_g\mathcal{L}_\omega(g_0, \hat{\eta})[\hat{g}-g_0] \notag \\
   =& \E\left[ \Bigl(\1(R=0) \hat{\omega}(X) \hat{\tau}_{\mathrm{AIPW}} + \1(R=1) \hat{\omega}(X) \hat\psi_{\mathrm{obs}} + \big( (\hat\mu_1-\hat\mu_0\big) \hat{\Omega} - \hat{\omega^*}g_0\Bigr)(\hat{g} - g_0)\right] \\
   \lesssim& \E\left[ \1(R=0) \hat{\omega}(X) \tau^0(X)  + \big( (\hat\mu_1-\hat\mu_0\big) \hat{\Omega} - \hat{\omega^*}g_0 \right] \text{\qquad \qquad \qquad part (I)+(II)+(III)} \notag \\ 
   &+\E\left[C_0|\hat{\pi}(X)-\pi(X)|\left(|\hat\mu_1(X)-\mu_1(X)|+|\hat\mu_0(X)-\mu_0(X)| \right)(\hat{g}(X) - g_0(X))\right]  \notag \\
   &+ C_h\cdot\E\left[|\hat{h}(S,X) - h(S,X)| (|\hat\pi(X) - \pi(X)| + |\hat\rho_s(S,X) - \rho_s(S,X)|) (\hat{g}(X) - g_0(X))\right]
\end{align}
With the previously derived result (see Eq.~(\ref{part-I}), Eq.~(\ref{part-II}) and Eq.~(\ref{part-III})), we have
\begin{align}
    & \text{part (I)+(II)+(III)} \\
    \lesssim & \E\left[\left( |\hat\mu_1(X) - \mu_1(X)| + |\hat\mu_0(X) - \mu_0(X)| \right) 
\cdot C_\Omega \left(|\pi(X) - \hat\pi(X)| + |\rho(X) - \hat\rho(X)|\right)(\hat{g}(X)-g_0(X))\right] \\
& + \E\left[\text{sup}_{X}|C_{\omega^*}(X) \left(|\hat{\pi}(X)-\pi(X)|^2 + |\hat{\rho}(X)-\rho(X)|^2\right)(\hat{g}(X)-g_0(X))\right] + 0.
\end{align}
Let $|\hat\mu(X)-\mu(X)|=\text{max} \{ |\hat\mu_1-\mu_0|, |\hat\mu_0-\mu_0| \} $. Applying Cauchy-Schwarz inequality to the decomposition of $D_g\mathcal{L}_\omega(g_0, \hat{\eta})[\hat{g}-g_0]$, we have that: $\exists C>0$ s.t.
\begin{align}
    & D_g\mathcal{L}_\omega(g_0, \hat{\eta})[\hat{g}-g_0] \\
    \leq& C\cdot\Biggl(\sqrt{\E\left[(\hat\pi(X) - \pi(X))^2 (\hat\mu(X) - \mu(X))^2\right]}  + \sqrt{\E\left[(\hat\rho(X) - \rho(X))^2 (\hat\mu(X) - \mu(X))^2\right]} \notag \\
     + &\sqrt{\E\left[(\hat\pi_s(S,X) - \pi_s(S,X))^2 (\hat{h}(S,X) - h(S,X))^2\right]}  +\sqrt{\E\left[(\hat\rho_s(S,X) - \rho_s(S,X))^2 (\hat{h}(S,X) - h(S,X))^2\right]} \\
     + & \sqrt{\E\left[(\hat\pi(X) - \pi(X))^4 \right]} + \sqrt{\E\left[(\hat\rho(X) - \rho(X))^4 \right]}
    \Biggr)\|\hat{g}-g_0\|.
\end{align}
Using the AM-GM inequality on the right hand, we have that: $\forall \delta > 0$ 
\begin{align}
    & D_g\mathcal{L}_\omega(g_0, \hat{\eta})[\hat{g}-g_0] \\
    &\leq \Biggl(\frac{1}{\delta}\E\left[(\hat\pi(X) - \pi(X))^2 (\hat\mu(X) - \mu(X))^2\right] + \frac{1}{\delta} \E\left[(\hat\rho(X) - \rho(X))^2 (\hat\mu(X) - \mu(X))^2\right]\notag \\
    &+ \frac{1}{\delta} \E\left[(\hat\pi_s(S,X) - \pi_s(S,X))^2 (\hat{h}(S,X) - h(S,X))^2\right] + \frac{1}{\delta} \E\left[(\hat\rho_s(S,X) - \rho_s(S,X))^2 (\hat{h}(S,X) - h(S,X))^2\right] \notag \\
    &+\frac{1}{\delta} \E\left[(\hat\pi(X)-\pi(X))^4\right] + \frac{1}{\delta} \E\left[(\hat\rho(X)-\rho(X))^4\right] + 6\delta\|\hat{g}-g_0\|^2 \Biggr) \cdot C
\end{align}
Substituting the inequality back to the initial bound in Eq.~(\ref{initial-bound}), we have that
\begin{align}
    &\frac{\alpha}{2}\|\hat{g}-g_0\|^2 \leq R_g + \Biggl(\frac{1}{\delta}\E\left[(\hat\pi(X) - \pi(X))^2 (\hat\mu_1(X) - \mu_1(X))^2\right] + \frac{1}{\delta} \E\left[(\hat\pi(X) - \pi(X))^2 (\hat\mu_0(X) - \mu_0(X))^2\right] \notag \\
    &+ \frac{1}{\delta} \E\left[(\hat\rho(X) - \rho(X))^2 (\hat\mu_1(X) - \mu_1(X))^2\right] + \frac{1}{\delta} \E\left[(\hat\rho(X) - \rho(X))^2 (\hat\mu_0(X) - \mu_0(X))^2\right] \notag \\
    &+ \frac{1}{\delta} \E\left[(\hat\pi_s(X) - \pi_s(X))^2 (\hat{h}(S,X) - h(S,X))^2\right] + \frac{1}{\delta} \E\left[(\hat\rho_s(S,X) - \rho_s(S,X))^2 (\hat{h}(S,X) - h(S,X))^2\right] \\
    &+6\delta\|\hat{g}-g_0\|^2\Biggr) \cdot C
\end{align}
Let $\delta < \frac{\alpha}{2} / 12C$, then:
\begin{align}
    \|\hat{g}-g_0\|^2 &\leq \frac{1}{\alpha/2 - 6C\delta} \Biggl(R_g + \frac{1}{\delta}\E\left[(\hat\pi(X) - \pi(X))^2 (\hat\mu(X) - \mu(X))^2\right] + \frac{1}{\delta} \E\left[(\hat\rho(X) - \rho(X))^2 (\hat\mu(X) - \mu(X))^2\right] \notag \\
    &+\frac{1}{\delta} \E\left[(\hat\pi_s(X) - \pi_s(X))^2 (\hat{h}(S,X) - h(S,X))^2\right]  + \frac{1}{\delta} \E\left[(\hat\rho_s(S,X) - \rho_s(S,X))^2 (\hat{h}(S,X) - h(S,X))^2\right] \notag \\
    &+\frac{1}{\delta} \E\left[(\hat\pi(X)-\pi(X))^4\right] + \frac{1}{\delta} \E\left[(\hat\rho(X)-\rho(X))^4\right] \Biggr).
\end{align}
Let $\|\hat\mu-\mu\|_4=\text{max} \{ \|\hat\mu_1-\mu_1\|_4, \|\hat\mu_0-\mu_0\|_4 \}$ and use Cauchy-Schwarz inequality again, we have:
\begin{align}
    \|\hat{g}-g_0\|^2 &\lesssim R_g + \|\hat\mu-\mu\|_4^2(\|\hat\rho-\rho\|_4^2 + \|\hat\pi-\pi\|_4^2) + \|\hat{h}-h\|_4^2( \|\hat\pi_s-\pi_s\|_4^2 + \|\hat\rho_s-\rho_s\|_4^2) + \|\hat\pi-\pi\|_4^4 + \|\hat\rho-\rho\|_4^4
\end{align}
\newpage

\section{Datasets}
\subsection{Synthetic Data}
\label{app:synthetic}

Here, we take inspiration from \citet{Nie.2021, Morzywolek.2023}, and extend the data-generating process to a two-sample setting. The generic simulation steps are:
\begin{enumerate}
    \item Sample belongingness $R$: First, we use the pre-defined true nuisance $\rho(x)$ to decide the belonging of the sample with covariate $X\sim P_X$. $R \in \{0,1\}$ is randomly drawn from $\text{Ber}(\rho(X))$.
    \item Treatment in $\mathcal{D}_1$: $A\mid X \sim \text{Ber}(\pi(X))$. Here $\pi$ is the true treatment propensity score function.
    \item Implicit treatment in observational data: $A\mid X \sim \text{Ber}(e(X))$. Here $e$ is a pre-defined implicit treatment propensity score function.
    \item Short-term outcome: $S =  a(X) + (A-0.5)\tau_S(X) + \delta$, with $\delta \sim \mathcal{N}(0, \sigma_S)$.
    \item Long-term outcome in $\mathcal{D}_2$: $Y = b(X,S) + (A - 0.5)\tau_Y(X) + \epsilon$. Here $\epsilon$ is exogenous noise variable $\epsilon \sim \mathcal{N}(0, \sigma_Y)$. Under our surrogacy assumption ($A \indep Y \mid S, X$), $\tau_Y^* = 0$. Then:
    \begin{equation}
        Y = b(S,X) + \epsilon.
    \end{equation}
\end{enumerate}
The true heterogeneous effect of $A$ on the long-term outcome $Y$ is:
\begin{equation}
\label{eq:true-cate-simulation}
    \tau(X)=\E \left[Y(1)-Y(0)\mid X, R=0\right] = \E\left[b(X,S(1)) - b(X, S(0)) \mid X, R=0\right].
\end{equation}

We now instantiate the above generic framework in analogy with Setup~A of \citet{Nie.2021} (with surrogacy). Let $X \sim \mathrm{Unif}([-1,1]^{10})$, and define the sample belonging propensity score
\begin{align}
    \rho(X) = \eta_\rho + (1 - \eta_\rho)\,\sigma\!\big(X_1 + X_2 - \gamma_\rho\big),
\end{align}
where $\sigma(\cdot)$ is the logistic sigmoid and $\eta_\rho \in [0,1)$ enforces a lower bound $\rho(X) \ge \eta_\rho$ (we use $\eta_\rho = 10^{-3}$). The parameter $\gamma_\rho$ controls long-term outcome overlap: larger $\gamma_\rho$ shifts $\rho$ toward $\eta_\rho$ and shrinks the observational sample.

The treatment propensity on the experimental dataset uses a standardized logit. Let
\begin{equation}
    h(X) := X_1 X_2 + X_3 + X_4 + X_8^2,
    \qquad
    \mu_h = \E[h(X)],
    \qquad
    \sigma_h = \sqrt{\var[h(X)]},
\end{equation}
under $X \sim \mathrm{Unif}([-1,1]^{10})$. Then
\begin{equation}
    \pi(X)
    = \eta_\pi + (1 - 2\eta_\pi)\,\sigma\!\Big(\gamma_\pi \cdot \tfrac{h(X) - \mu_h}{\sigma_h}\Big),
\end{equation}
with $\eta_\pi = 10^{-2}$ enforcing $\pi(X) \in [\eta_\pi, 1-\eta_\pi]$. The parameter $\gamma_\pi$ controls treatment overlap: larger $\gamma_\pi$ pushes $\pi$ toward the bounds. Standardizing $h(X)$ keeps the effective overlap difficulty comparable across choices of $h$.

For the observational dataset, the implicit treatment propensity is
\begin{equation}
    e(X)
    = \mathrm{trim}_{0.1}\!\big(\sigma(X_2 + X_3 + X_4)\big),
    \qquad
    \mathrm{trim}_{\eta}(u)
    := \max\{\eta,\min(u,1-\eta)\}.
\end{equation}
For the short-term outcome $S$, we set
\begin{equation}
    a(X)
    = \sin(\pi X_1 X_2)
      + 2(X_3 - 0.5)^2
      + X_4 + 0.5 X_5 + X_6,
\end{equation}
and the short-term treatment effect $\tau_S(X) = 1 + \frac{X_1 + X_2 + X_3 + X_4}{4}$, so
\begin{equation}
    S
    = a(X)
      + (A - 0.5)\,\tau_S(X)
      + \delta,
    \qquad \delta \sim \mathcal{N}(0, \sigma_S^2),\ \sigma_S = 0.2.
\end{equation}
The long-term outcome regression is quadratic in $S$:
\begin{equation}
    Y = b(X,S) + \epsilon
      = \sin(X_1 X_2) + X_7^2 + X_8 + \tfrac{S^2}{4} + \epsilon,
    \qquad
    \epsilon \sim \mathcal{N}(0, 0.5^2).
\end{equation}
\textbf{Note}: Outcomes are normalized by the sample mean and standard deviation.

\paragraph{Experimental regimes.} We instantiate four regimes by varying $(\gamma_\pi, \gamma_\rho)$ at a fixed total sample size $n=10000$ (drawn before the $R$-split):
\begin{center}
\begin{tabular}{lcc r}
\toprule
Regime & $\gamma_\pi$ & $\gamma_\rho$ & $n$ \\
\midrule
$\mathcal{D}_{*}$\ (healthy)            & 0 & 0 & 10000 \\
$\mathcal{D}_{t}$\ (low treatment overlap)   & 5 & 0 & 10000 \\
$\mathcal{D}_{o}$\ (low outcome overlap)     & 0 & 4 & 10000 \\
$\mathcal{D}_{t+o}$\ (dual overlap)          & 5 & 3 & 10000 \\
\bottomrule
\end{tabular}
\end{center}

\subsection{GAIN job training dataset}
\label{app:gain}

The Greater Avenues for Independence (GAIN) study~\citep{Hotz.2006} applied randomized job-training to individuals across four California sites (Riverside, Alameda, Los Angeles, and San Diego) and then recorded $36$ quarters of earnings and AFDC receipt for each individual after the training program. We use the pre-treatment covariates $X$ comprising demographics (age, sex, race/ethnicity, education, marital and child status), pre-RA labor and welfare history ($10$ quarters of earnings \texttt{tcprn1}--\texttt{tcprn10}, paid-employment indicators \texttt{paid1}--\texttt{paid4}, AFDC receipt \texttt{adcpc1}--\texttt{adcpc4} and \texttt{padcpc1}--\texttt{padcpc4}), and local macroeconomic conditions (\texttt{grew1}, \texttt{gepop1}). The treatment $A$ is the GAIN program assignment indicator. The surrogate $S$ stacks the first six post-RA quarters of earnings (\texttt{tcedd1}--\texttt{tcedd6}), AFDC receipt (\texttt{aid1}--\texttt{aid6}), and the derived employment indicators $\mathbf{1}\{\texttt{tcedd}_k>0\}$ for $k=1,\dots,6$. The long-term outcome $Y$ is the average quarterly earnings over quarters $13$--$36$ (years $4$--$9$).

Following \citet{Athey.2025}, we select the dataset from Riverside site (5445 units) as $\mathcal{D}_1$, retaining $(X,A,S)$ and dropping $Y$. Then, we pool Alameda, Los Angeles and San Diego (13725 units) as $\mathcal{D}_2$, retaining $(X,S,Y)$ and discarding $A$. The cross-site covariate gap induces nontrivial $\rho(X)$ heterogeneity without any synthetic labelling. In Riverside, we apply covariate-dependent rejection sampling, keeping a treated unit with probability $m(X)$ and a control unit with probability $1-m(X)$, where
\begin{equation}
    m(X) = \mathrm{trim}_{0.01}\!\Big(\sigma\!\big(\gamma_\pi\cdot (0.7\, X_{\text{earn}} + 0.3\, X_{\text{age}})\big)\Big),
\end{equation}
with $X_{\text{earn}}$ the standardized mean of \texttt{tcprn1}--\texttt{tcprn10} and $X_{\text{age}}$ the standardized age. The parameter $\gamma_\pi = 2$ creates a regime with low treatment overlap. The induced post-filter treatment propensity is
\begin{equation}
    \pi(X) = \frac{m(X)\,\pi_0}{m(X)\,\pi_0 + (1-m(X))(1-\pi_0)},\qquad \pi_0\approx 0.81.
\end{equation}

GAIN admits no ground-truth HLTE, so we approximate it via a pseudo-oracle $\hat\tau^\star$ that exploits the unconditional randomization in Riverside: on the full Riverside sample, $A\indep (Y(0),Y(1))\mid X$, hence $\tau$ is identified from $(X,A,Y)$ alone without invoking the surrogate. We fit the RA-learner and DR-learner (usually we have $\text{RA-learner}>\text{T-learner}$ and $\text{DR-learner}>\text{IPW-learner}$) with $5$-fold cross-fitting on this full sample, using gradient-boosted trees for $\mu(a,X)$ and logistic regression for the propensity, and store the pseudo-ground-truth $\{\hat\tau^\star(X_i)\}$ for all Riverside units as the averaged predictions from all the learners. The LT-O learners never observe $Y$ for any Riverside unit during training, so no leakage is introduced. The stored $\hat\tau^\star$ then serves as the benchmark on the held-out test fold. We note that $\hat\tau^\star$ is only an approximated proxy for the true HLTE, hence the reported PEHE should be interpreted with caution.

\textbf{Data access:} The GAIN micro-data are available under restricted access from ICPSR (study 38125, \url{https://www.icpsr.umich.edu/web/ICPSR/studies/38125}). The dataset cannot be redistributed alongside our code; reproducing the GAIN experiments requires obtaining access independently from ICPSR and following the preprocessing described above.

\newpage

\section{Baselines}
\label{app:baselines}

\subsection{Analogy between HTE learners and HLTE learners}
\label{app:analogy}
In this section, we briefly explain the analogy between standard HTE learners and their counterparts for HLTE estimation. We show how do we adapt the classic confounding adjustment strategies for HTE estimation to the long-term outcome setting.

\paragraph{HTE setting:} In the standard HTE estimation setting, we have a single dataset $\mathcal{D}=\{(X_i, A_i, Y_i)\}_{i=1}^N$ where $X$ are pre-treatment covariates, $A\in\{0,1\}$ is the treatment, and $Y$ is the outcome. The goal is to estimate the conditional average treatment effect (CATE)
\begin{align*}
\tau(x) = \E[Y(1) - Y(0) \mid X=x].
\end{align*}
Under \textit{unconfoundedness} ($A \indep \{Y(0), Y(1)\} \mid X$) and positivity, the CATE is identified from observed data via two equivalent expressions:
\begin{align*}
\tau(x) = \mu_1(x) - \mu_0(x) = \E\!\left[Y \cdot \frac{A - \pi(X)}{\pi(X)(1-\pi(X))} \,\Big|\, X = x\right],
\end{align*}
where $\mu_a(x) := \E[Y\mid A=a,X=x]$ and $\pi(x) := \Pr(A=1\mid X=x)$. The four standard meta-learners below correspond to different ways of plugging estimated nuisances $(\hat\mu_a, \hat\pi)$ into these identities.

\textbf{Plug-in estimator (T-learner~\citep{Kunzel.2019}):}
The T-learner directly plugs the estimated response surfaces $\hat\mu_a(x)$ into the first identification formula:
\begin{align*}
\hat\tau_T(x) = \hat\mu_1(x) - \hat\mu_0(x).
\end{align*}
The estimator is simple but inherits the full estimation error of $\hat\mu_a$ (\textit{plug-in bias}~\citep{Kennedy.2022-Sempiparametric}).

\textbf{RA-learner~\citep{Curth.2021}:}
The RA-learner constructs a regression-adjusted pseudo-outcome
\begin{align*}
\mathcal{T}_\mathrm{RA} = A\,\bigl(Y - \hat\mu_0(X)\bigr) + (1-A)\,\bigl(\hat\mu_1(X) - Y\bigr),
\end{align*}
which has conditional mean $\tau(x)$ under unconfoundedness, and regresses $\mathcal{T}_\mathrm{RA}$ on $X$ in the second stage.

\textbf{IPW-learner~\citep{Curth.2021}:}
The IPW-learner instantiates the second identification formula via the inverse-propensity-weighted pseudo-outcome
\begin{align*}
\mathcal{T}_\mathrm{IPW} = \left(\frac{A}{\hat\pi(X)} - \frac{1-A}{1-\hat\pi(X)}\right) Y,
\end{align*}
and regresses $\mathcal{T}_\mathrm{IPW}$ on $X$. The estimator is unstable when $\hat\pi$ is close to $0$ or $1$ (low treatment overlap).

\textbf{DR-learner~\citep{Kennedy.2023}:}
The DR-learner is based on the \textit{efficient influence function} (EIF) of the ATE, which yields the AIPW pseudo-outcome
\begin{align*}
\mathcal{T}_\mathrm{AIPW} = \hat\mu_1(X) - \hat\mu_0(X) + \frac{A - \hat\pi(X)}{\hat\pi(X)(1-\hat\pi(X))}\,\bigl(Y - \hat\mu_A(X)\bigr).
\end{align*}
Because the Gateaux derivative of the EIF w.r.t.\ the nuisances vanishes at the truth~\citep{Chernozhukov.2018}, the resulting second-stage loss is Neyman-orthogonal in $(\hat\mu,\hat\pi)$: first-order errors in either nuisance contribute only at second order to the estimator.

\paragraph{How we adapt to the HLTE setting:}
In the HLTE setting, the long-term outcome $Y$ is unobserved in $\mathcal{D}_1$ (where the treatment $A$ is observed). The bridge is the \textit{surrogate index} $h(S,X) = \E[Y\mid S,X,R=1]$, which under surrogacy and comparability (Assumptions~\ref{assumption-surrogacy}--\ref{assumption-compara}) is an unbiased proxy for $Y$ on $\mathcal{D}_1$~\citep{Athey.2025, Chen.2023}. Concretely:
\begin{itemize}
\item For the \textbf{non-orthogonal learners} (\LTT, \LTRA, \LTIPW), we fit $\hat h$ on $\mathcal{D}_2$, substitute every occurrence of $Y$ with $\hat h(S,X)$ in the pseudo-outcomes above, and run the second-stage regression on $\mathcal{D}_1$. The remaining nuisances $\hat\mu_a$ and $\hat\pi$ are fit on $\mathcal{D}_1$.
\item For the \textbf{orthogonal \DR}, we compute the EIF for the long-term ATE of \citet{Athey.2025, Chen.2023} to get the pseudo-outcome $\mathcal{T}_\mathrm{DR}$. Note:We recover \DR as the special case $\omega\equiv 1$ of our \methods.
\end{itemize}
The non-orthogonal learners thus inherit \textit{both} the plug-in/IPW limitations of the standard HTE estimators \textit{and} the additional first-order error from $\hat h$. The \DR removes the latter via orthogonality but remains sensitive to low treatment or outcome overlap, motivating our weighted variants in Section~\ref{subsec:instantiation}.

\subsection{Weighted learners}
We briefly explain the baseline weighted learners used in the experiments.

\textbf{Weighted RA-learner:} For the weighted version of the learner, we simply add weights to all the samples during the second-stage pseudo-outcome regression. Specifically, the \textit{weighted} RA-learner minimizes the following loss function:
    \begin{equation}
        \mathcal{L}_\mathrm{\omega,RA}(g;\eta)=\E_n\left[\omega(X)(g(X)-\mathcal{T}_\mathrm{RA})^2\mid R=0\right], \; g \in \mathcal{G}
    \end{equation}
\textbf{Weighted DR-learner:} The weighted DR-learner minimizes the following loss function in the second stage:
\begin{equation}
        \mathcal{L}_\mathrm{\omega,DR}(g;\eta)=\mathbb{E} \left[ \omega(X)\bigl(\1(R=0)g(X)^2 -2\mathcal{T}_{DR}\cdot g(X)\bigr)\right], \; g \in \mathcal{G}
\end{equation}


\newpage

\section{Implementation}
\label{app:implementation}
All experiments are run on Intel Core Ultra 7 155U (1.70 GHz) CPU and 32GB RAM. Data processing and model training are all implemented with Python. We provide the source code at \href{https://anonymous.4open.science/r/Orthogonal-learners-for-Long-term-effects-64DD}{https://anonymous.4open.science/r/Orthogonal-learners-for-Long-term-effects-64DD}.

\textbf{Model architectures: } We instantiate a separate multi-layer perceptron (MLP) for each sub-regression task. Neural networks that estimate propensity nuisance functions $(\rho,\rho_s,\pi,\pi_s)$ have two middle layers and a sigmoid output transformation at the end. Networks used to estimate outcome-related nuisances, as well as the second-stage predictor optimized under the meta-learner objective, consist of four hidden layers. All hidden layers use ReLU activations. Each learner instantiates the subset of these networks required by its algorithm. This ensures that all baselines use the same architecture. As a result, any differences in performance can be attributed to the meta-algorithm itself rather than model capacity.

\textbf{Training: } Hidden-layer widths are set to either $(20,10)$ (two-layer MLPs) or $(20,20,10,10)$ (four-layer MLPs). Nuisance networks are trained for 20 epochs with batch size 64 and learning rate $10^{-3}$. The second-stage network is trained for 40 epochs with learning rate $10^{-3}$.

\textbf{Cross-fitting:} We partition the unified dataset $\mathcal{D}=\{Z_i\}_{i=1}^N$ into $K=5$ folds. For each held-out fold $k$, we fit the nuisances on the remaining $K-1$ folds in a sequential order dictated by the definitions in Eq.~(\ref{eq:def-eta}). Specifically, $(\hat\pi,\hat\pi_s,\hat\rho,\hat\rho_s)$ are fit by direct regression, and $\hat h$ is fit on the $R=1$ subsample as a regression of $Y$ on $(S,X)$. Since $\mu(a,x)=\E[h(S,X)\mid A=a,X=x,R=0]$, we then form the pseudo-targets $\hat h(S_i,X_i)$ on the $R=0$ subsample and fit $\hat\mu$ by regressing them on $(A,X)$. We finally evaluate $\hat\omega^*_i$ and $\hat{\mathcal{T}}_i$ on the held-out fold $k$ using the cross-fitted nuisances and aggregate across folds to solve Eq.~(\ref{eq:empirical-loss}).

\textbf{Stabilization of $\omega^*$:} The pseudo-weight $\omega^*(Z;\eta)$ in Eq.~(\ref{omega-star-def}) can take values arbitrarily close to $0$ (e.g.\ for observational samples under $\omega\equiv 1$ or $\omega=\pi^2(1-\pi)^2$), which leaves the empirical loss without a quadratic curvature term and destabilizes second-stage training. To avoid this, we trim $\hat\omega^*_i$ such that $|\hat\omega^*_i|\geq 10^{-7}$, preserving its sign.

%
%
%
%
%
%
%
%
%
%
%
%
%
%
%
%
%
\newpage

\section{Optimality of the Overlap Weighting Scheme}
\label{sec:optimality}

In this section, we provide a semiparametric efficiency-based justification for the overlap weighting scheme $\omega(x) = \pi(x)^2(1-\pi(x))^2\,\rho(x)$ used in the LT-O-DO-learner. Our argument proceeds in three steps:
(i) we derive the efficiency bound for the $\omega$-weighted average long-term treatment effect (WALTE), which we denote $V^{\text{eff}}_\omega$;
(ii) we characterize the weight $\omega_\mathrm{opt}(x)$ that minimizes $V^{\text{eff}}_\omega$ for the
\emph{sample} estimand;
(iii) we show that the LT-O-DO weight $\omega(x) = \pi^2(1-\pi)^2\rho$ arises naturally as a principled, orthogonality-preserving relaxation of  $\omega_\mathrm{opt}$ via a structural lower bound
on the variance functional $V(x)$.

\subsection{Efficiency bound for the weighted long-term ATE}
\label{sec:efficiency_bound}

Recall the efficient influence function (EIF) for the WALTE, under the simplification that $\omega(\cdot)$ does not depend on the nuisances $\rho$ or $\pi$:
\begin{align}
\phi(Z; \eta) = \frac{1}{p_0\, D_\omega} \Bigg[
&\1(R=0)\,\omega(X)\,\big(\hat{\tau}_{\text{AIPW}}(Z;\eta) - \tau^0_\omega\big)
+ \1(R=1)\,\omega(X)\,\psi_{\text{obs}}(Z;\eta) \Bigg],
\label{eq:eif_simplified}
\end{align}
where $p_0 = P(R=0)$, $D_\omega = \mathbb{E}[\omega(X) \mid R=0]$, and the components are
\begin{align}
\hat{\tau}_{\text{AIPW}}(Z;\eta)
&= \mu(1,X) - \mu(0,X)
+ \frac{A}{\pi(X)}\big(h(S,X) - \mu(1,X)\big)
- \frac{1-A}{1-\pi(X)}\big(h(S,X) - \mu(0,X)\big), \\
\psi_{\text{obs}}(Z;\eta)
&= \frac{1-\rho_s(S,X)}{\rho_s(S,X)}
\left(\frac{\pi_s(S,X) - \pi(X)}{\pi(X)(1-\pi(X))}\right)
\big(Y - h(S,X)\big).
\end{align}

\begin{proposition}[Efficiency bound for the WALTE]
\label{prop:eif_variance}
Let $\sigma^2_{h|a}(x) := \mathrm{Var}(h(S,X) \mid A=a, X=x, R=0)$ and
$\sigma^2_Y(s,x) := \mathrm{Var}(Y \mid S=s, X=x, R=1)$. Then, the semiparametric
efficiency bound for the $\omega$-weighted long-term ATE is
\begin{align}
V^{\text{eff}}_\omega
&= \frac{1}{p_0 D_\omega^2}\,\mathbb{E}\!\left[\omega(X)^2\!\left(
\frac{\sigma^2_{h|1}(X)}{\pi(X)} + \frac{\sigma^2_{h|0}(X)}{1-\pi(X)}
+ \big(\tau^0(X) - \tau^0_\omega\big)^2
\right)\,\Big|\,R=0\right] \notag\\
&\quad + \frac{1-p_0}{p_0^2 D_\omega^2}\,\mathbb{E}\!\left[\omega(X)^2
\left(\frac{1-\rho_s(S,X)}{\rho_s(S,X)}\right)^{\!2}
\!\left(\frac{\pi_s(S,X)-\pi(X)}{\pi(X)(1-\pi(X))}\right)^{\!2}
\sigma^2_Y(S,X)\,\Big|\,R=1\right].
\label{eq:efficiency_bound}
\end{align}
\end{proposition}

\begin{proof}
The efficiency bound is $V^{\text{eff}}_\omega = \mathbb{E}[\phi(Z;\eta)^2]$. By squaring Eq.~(\ref{eq:eif_simplified}), the cross-term vanishes since
$\1(R=0)\cdot\1(R=1)=0$, thus leaving two terms.

\paragraph{Experimental term ($R=0$).}
Decompose $\hat{\tau}_{\text{AIPW}} - \tau^0_\omega
= (\hat{\tau}_{\text{AIPW}} - \tau^0(X)) + (\tau^0(X) - \tau^0_\omega)$.
At true nuisances, $\mathbb{E}[\hat{\tau}_{\text{AIPW}} \mid X, R=0] = \tau^0(X)$, so the cross-term vanishes upon conditioning. The conditional variance is computed using
$A \mid X, R=0 \sim \mathrm{Ber}(\pi(X))$, yielding
\begin{equation}
\mathrm{Var}(\hat{\tau}_{\text{AIPW}} \mid X, R=0)
= \frac{\sigma^2_{h|1}(X)}{\pi(X)} + \frac{\sigma^2_{h|0}(X)}{1-\pi(X)}.
\end{equation}

\paragraph{Observational term ($R=1$).}
Since $h(S,X) = \mathbb{E}[Y \mid S, X, R=1]$, we have
$\mathbb{E}[Y - h(S,X) \mid S, X, R=1] = 0$, so conditioning on $(S, X, R=1)$:
\begin{equation}
\mathbb{E}[\psi_{\text{obs}}^2 \mid S,X,R=1]
= \left(\frac{1-\rho_s}{\rho_s}\right)^{\!2}
\left(\frac{\pi_s - \pi}{\pi(1-\pi)}\right)^{\!2} \sigma^2_Y(S,X).
\end{equation}
Combining and accounting for the prefactors $1/(p_0 D_\omega)^2$ and the marginal
probabilities $P(R=0)=p_0$, $P(R=1)=1-p_0$ yields Eq.~(\ref{eq:efficiency_bound}).
\end{proof}

\subsection{Optimal weight for the sample WALTE}
\label{sec:optimal_weight}

Following Crump et al.\ (2006), we target the \emph{sample} WALTE
$\tau_{\omega,S} = \sum_i \tau^0(X_i)\omega(X_i) / \sum_i \omega(X_i)$ rather than its
population counterpart. This eliminates the squared-bias term
$(\tau^0(X) - \tau^0_\omega)^2$ from Eq.~(\ref{eq:efficiency_bound}), since the sample
estimand absorbs $\tau^0_\omega$ as a sample average rather than a population functional.

To unify the two terms, we use the density ratio identity
$(1-p_0)\,f_{X|R=1}(x) = p_0\,f_{X|R=0}(x)\,\rho(x)/(1-\rho(x))$ to rewrite the
observational expectation under the experimental measure. For this, we define
\begin{equation}
\Sigma_o(x) := \mathbb{E}\!\left[
\left(\frac{1-\rho_s}{\rho_s}\right)^{\!2}
\!\left(\frac{\pi_s - \pi}{\pi(1-\pi)}\right)^{\!2}
\sigma^2_Y(S,X) \,\Big|\, X, R=1\right].
\end{equation}
The objective becomes
\begin{equation}
J[\omega] = p_0 \cdot \mathbb{E}\!\left[\omega(X)^2 V(X) \mid R=0\right],
\qquad
V(x) := \underbrace{\frac{\sigma^2_{h|1}(x)}{\pi(x)} + \frac{\sigma^2_{h|0}(x)}{1-\pi(x)}}_{\Sigma_t(x)}
+ \underbrace{\frac{\rho(x)}{1-\rho(x)}\,\Sigma_o(x)}_{\text{outcome overlap cost}},
\label{eq:Vx}
\end{equation}
subject to the normalization $D_\omega = \mathbb{E}[\omega(X) \mid R=0] = 1$.

\begin{proposition}[Efficiency-optimal weight]
\label{prop:optimal_weight}
The weight $\omega_\mathrm{opt}(x)$ minimizing $J[\omega]$ subject to $D_\omega = 1$ is
\begin{equation}
\omega_\mathrm{opt}(x) = \frac{V(x)^{-1}}{\mathbb{E}\!\left[V(X)^{-1} \mid R=0\right]}
\;\propto\; V(x)^{-1},
\label{eq:omega_star}
\end{equation}
and the corresponding minimum efficiency bound is
$V^{\text{eff}}_{\omega_\mathrm{opt}} = \big(p_0\,\mathbb{E}[V(X)^{-1} \mid R=0]\big)^{-1}$.
\end{proposition}

\begin{proof}
Form the Lagrangian
$\mathcal{L}[\omega, \lambda] = p_0\,\mathbb{E}[\omega(X)^2 V(X) \mid R=0]
- \lambda\big(\mathbb{E}[\omega(X) \mid R=0] - 1\big)$.
Taking the functional derivative under perturbation $\omega \to \omega + \varepsilon\eta$,
\begin{equation}
\frac{\diff \mathcal{L}}{\diff \varepsilon}\bigg|_{\varepsilon=0}
= 2p_0\,\mathbb{E}[\eta(X)\omega(X)V(X) \mid R=0]
- \lambda\,\mathbb{E}[\eta(X) \mid R=0] = 0 \quad \forall\,\eta.
\end{equation}
By the fundamental lemma of the calculus of variations,
$2p_0\,\omega(x)V(x) = \lambda$ a.s., which yields $\omega_\mathrm{opt}(x) \propto V(x)^{-1}$.
The constant is fixed by $\mathbb{E}[\omega_\mathrm{opt}(X) \mid R=0] = 1$.
\end{proof}

\subsection{Lower bound on \texorpdfstring{$V(x)$}{V(x)} and the LT-O-DO weight}
\label{sec:lower_bound}

The optimal weight $\omega_\mathrm{opt}$ depends on \emph{all} nuisance functions
$(\sigma^2_{h|a}, \pi, \rho, \rho_s, \pi_s, \sigma^2_Y)$ through $V(x)$, and therefore
breaks Neyman-orthogonality of the LT-O-learner framework. We now derive a structural
lower bound on $V(x)$ that depends on $\pi$ and $\rho$ in a closed form, motivating the
LT-O-DO weight $\omega(x) = \pi^2(1-\pi)^2\rho$ as a principled approximation to $\omega_\mathrm{opt}$.

\paragraph{Treatment overlap term.}
Since $\sigma^2_{h|1}(1-\pi) + \sigma^2_{h|0}\pi \geq \min\{\sigma^2_{h|1}, \sigma^2_{h|0}\}$,
\begin{equation}
\Sigma_t(x) = \frac{\sigma^2_{h|1}(1-\pi) + \sigma^2_{h|0}\pi}{\pi(1-\pi)}
\;\geq\; \frac{C_t(x)^2}{\pi(x)(1-\pi(x))},
\qquad
C_t(x)^2 := \min\!\big\{\sigma^2_{h|1}(x), \sigma^2_{h|0}(x)\big\}.
\label{eq:Sigma_t_bound}
\end{equation}

\paragraph{Outcome overlap term.}
We separate the $\pi$-dependent factor from the surrogate-dependent factor by writing
\begin{equation}
\Sigma_o(x)
= \frac{1}{\pi(x)^2(1-\pi(x))^2}\,\mathbb{E}\!\left[
\left(\tfrac{1-\rho_s}{\rho_s}\right)^{\!2}\!\tilde m(S,X)\,\Big|\,X,R=1\right],
\qquad
\tilde m(s,x) := (\pi_s(s,x) - \pi(x))^2\,\sigma^2_Y(s,x) \geq 0.
\end{equation}
By the (first) mean value theorem for integrals applied to the nonnegative weighting
function $\big((1-\rho_s)/\rho_s\big)^2$, there exists $\bar{m}(x)$ in the range of
$\tilde m(\cdot, x)$ such that
\begin{equation}
\mathbb{E}\!\left[\left(\tfrac{1-\rho_s}{\rho_s}\right)^{\!2}\tilde m(S,X)\,\Big|\,X,R=1\right]
= \bar{m}(x)\cdot\mathbb{E}\!\left[\left(\tfrac{1-\rho_s}{\rho_s}\right)^{\!2}\,\Big|\,X,R=1\right].
\label{eq:mvt}
\end{equation}
Applying Jensen's inequality to $t \mapsto t^2$,
\begin{equation}
\mathbb{E}\!\left[\left(\tfrac{1-\rho_s}{\rho_s}\right)^{\!2}\,\Big|\,X,R=1\right]
\geq \left(\mathbb{E}\!\left[\tfrac{1-\rho_s}{\rho_s}\,\Big|\,X,R=1\right]\right)^{\!2}.
\label{eq:jensen}
\end{equation}
The first moment admits an exact identity. Using
$f(s|x, R=1) = \rho_s(s,x)\,f(s|x)\,/\,\rho(x)$ and the tower law
$\mathbb{E}[\rho_s(S,X) \mid X] = \mathbb{E}[R \mid X] = \rho(x)$, we yield
\begin{equation}
\mathbb{E}\!\left[\tfrac{1-\rho_s}{\rho_s}\,\Big|\,X,R=1\right]
= \frac{1}{\rho(x)} \int (1 - \rho_s(s,x))\,f(s|x)\,\diff s
= \frac{1 - \rho(x)}{\rho(x)}.
\label{eq:tower_identity}
\end{equation}
Combining Eq.~(\ref{eq:mvt})--Eq.~(\ref{eq:tower_identity}), we arrive at
\begin{equation}
\frac{\rho(x)}{1-\rho(x)}\,\Sigma_o(x)
\geq \frac{\rho}{1-\rho}\cdot \frac{1}{\pi^2(1-\pi)^2}\cdot\frac{(1-\rho)^2}{\rho^2}\,\bar{m}(x)
= \frac{(1-\rho(x))\,\bar{m}(x)}{\rho(x)\,\pi(x)^2(1-\pi(x))^2}
= \frac{C(x)^2}{\rho(x)\,\pi(x)^2(1-\pi(x))^2},
\label{eq:Sigma_o_bound}
\end{equation}
where we define the surrogate informativeness factor
$C(x)^2 := (1-\rho(x))\,\bar{m}(x)$.

\paragraph{Combined lower bound.}
Adding Eq.~(\ref{eq:Sigma_t_bound}) and Eq.~(\ref{eq:Sigma_o_bound}),
\begin{equation}
\boxed{\;
V(x) \;\geq\; \frac{C_t(x)^2}{\pi(x)(1-\pi(x))} + \frac{C(x)^2}{\rho(x)\,\pi(x)^2(1-\pi(x))^2}.
\;}
\label{eq:V_lower_bound}
\end{equation}

\paragraph{Justification of $\omega(x) = \pi^2(1-\pi)^2\rho$.}
Since $\omega_\mathrm{opt} \propto V(x)^{-1}$, the bound in Eq.~(\ref{eq:V_lower_bound}) gives
\begin{equation}
\omega_\mathrm{opt}(x) \;\leq\; \left(\frac{C_t(x)^2}{\pi(1-\pi)} + \frac{C(x)^2}{\rho\,\pi^2(1-\pi)^2}\right)^{\!-1}.
\label{eq:omega_upper}
\end{equation}
Under the joint-overlap-failure regime---where both treatment overlap ($\pi \to 0$ or
$1$) and outcome overlap ($\rho \to 0$) are limited---the second term in
Eq.~(\ref{eq:omega_upper}) dominates and is therefore the binding constraint. Retaining only
the binding term and setting $C(x) \equiv 1$ yields
\begin{equation}
\omega_\mathrm{opt}(x) \;\lessconst\; \rho(x)\,\pi(x)^2(1-\pi(x))^2,
\end{equation}
which we adopt as the LT-O-DO weight
\begin{equation}
\omega(x) = \pi(x)^2(1-\pi(x))^2\,\rho(x).
\end{equation}
This weight depends only on $(\pi, \rho)$ and therefore preserves Neyman-orthogonality of
the LT-O-learner framework. The squared treatment-overlap factor $\pi^2(1-\pi)^2$ arises
because poor treatment overlap amplifies the outcome overlap cost via the
$(\pi_s - \pi)^2/[\pi(1-\pi)]^2$ factor in $\Sigma_o$.

This weight matches the qualitative structure of $\omega_\mathrm{opt}$: it vanishes whenever
\emph{either} treatment overlap or outcome overlap fails, downweighting the same strata
that $\omega_\mathrm{opt}$ would.

\newpage
\section{Linear parametric instantiation: closed form and asymptotic variance}
\label{app:linear}

We instantiate the LT-O-learner with the linear parametric class
$\mathcal G=\{g(x)=\theta^\top x:\theta\in\mathbb R^d\}$.
This case admits a closed-form solution and yields an explicit
asymptotic-variance comparison with the LT-O-DR-learner.
Throughout, $\widehat\omega^{\,*}_i:=\omega^*(Z_i;\hat\eta)$ and
$\widehat{\mathcal{T}}_i:=\mathcal{T}_{\mathrm{LT}}(Z_i;\hat\eta)$ are the plug-in versions of
Eq.~(\ref{omega-star-def}) and Eq.~(\ref{Tau-LT-def}).

\subsection{Closed-form estimator}

\begin{proposition}[Closed-form estimator]\label{prop:closed-form}
Let $\mathbf X\in\mathbb R^{n\times d}$ stack $X_i^\top$ as rows,
$\widehat W=\mathrm{diag}(\widehat\omega^{\,*}_1,\dots,\widehat\omega^{\,*}_n)$, and
$\widehat{\mathcal{T}}=(\widehat{\mathcal{T}}_1,\dots,\widehat{\mathcal{T}}_n)^\top$.
If $\mathbf X^\top \widehat W\,\mathbf X\succ 0$, the empirical loss in Eq.~(\ref{eq:empirical-loss})
is strictly convex, and its unique minimizer is
\begin{equation}
\widehat\theta
=\bigl(\mathbf X^\top \widehat W\,\mathbf X\bigr)^{-1}\mathbf X^\top \widehat{\mathcal{T}}
=\Bigl(\sum_{i=1}^n \widehat\omega^{\,*}_i X_iX_i^\top\Bigr)^{-1}\Bigl(\sum_{i=1}^n \widehat{\mathcal{T}}_i X_i\Bigr).
\label{eq:closed-form}
\end{equation}
\end{proposition}

\begin{proof}
Substituting $g(X_i)=\theta^\top X_i$ into Eq.~(\ref{eq:empirical-loss}),
we yield $\widehat L_\omega(\theta)=\theta^\top \widehat H\theta - 2\theta^\top \widehat b$
with $\widehat H=n^{-1}\sum_i\widehat\omega^{\,*}_i X_iX_i^\top$ and
$\widehat b=n^{-1}\sum_i\widehat{\mathcal{T}}_iX_i$. The Hessian $\nabla^2_\theta\widehat L_\omega=2\widehat H$
is positive definite by assumption, so the first-order condition $\widehat H\theta=\widehat b$ has a
unique solution that is the global minimizer.
\end{proof}

\begin{remark}[Sample-level positivity]
For LT-O-LO ($\omega^*=(R-\rho)^2$) and LT-O-DR
($\omega^*=\mathbf 1(R=0)$), $\widehat\omega^{\,*}_i\ge 0$ a.s., so
$\mathbf X^\top\widehat W\mathbf X\succeq 0$, and strict positivity follows from the
usual full-rank condition on $\{X_i\}_{i:\widehat\omega^{\,*}_i>0}$. For LT-O-TO and LT-O-DO, however, the $R=0$ component carries the residual term
$\Omega$ and is not sample-wise sign-definite: LT-O-TO has $\1(R=0)\bigl[\pi^2(1-\pi)^2+2\pi(1-\pi)(1-2\pi)(A-\pi)\bigr]$, and LT-O-DO has
$\1(R=0)\,\rho\bigl[\rho\,\pi^2(1-\pi)^2+2\pi(1-\pi)(1-2\pi)(A-\pi)\bigr]$. The population analog $\mathbb E[\omega^*\mid X]=(1-\rho)\omega\ge 0$ together
with Assumptions~\ref{assumption-positivity},~\ref{core-assumption} ensure $\mathbf X^\top\widehat W\mathbf X\succ 0$ with
probability tending to one. In finite samples, a small ridge $\lambda I$ or
clipping of the residual term $\Omega$ guarantees PD-ness without affecting
the asymptotic statements below.
\end{remark}

\begin{remark}[WLS interpretation]\label{rem:wls}
When $\widehat\omega^{\,*}_i>0$ for every $i$ that contributes a nonzero $\widehat{\mathcal{T}}_i$,
the loss can be completed into a square:
\begin{align}
\widehat L_\omega(\theta)
= \frac{1}{n}\sum_{i=1}^n\widehat\omega^{\,*}_i\bigl(\widetilde Y_i-X_i^\top\theta\bigr)^2
- \frac{1}{n}\sum_{i=1}^n \frac{\widehat{\mathcal{T}}_i^{\,2}}{\widehat\omega^{\,*}_i},
\qquad \widetilde Y_i:=\widehat{\mathcal{T}}_i/\widehat\omega^{\,*}_i,
\end{align}
and $\widehat\theta$ is the WLS fit of $\widetilde Y_i$ on $X_i$ with weights
$\widehat\omega^{\,*}_i$. This applies to LT-O-DO when $\widehat\omega^{\,*}_i>0$ for all $i$.
For LT-O-DR and LT-O-TO, however, $R=1$ observations have $\widehat\omega^{\,*}_i=0$ but
$\widehat{\mathcal{T}}_i\neq 0$, so $\widetilde Y_i$ is undefined and the WLS reduction is invalid.
The correct unified statement is that $\widehat\theta$ solves the linear system
$\widehat H\theta=\widehat b$, which reduces to WLS only in the strictly
positive-weight regime.
\end{remark}

\subsection{Asymptotic variance}

We invoke the standard Z-estimator framework~\citep{Vaart.1998} with
cross-fitting~\citep{Chernozhukov.2018}. Define the score
\begin{equation}
m(Z;\theta,\eta)=\omega^*(Z;\eta)\,XX^\top\theta - \mathcal{T}_{\mathrm{LT}}(Z;\eta)\,X.
\label{eq:score}
\end{equation}

\begin{assumption}[Z-estimation regularity]\label{ass:zest}
Suppose $\tau_0(x)=\theta_0^\top x$ with $\theta_0\in\mathbb R^d$, $d$ fixed, and
that the nuisance estimator $\hat\eta$ is constructed via $K$-fold cross-fitting.
In addition to Assumptions~\ref{nondegen-overlap}--\ref{core-assumption}, assume:
\begin{enumerate}\itemsep0pt
\item[(i)] \emph{Invertibility:} $J(\omega):=\mathbb E[\omega^*(Z;\eta_0)\,XX^\top]\succ 0$.
\item[(ii)] \emph{Bounded second moments:} $\mathbb E[\|X\|^4]<\infty$ and
$\mathbb E[\xi(Z;\eta_0)^2\|X\|^2]<\infty$, where
$\xi(Z;\eta_0):=\mathcal{T}_{\mathrm{LT}}(Z;\eta_0)-\omega^*(Z;\eta_0)\,\tau_0(X)$.
\item[(iii)] \emph{Stochastic equicontinuity:}
$\bigl\|m(\cdot;\theta_0,\hat\eta)-m(\cdot;\theta_0,\eta_0)\bigr\|_{P,2}=o_p(1)$.
\item[(iv)] \emph{Orthogonality remainder;}
$\sqrt n\,\bigl\|\mathbb E\!\bigl[m(Z;\theta_0,\hat\eta)-m(Z;\theta_0,\eta_0)\bigm|\hat\eta\bigr]\bigr\|=o_p(1)$.
\item[(v)] \emph{Convergence of $\widehat J$:}
Let $\widehat J:= \mathbb P_n[\omega^*(Z;\hat\eta)XX^\top]$. Then $\|\widehat J - J(\omega)\|_{\mathrm{op}}=o_p(1)$.

\end{enumerate}
\end{assumption}

Conditions (iii)--(v) are the standard Z-estimator conditions; (iv) is implied by Theorem~\ref{thm:orthogonality}
(Neyman-orthogonality) when nuisance products are $o_p(n^{-1/2})$. We state them explicitly because
the quasi-oracle rate alone does not control the empirical process and bias
remainders of the moment equation.

\begin{lemma}[Asymptotic distribution]\label{lem:asymp-var}
Under Assumption~\ref{ass:zest},
\begin{align}
\sqrt n\bigl(\widehat\theta-\theta_0\bigr)\;\xrightarrow{d}\;
\mathcal N\!\bigl(0,\;\mathcal{V}(\omega)\bigr),
\qquad \mathcal{V}(\omega)=J(\omega)^{-1}\Sigma(\omega)J(\omega)^{-1},
\end{align}
with
\begin{align}
J(\omega)=\mathbb E[(1-\rho(X))\,\omega(X)\,XX^\top],
\qquad
\Sigma(\omega)=\mathbb E[V^\star_\omega(X)\,XX^\top],
\end{align}
and
\begin{align}
V^\star_\omega(X)=\mathrm{Var}\!\bigl(\mathcal{T}_{\mathrm{LT}}(Z;\eta_0)-\omega^*(Z;\eta_0)\tau_0(X)\bigm|X\bigr).
\end{align}
\end{lemma}

\begin{proof}
At the truth, $m(Z;\theta_0,\eta_0)=-X\,\xi(Z;\eta_0)$. The two identities
$\mathbb E[\omega^*\mid X]=(1-\rho)\omega$ and
$\mathbb E[\mathcal{T}_{\mathrm{LT}}\mid X]=(1-\rho)\omega\,\tau_0$ give
$\mathbb E[\xi\mid X]=0$, hence $\mathbb E[m(Z;\theta_0,\eta_0)]=0$ and
$\Sigma(\omega)=\mathbb E[\xi^2\,XX^\top]=\mathbb E[V^\star_\omega(X)\,XX^\top]$.
Linearity of $m$ in $\theta$ gives $\nabla_\theta m=\omega^*XX^\top$, and by
the tower property,
$\mathbb E[\nabla_\theta m]=J(\omega)$.
The estimating equation $n^{-1}\sum_i m(Z_i;\widehat\theta,\hat\eta)=0$,
combined with linearity and Assumption~\ref{ass:zest}(iii)--(iv), yields the
asymptotic linearization
\begin{align}
\sqrt n(\widehat\theta-\theta_0)
=J(\omega)^{-1}\,\frac{1}{\sqrt n}\sum_{i=1}^n X_i\,\xi(Z_i;\eta_0)+o_p(1),
\end{align}
and the central-limit theorem under Assumption~\ref{ass:zest}(ii) gives the stated normal limit.
\end{proof}

\subsection{Variance comparison in low-overlap regimes}
\label{app:linear-asymvar}
In this section, we prove Theorem~\ref{thm:linear-asymvar-main}.

Direct algebra yields explicit forms for $V^\star_\omega(X)$ in two
canonical cases.

\begin{lemma}[Conditional residual variances]\label{lem:Vstar}
Define
$\Psi(X):=\mathbb E\!\Bigl[\dfrac{(1-\rho_s)^2(\pi_s-\pi)^2(Y-h)^2}{\rho_s^{2}}\Bigm|X,R=1\Bigr]$
and $\sigma_h^2(a,X):=\mathrm{Var}(h(S,X)\mid A=a,X,R=0)$. Then

\begin{align}
V^{\star}_{\omega}(X)\;:=\;\mathrm{Var}\!\bigl(\xi(Z;\eta_0)\bigm| X\bigr)
\;=\;\omega(X)^{2}\bigl[(1-\rho(X))\,\Sigma_{t}(X)+\rho(X)\,\Sigma_{o}(X)\bigr],
\label{eq:vstar-result}
\end{align}
where
\begin{align*}
\Sigma_{t}(X)
&:=\frac{\sigma_{\mu}(1,X)}{\pi(X)}+\frac{\sigma_{\mu}(0,X)}{1-\pi(X)},\\[2pt]
\Sigma_{o}(X)
&:=\mathbb{E}\!\left[\!\left(\tfrac{1-\rho(X,S)}{\rho(X,S)}\right)^{\!2}\!\!\left(\tfrac{\pi(X,S)-\pi(X)}{\pi(X)(1-\pi(X))}\right)^{\!2}\!\sigma_{y}(S,X)\,\Big|\,X,R=1\right].
\end{align*}
Specifically, 
\begin{align}
V^\star_1(X)
&=(1-\rho)\!\left[\frac{\sigma_h^{2}(1,X)}{\pi}+\frac{\sigma_h^{2}(0,X)}{1-\pi}\right]
+\frac{\rho\,\Psi(X)}{\pi^{2}(1-\pi)^{2}},\label{eq:VstarDR}\\[2pt]
V^\star_{\omega}(X)
&=(1-\rho)\rho^{2}\pi(1-\pi)\!\bigl[(1-\pi)\sigma_h^{2}(1,X)+\pi\sigma_h^{2}(0,X)\bigr]
+\rho^{3}\Psi(X),\label{eq:VstarDO}
\end{align}
the second formula corresponding to $\omega(x)=\pi(x)(1-\pi(x))\rho(x)$. The theorem shows that the weight $\pi(x)(1-\pi(x))\rho(x)$ is already sufficient to resolve the overlap issue; stability of the LT-O-DO-learner with the stronger weight $\omega=\pi^{2}(1-\pi)^{2}\rho$ used in the main text follows naturally.
\end{lemma}
\begin{proof}

Substituting the definitions
\begin{align}
\mathcal{T}_{\mathrm{LT}}(Z;\eta_0)
&=\mathbf{1}(R{=}0)\,\omega(X)\,\hat{\tau}_{\mathrm{AIPW}}(Z;\eta_0)
 +\mathbf{1}(R{=}1)\,\omega(X)\,\psi_{\mathrm{obs}}(Z;\eta_0)
 +\tau_{0}(X)\,\Omega(Z;\eta_0),\label{eq:TLT-def}\\
\omega^{*}(Z;\eta_0)
&=\mathbf{1}(R{=}0)\,\omega(X)+\Omega(Z;\eta_0),\label{eq:omegastar-def}
\end{align}
into $\xi$, and using $\tau_{0}(X)=\mu(1,X)-\mu(0,X)$,
\begin{align}
\xi(Z;\eta_0)
&=\mathbf{1}(R{=}0)\,\omega(X)\,\hat{\tau}_{\mathrm{AIPW}}
+\mathbf{1}(R{=}1)\,\omega(X)\,\psi_{\mathrm{obs}}
+\tau_{0}(X)\,\Omega\notag\\
&\quad-\mathbf{1}(R{=}0)\,\omega(X)\,\tau_{0}(X)
-\Omega\,\tau_{0}(X)\notag\\[2pt]
&=\mathbf{1}(R{=}0)\,\omega(X)\bigl[\hat{\tau}_{\mathrm{AIPW}}(Z;\eta_0)-\tau_{0}(X)\bigr]
\;+\;\mathbf{1}(R{=}1)\,\omega(X)\,\psi_{\mathrm{obs}}(Z;\eta_0).\label{eq:D-canceled}
\end{align}
The two $\Omega$ terms cancel exactly. Then, the derivation is similar to the standard DR variance decomposition in Appendix~\ref{app:variance}. For simplicity, we omit the details.
\end{proof}

\begin{remark}[Tower-law identity for $\Psi$]\label{rem:psi}
Using $f(s\mid X,R=1)=f(s\mid X)\rho_s/\rho$ and Jensen,
$\mathbb E[\rho_s^{-1}\mid X,R=1]=\rho^{-1}$ and
$\mathbb E[\rho_s^{-2}\mid X,R=1]\ge\rho^{-2}$. Thus $\Psi(X)\asymp\rho^{-2}$
(here and below, $A\asymp B$ means $cB\le A\le CB$ for some constants $0<c\le C<\infty$)
holds whenever $|\pi_s-\pi|$ and $\mathbb E[(Y-h)^2\mid X,S,R=1]$ are bounded
above and below and $\rho_s$ degenerates uniformly with $\rho$ (i.e.\
$\rho_s\asymp\rho$). We assume this explicitly below.
\end{remark}

\begin{assumption}[Low-overlap sequence]\label{ass:low-overlap}
There exist constants $c_0,c_1,c_3,c_4,p_0,p_1,\bar\sigma,\bar\Psi>0$,
$\lambda_R,\lambda_B>0$, $\varepsilon_0>0$, and a covariate bound $C_X<\infty$
such that, for all $\varepsilon\le\varepsilon_0$:
\begin{enumerate}\itemsep0pt
\item[(a)] \emph{Bounded covariates and nonvanishing experimental-sample mass:}
$\|X\|\le C_X$ and $\rho_\varepsilon\le 1-c_3$ a.s.\ under $P_\varepsilon$.
\item[(b)] \emph{Bulk overlap:} there exists $\mathcal B$ with $P_\varepsilon(\mathcal B)\ge p_1$
on which $\pi_\varepsilon(1-\pi_\varepsilon)\rho_\varepsilon\ge c_0$, and
$\mathbb E_\varepsilon[XX^\top\mathbf 1\{\mathcal B\}]\succeq\lambda_B I$ uniformly.
\item[(c)] \emph{Low-overlap region:} there exists $\mathcal R_\varepsilon$ with
$P_\varepsilon(\mathcal R_\varepsilon)\ge p_0$,
$\mathbb E_\varepsilon[XX^\top\mathbf 1\{\mathcal R_\varepsilon\}]\succeq\lambda_R I$
uniformly, and on $\mathcal R_\varepsilon$ either
$\pi_\varepsilon(1-\pi_\varepsilon)\asymp\varepsilon$ (low-TO scenario), or
$\rho_\varepsilon\asymp\varepsilon$ with $\pi_\varepsilon(1-\pi_\varepsilon)\ge c_4$
(low-LO scenario)
\item[(d)] \emph{Bounded conditional variances:}
$c_1\le\sigma_{h,\varepsilon}^2(a,X)\le\bar\sigma^2$ uniformly.
\item[(e)] \emph{Surrogate-propensity scaling:} in the low-LO scenario,
$\Psi_\varepsilon(X)\asymp\rho_\varepsilon(X)^{-2}$ on $\mathcal R_\varepsilon$;
in the low-TO scenario, $\Psi_\varepsilon$ is bounded above on $\mathcal R_\varepsilon$.
In both scenarios, $\rho_\varepsilon(X)^{3}\,\Psi_\varepsilon(X)\le\bar\Psi$ a.s.\ under $P_\varepsilon$.
\end{enumerate}
\end{assumption}

We call $\varepsilon$ in condition (c) the \textbf{overlap score} as it describes the severity of the overlap issue. In Condition (c), the second-moment bound is the directional non-degeneracy missing from the original statement: it ensures that no direction $v$ is orthogonal to the support of $X$ on $\mathcal R_\varepsilon$.

Theorem~\ref{thm:linear-asymvar-main} is a direct consequence of the following more explicit statement.

\begin{theorem}[Variance stabilization, detailed]\label{thm:linear-asymvar-detail}
Under Assumption~\ref{ass:low-overlap}, for every fixed $v\in\mathbb R^d\setminus\{0\}$,
\begin{align}
v^\top \mathcal{V}_\varepsilon(\mathbf 1)\,v\;\xrightarrow[\varepsilon\downarrow 0]{}\;\infty,
\qquad
\limsup_{\varepsilon\downarrow 0}\,v^\top \mathcal{V}_\varepsilon(\omega_\varepsilon)\,v\;<\;\infty.
\end{align}
In particular,
$v^\top \mathcal{V}_\varepsilon(\mathbf 1)\,v\,/\,v^\top \mathcal{V}_\varepsilon(\omega_\varepsilon)\,v\to\infty$ as $\varepsilon\downarrow 0$.
\end{theorem}

\begin{proof}
Substituting $V_\varepsilon(\omega)=J_\varepsilon(\omega)^{-1}\Sigma_\varepsilon(\omega)J_\varepsilon(\omega)^{-1}$:
\begin{equation}
v^\top \mathcal{V}_\varepsilon(\omega)\,v
=\bigl(J_\varepsilon(\omega)^{-1}v\bigr)^\top\Sigma_\varepsilon(\omega)\bigl(J_\varepsilon(\omega)^{-1}v\bigr)
=\mathbb E_\varepsilon\!\bigl[V^\star_{\omega,\varepsilon}(X)\,(X^\top u_\varepsilon)^2\bigr],
\label{eq:directional}
\end{equation}
with $u_\varepsilon:=J_\varepsilon(\omega)^{-1}v$. We bound this quantity for
$\omega=\mathbf 1$ and $\omega=\omega_\varepsilon$ separately.

\smallskip
\emph{Step 1 (Regularity of $J_\varepsilon$).}
For $\omega\in\{\mathbf 1,\omega_\varepsilon\}$ we have $0\le(1-\rho_\varepsilon)\omega\le 1$, so
$\|J_\varepsilon(\omega)\|_{\mathrm{op}}\le C_X^2$ by (a). Hence
$u_\varepsilon:=J_\varepsilon(\omega)^{-1}v$ satisfies the lower bound
$\|u_\varepsilon\|\ge\|v\|/C_X^2$ (used in Step~2). Conversely, (a)--(b) give
\[
J_\varepsilon(\mathbf 1)\succeq c_3\,\mathbb E_\varepsilon[XX^\top\mathbf 1\{\mathcal B\}]\succeq c_3\lambda_B I,
\qquad
J_\varepsilon(\omega_\varepsilon)\succeq c_0c_3\,\mathbb E_\varepsilon[XX^\top\mathbf 1\{\mathcal B\}]\succeq c_0c_3\lambda_B I,
\]
so $\|J_\varepsilon(\omega)^{-1}\|_{\mathrm{op}}\le 1/(c_0c_3\lambda_B)$ and
$\|u_\varepsilon\|\le\|v\|/(c_0c_3\lambda_B)$ uniformly in $\varepsilon$ (used in Step~3).

\smallskip
\emph{Step 2 (DR diverges).} Restrict the integral in Eq.~(\ref{eq:directional})
to $\mathcal R_\varepsilon$. By Lemma~\ref{lem:Vstar} and (a),(c)--(e):
in the low-TO scenario, using $\sigma_{h,\varepsilon}^2(a,X)\ge c_1$ from (d)
and the algebraic identity $1/\pi+1/(1-\pi)=1/[\pi(1-\pi)]$, the first term
of Eq.~(\ref{eq:VstarDR}) satisfies
\begin{equation}
(1-\rho_\varepsilon)\!\left[\tfrac{\sigma_{h,\varepsilon}^2(1,X)}{\pi_\varepsilon}+\tfrac{\sigma_{h,\varepsilon}^2(0,X)}{1-\pi_\varepsilon}\right]
\;\ge\;\tfrac{c_3 c_1}{\pi_\varepsilon(1-\pi_\varepsilon)}\;\asymp\;\tfrac{1}{\varepsilon},
\end{equation}
giving $V^\star_{\mathbf 1,\varepsilon}\gtrsim 1/\varepsilon$ on $\mathcal R_\varepsilon$;
in the low-LO scenario, since $\pi_\varepsilon(1-\pi_\varepsilon)\ge c_4$ on
$\mathcal R_\varepsilon$, the second term of~Eq.~(\ref{eq:VstarDR}) satisfies
$\rho_\varepsilon\Psi_\varepsilon/[\pi_\varepsilon(1-\pi_\varepsilon)]^2\asymp\rho_\varepsilon^{-1}\asymp 1/\varepsilon$.
Either way,
\begin{align}
v^\top \mathcal{V}_\varepsilon(\mathbf 1)\,v
\ge\frac{c}{\varepsilon}\,\mathbb E_\varepsilon\!\bigl[(X^\top u_\varepsilon)^2\mathbf 1\{\mathcal R_\varepsilon\}\bigr]
\ge\frac{c}{\varepsilon}\,\lambda_R\|u_\varepsilon\|^2
\ge\frac{c'}{\varepsilon}\|v\|^2,
\end{align}
where the second inequality uses (c) and the third uses Step~1. Taking $\varepsilon\downarrow 0$ gives divergence.

\smallskip
\emph{Step 3 (DO bounded).} By Lemma~\ref{lem:Vstar}, (a) and (d)--(e),
$V^\star_{\omega_\varepsilon,\varepsilon}(X)$ is uniformly bounded on $\mathcal X$:
the first term of Eq.~(\ref{eq:VstarDO}) is at most $\bar\sigma^2/4$ via
$(1-\rho_\varepsilon)\rho_\varepsilon^2\pi_\varepsilon(1-\pi_\varepsilon)\le 1/4$
and (d), while the second term $\rho_\varepsilon^3\Psi_\varepsilon\le\bar\Psi$
by (e). Setting $\bar V:=\bar\sigma^2/4+\bar\Psi$,
\begin{align}
v^\top \mathcal{V}_\varepsilon(\omega_\varepsilon)\,v
\le \bar V\cdot\mathbb E_\varepsilon\!\bigl[(X^\top u_\varepsilon)^2\bigr]
=\bar V\cdot u_\varepsilon^\top\mathbb E_\varepsilon[XX^\top]u_\varepsilon
\le\bar V C_X^2\|u_\varepsilon\|^2,
\end{align}
which is uniformly bounded by Step~1.

\smallskip
\emph{Step 4 (Ratio well-defined).} On $\mathcal B$ we have
$(1-\rho_\varepsilon)\omega_\varepsilon=(1-\rho_\varepsilon)\pi_\varepsilon(1-\pi_\varepsilon)\rho_\varepsilon\ge c_3 c_0$,
and from Eq.~(\ref{eq:VstarDO}) together with (d),
$V^\star_{\omega_\varepsilon,\varepsilon}\ge c_3 c_0^2 c_1$ on $\mathcal B$.
Combined with $\mathbb E_\varepsilon[XX^\top\mathbf 1\{\mathcal B\}]\succeq\lambda_B I$ from (b),
$v^\top \mathcal{V}_\varepsilon(\omega_\varepsilon)v\ge c_3 c_0^2 c_1\,\lambda_B\|u_\varepsilon\|^2>0$
for $v\ne 0$. Combining Steps~2--4 gives the claim.
\end{proof}

\begin{remark}[Scope of the comparison]
Theorem~\ref{thm:linear-asymvar-main} (detailed in Theorem~\ref{thm:linear-asymvar-detail}) establishes \emph{directional robustness in the low-overlap regime}, not uniform PSD (positive semi-definite) dominance: in the high-overlap bulk,
the smaller $J(\omega)$ relative to $J(\mathbf 1)$ can tilt the comparison
in the opposite direction. The non-degeneracy condition
$\mathbb E[XX^\top\mathbf 1\{\mathcal R_\varepsilon\}]\succeq\lambda_R I$
is what makes the divergence hold for every direction $v$; without it,
one obtains divergence only along $v$ for which $X^\top J(\mathbf 1)^{-1}v$
has non-vanishing $L^2$ mass on the low-overlap region.
\end{remark}

\end{document}